\documentclass[11pt]{article}

\usepackage{macros}
\usemedgeometry
\bluehyperref

\usepackage{enumitem}

\newcommand{\diffp}{\varepsilon}

\newcommand{\eqdiffp}{\stackrel{d}{=}_{\diffp}}

\newcommand{\totdiffp}{\diffp_{\textup{total}}}

\newcommand{\param}{\theta}
\newcommand{\hparam}{\what{\theta}}

\newcommand{\nsub}{s}
\newcommand{\nsubs}{\nsub}
\newcommand{\subsize}{m}
\newcommand{\subsample}{^{\textup{sub}}}
\newcommand{\approxdist}{\stackrel{\textup{dist}}{\approx}}

\newcommand{\tparam}{\wt{\param}}

\newcommand{\est}[1]{\widehat{#1}}

\newcommand{\unstudstatest}{U_n}  %
\newcommand{\unstudstatpriv}{\wt{U}_n}  %
\newcommand{\blbstatistic}{U_{\subsize,n}^*}  %
\newcommand{\tblbstatistic}{\wt{U}_{\subsize,n}^*}  %

\newcommand{\quant}{\mathsf{quant}}

\newcommand{\exactconf}{C_\alpha^{\textup{exact}}}

\newcommand{\dham}{d_{\textup{ham}}}

\newcommand{\floors}[1]{\lfloor{#1}\rfloor}

\newcommand{\noisevar}{W}
\newcommand{\noiseval}{w}

\newcommand{\prn}[1]{\left({#1}\right)}
\newcommand{\prninline}[1]{({#1})}

\newcommand{\threshold}{\tau}
\renewcommand{\abs}[1]{\left|{#1}\right|}
\newcommand{\absinline}[1]{|{#1}|}
\newcommand{\medacc}{\omega}  %
\newcommand{\cdfacc}{\omega_{\textup{cdf}}}  %
\newcommand{\varacc}{\omega_{\textup{var}}}  %
\newcommand{\subsampleind}[1]{^{\textup{sub}({#1})}}
\newcommand{\numqueries}{T}
\newcommand{\nmontecarlo}{N_{\textup{mc}}}
\newcommand{\failprob}{\beta}

\newcommand{\smoothparam}{\rho}
\newcommand{\len}{\mathsf{len}}
\newcommand{\ordstat}{\textsc{OrdSt}\xspace}

\newcommand{\poploss}{L}
\newcommand{\loss}{\ell}
\newcommand{\lambdareg}[1][]{
  \ifthenelse{\isempty{#1}}{%
    \lambda_{\textup{reg}}
  }{%
    \lambda_{\textup{reg},{#1}}
  }
}
\newcommand{\lipconst}{G}
\newcommand{\lipobj}{\lipconst_0}
\newcommand{\lipgrad}{\lipconst_1}
\newcommand{\liphess}{\lipconst_2}
\newcommand{\lambdamin}{\lambda_{\min}}
\newcommand{\lambdamax}{\lambda_{\max}}

\newcommand{\remainder}{R}

\newcommand{\algname}[1]{\textsc{#1}}

\newcommand{\abthreshmed}{\hyperref[algorithm:above-threshold-median]{\algname{AboveThr}}}

\newcommand{\privmed}{\hyperref[algorithm:priv-median]{\algname{PrivMedian}}}
\newcommand{\blbvar}{\hyperref[algorithm:blb-var]{\algname{BLBvar}}}

\newcommand{\bootvaralg}{\hyperref[algorithm:boot-var]{\algname{BootVar}}}

\newcommand{\blbquant}{\hyperref[algorithm:blb-quant]{\algname{BLBquant}}}

\usepackage{subcaption}

\usepackage[capitalize]{cleveref}
\usepackage{crossreftools}
\usepackage[numbers]{natbib}

\usepackage{setspace}

\title{Resampling methods for private statistical inference}

\Crefname{algorithm}{Algorithm}{Algorithm}
\crefname{algocf}{alg.}{algs.}
\Crefname{algocf}{Algorithm}{Algorithms}
\Crefname{assumption}{Assumption}{Assumptions}

\author{Karan Chadha$^1$~~~~~John C.\ Duchi$^{1,2}$~~~~Rohith Kuditipudi$^3$
  \vspace{.1cm}
  \\
  \vspace{.1cm}
  Departments of $^1$Electrical Engineering,
  $^2$Statistics, and $^3$Computer Science \\
  Stanford University}
\date{May 2024}

\begin{document}

\maketitle

\begin{abstract}
We propose two private variants of the non-parametric bootstrap for
privately computing confidence sets. Each privately computes the median of
results of multiple “little” bootstraps, yielding asymptotic bounds on the
coverage error of the resulting confidence sets. For a fixed differential
privacy parameter $\varepsilon$, our methods enjoy the same error rates
as the non-private bootstrap to within logarithmic factors in the
sample size $n$. We empirically validate the performance of our methods for
mean estimation, median estimation, and logistic regression, and our methods
achieve similar coverage accuracy to existing methods (and non-private
baselines) while providing notably shorter (about $10\times$) confidence
intervals than previous approaches.
\end{abstract}

\section{Introduction}\label{sec:intro}

To realize the promise of private data analysis, we must bring it into
closer contact with the inferential goals of modern data analysis and
statistics. Yet many challenges conflict with these goals: most work in
(differential) privacy targets point estimates rather than valid confidence
sets or other inferential tools~\cite{Wasserman12, DworkRo14}; private
estimators are typically bespoke, tailored to particular
problems~\cite{ChaudhuriMoSa11, KarwaVa18, GaboardiRoSh19, AsiDu20}; and the
fine-grained analysis necessary to evaluate variance-like properties of
estimators, essential to applying classical inferential results using
asymptotic normality~\cite{VanDerVaart98, LehmannRo05}, remains out of reach
for many private estimators. We lack generic procedures for privately
completing the most basic inferential task, to construct confidence
intervals around a given point estimate.  We therefore develop tools, based
on resampling methods, to do so for a broad collection of statistics.

\newcommand{\pointmass}{\mathbf{1}}

Letting data $X_1, \ldots, X_n \simiid P$ for a
population distribution $P$, our goal is to provide a valid confidence set
for a parameter $\param(P)$ of interest. The sample mean $\param(P) =
\E_P[X]$ will provide an example we repeatedly revisit. Let $P_n =
\frac{1}{n} \sum_{i = 1}^n \pointmass_{X_i}$ be the empirical distribution
on the sample (which we sometimes simply identify as the sample), where
$\pointmass_x$ denotes a point mass at $x$, so that any estimator is a
function of $P_n$. Then we wish to construct an asymptotically correct
confidence set $C_\alpha(P_n)$ for $\param(P)$, meaning that for each
population $P$,
\begin{equation}
  \label{eqn:asymptotic-cl}
  \lim_{n \to \infty} \P\left(\param(P) \in C_\alpha(P_n)\right)
  = 1 - \alpha,
\end{equation}
where the probability is taken over the sample $X_i \simiid P$ and
randomness in $C_\alpha(P_n)$. Focusing
on $\R$-valued parameters for simplicity (though this is not
essential), typical
methods~\cite{Lehmann99, LehmannRo05} achieve the
coverage~\eqref{eqn:asymptotic-cl} by approximating the distribution of the
centered statistic
\begin{equation*}
  \unstudstatest \defeq \sqrt{n}\left(\param - \param(P_n)\right)
  \cd \normal(0, \sigma^2),
\end{equation*}
and assuming an estimate $\what{\sigma}_n \cp \sigma$, the confidence
interval $C_\alpha(P_n) = \param(P_n) \pm z_{1 - \alpha/2} \what{\sigma}_n /
\sqrt{n}$, where $z_\alpha$ is the usual $\alpha$-quantile of the standard
normal, gives asymptotic coverage.

We develop variants of Kleiner et al.'s Bag of Little Boostraps
(BLB)~\cite{KleinerTaSaJo14} to approximate the statistic $\unstudstatest$
privately, both directly via resampling and by estimating the variance
$\sigma^2$ and making a normal approximation. Key to both the theoretical
and practical performance of our methods are that we aggregate information
from resampling via private median-like algorithms, whose stability both
improves accuracy and privacy. Our main theoretical results provide both
general consistency results~\eqref{eqn:asymptotic-cl} as well as asymptotic
rate guarantees that $\P(\param(P) \in C_\alpha(P_n)) = 1 - \alpha +
\wt{O}(n^{-1/2})$, where $\wt{O}$ hides logarithmic factors and dependence
on the privacy level $\diffp$; this rate is similar to those achievable in
non-private cases, though attaining the higher-order accuracy sometimes
possible via bootstrap resampling methods~\cite{Hall92} remains open.

By showing that the results of each ``little'' bootstrap concentrate
around appropriate population-level values, we can use
the strong error guarantees private median algorithms provide
for concentrated inputs~\cite{AsiDu20} to derive our error bounds.
Summarizing, our main contributions follow:
\begin{enumerate}[leftmargin=*,label=(\arabic*)]
\item We propose two private variants of the unstudentized non-parametric
  bootstrap. The first, Alg.~\ref{algorithm:blb-quant} ($\blbquant$),
  constructs private confidence intervals by estimating quantiles of 
  distribution of the estimators. The second,
  Alg.~\ref{algorithm:blb-var} ($\blbvar$), estimates the squared error of
  an estimator, which then allows normal
  approximation.
\item We show consistency of the bootstrap procedures for sufficiently
  accurate private estimators, proving asymptotic bounds on the coverage
  error rates that match non-private bootstrap methods to within logarithmic
  factors.
\item We investigate the performance of our methods on both synthetic and
  real-world datasets (see Section~\ref{sec:expts}). The results highlight
  the importance of both (i) strong baseline private estimators for good
  uncertainty estimation and (ii) newer median-based (as opposed to
  mean-based) aggregation methods.
\end{enumerate}

\subsection{Related Work}\label{sec:related-work}

Whereas much early work on differential privacy focuses on the design of
algorithms for privately querying data and constructing sample
statistics~\cite{DworkRo14, Wasserman12}, a growing body of work seeks to
marry differential privacy and statistical inference.  \citet{Smith11} shows
that a large class of (asymptotically normal) statistical estimators admit
private counterparts that asymptotically converge to the same Gaussian
distribution.  More recent work gives theoretical results on estimating
variance and covariance, such as for univariate
Gaussians~\cite{KarwaVa18}, sub-Gaussian vectors~\citep{BiswasDoKaUl20},
or empirical risk minimization~\cite{WangKiLe19, AvellaBrLo21},
which allows normal approximation to compute confidence intervals.
But natural covariance estimates, for example,
the standard inverse Hessian appearing in
the covariance of classical empirical risk minimization
and M-estimators~\cite[e.g.][Thm.~5.41]{VanDerVaart98}, suffer
instabilities, making practical private estimation challenging.
Besides normal
approximation, there is also work that studies confidence intervals for
specific problems. For example,
\citet{Sheffet17} and Alabi et al.~\cite{AlabiMcSaSmVa22, AlabiVa22} study
differentially private linear regression and hypothesis testing for presence
of linear relationships. \citet{DrechslerGlMcSaSm21} benchmark different
methods for constructing differentially private confidence intervals for the
median.

We lack non-parametric, fully data-driven mechanisms that generate private
confidence intervals for broad classes of estimators. \citet{FerrandoWaSh22}
propose mechanisms for privatizing the parametric bootstrap, but their
validity guarantees rely on well-specified parametric data models.
\citet{EvansKiScTh21} and \citet{CovingtonHeHoKa21} propose bootstrap
algorithms similar to ours, using these to estimate covariance of arbitrary
estimators by aggregating results of little bootstraps via clipping and
noise additiong or CoinPress~\citep{BiswasDoKaUl20}, which privately
estimates means directly. However, they do not appear to provide accuracy
guarantees, and---as we see especially in experiments---more careful
aggregation via median-like mechanisms can yield much stronger performance.

\section{Preliminaries and notation}\label{sec:preliminaries}

Consider a parameter of interest $\param(P) \in \R$; we
address vector-valued parameters presently.
For a sample $P_n$, i.e., the empirical distribution on $n$ observations
drawn i.i.d.\ from $P$, we (classically, without privacy)
have an estimator $\param(P_n)$ of $\param(P)$,
and wish to estimate the distribution of the centered statistic
\begin{equation}
  \unstudstatest = \sqrt{n}\left(\param(P) - \param(P_n)\right).
  \label{eqn:unstudstatest}
\end{equation}
We will typically assume we have the standard asymptotics that
\begin{equation*}
  \unstudstatest \cd \normal\left(0, \sigma^2(\param)\right),
\end{equation*}
where $\sigma^2(\param)$ is the asmptotic variance for the particular
parameter (we leave dependence on the estimator
$\param(P_n)$ implicit).
Letting
\begin{equation*}
  \quant_\beta(X) \defeq \inf\left\{t \in \R \mid \P(X \le t) \ge \beta
  \right\}
  ~~ \mbox{and} ~~
  \quant_\beta(P) \defeq \inf\left\{t \in \R \mid P\left(\openleft{-\infty}{t}
  \right) \ge \beta \right\}
\end{equation*}
denote the $\beta$-quantile of a random variable $X$ or probability
distribution $P$, respectively, it is then immediate that
\begin{equation*}
  \exactconf(P_n) \defeq
  \left[\param(P_n) + n^{-1/2} \quant_{\alpha/2}(\unstudstatest),
    \param(P_n) + n^{-1/2} \quant_{1-\alpha/2}(\unstudstatest)\right]
\end{equation*}
is a $1 - \alpha$ confidence set. Normal
approximations and boostrap resampling methods approximate
the distribution of $\unstudstatest$ to mimic this exact confidence set.

We adapt the Bag-of-Little-Bootstraps (BLB)~\citep{KleinerTaSaJo14},
which descends from Efron's bootstrap~\citep{Efron79, EfronTi93}.
The BLB first constructs
$\nsub$ subsamples of $P_n$, each of size $\subsize = n / \nsub$;
let $P_\subsize\subsample$ denote one of these subsamples.
For each subsample, the BLB redraws an i.i.d.\ sample of size $n$ to yield
a resampled distribution $P_{\subsize,n}^*$,  the empirical
distribution of the resample from the $\subsize$ points in the
subsample. It then makes the approximation
\begin{equation}
  \label{eqn:key-blb-approximation}
  \unstudstatest \approxdist \sqrt{n}\left(\param(P_\subsize\subsample) -
  \param(P_{\subsize,n}^*)\right)
  \eqdef \blbstatistic
\end{equation}
where Monte-Carlo sampling yields arbitrarily accurate estimates of the
distribution of the right hand side.  The
approximation~\eqref{eqn:key-blb-approximation} holds so long as the
estimates $\param(P_n)$ have appropriate smoothness and convergence
properties relative to the sampling distributions $P_n$, and demonstrating
its accuracy forms the core of any argument for validity of
confidence intervals.  Assuming it holds, the confidence set
\begin{equation*}
  C_\alpha(P_n) \defeq \left[\param(P_n) + n^{-1/2}
    \quant_{\alpha/2}(\blbstatistic),
    \param(P_n) + n^{-1/2}
    \quant_{1 - \alpha/2}(\blbstatistic)\right]
\end{equation*}
satisfies $\P(\theta(P) \in C_\alpha(P_n)) \to \P(\theta(P) \in
\exactconf(P_n)) = 1- \alpha$. (See the sequel and \citet{KleinerTaSaJo14}
for more precision.)
We may also construct confidence sets by directly leveraging the asymptotic normality of the centered
statistic $\unstudstatest$: for \emph{any} consistent
estimate $\what{\sigma}^2$ of the asymptotic variance $\sigma^2(\param)$,
we have
\begin{equation*}
  \P\left(\theta(P) \in \theta(P_n) \pm n^{-1/2} \what{\sigma}^2 \cdot
  z_{1 - \alpha/2}
  \right) \to 1 - \alpha,
\end{equation*}
where $z_\alpha$ is the $\alpha$-quantile of a standard Gaussian.  The BLB
variance estimate $\what{\sigma}^2 = \var(\blbstatistic \mid P_\subsize)$
provides a natural proxy for this variance.

\newcommand{\mech}{M}

We use the standard definition of differential
privacy~\cite{DworkRo14}. We
will typically use a tilde, $\wt{\param}$, to denote a private estimator
and distinguish it from $\param(\cdot)$. Such an estimator typically
requires additional randomness independent of the sample $P_n$,
which it uses to obfuscate the input distribution;
to abstract this away, we simply call this
random variable $\noisevar$. Then for $\diffp > 0$ and $\delta \in (0, 1)$,
$\wt{\param}$ is $(\diffp, \delta)$-differentially
private if
for all $P_n, P_n'$ differing in at most a single observation
and all measurable sets $O$,
\begin{equation}
  \label{eqn:diffp-def}
  \P\left(\wt{\param}(P_n, \noisevar) \in O\right)
  \le \exp(\diffp) \cdot \P\left(\wt{\param}(P_n', \noisevar) \in O\right)
  + \delta.
\end{equation}
When the noise is clear from context we use $\wt{\param}(P_n)$ instead of
$\wt{\param}(P_n, \noisevar)$.  It will be convenient to write
distributional closeness, where for random variables $X, X'$, we write $X
\eqdiffp X'$ to mean that $\P(X \in O) \le e^\diffp \P(X' \in O)$ for all
measurable sets $A$. Definition~\eqref{eqn:diffp-def}
with $\delta = 0$
is thus equivalent to $\wt{\param}(P_n) \eqdiffp \wt{\param}(P_n')$.

\paragraph{Notation.}
We use $[k]$ to denote the set of integers $\{1,\dots,k\}$ and
$y_{1:s}$ to
denote the tuple $(y_1,\dots,y_s)$
We let $x_{(1)} \le x_{(2)} \le \cdots \le x_{(n)}$ denote the
order statistics of $x \in \R^n$.  We use $\Phi(t)$ and $\phi(t)$
to denote the cumulative distribution and
probability density of the standard normal at $t \in \R$.

\section{Private Confidence Intervals}
\label{sec:priv-conf-int}

\jcdcomment{It seems we should have a discussion of exactly
  the accuracies in Assumptions~\ref{assumption:bootcdf-accuracy}
  and~\ref{assumption:bootstd-accuracy}, perhaps in the
  context of the examples below.
  Without such a discussion, they seem to come from nowhere and
  do not really get referenced.}

To construct confidence intervals with differential privacy, we replace the
estimator in the centered statistic~\eqref{eqn:unstudstatest} and in the
bootstrap resampling approximation~\eqref{eqn:key-blb-approximation} with
a privatized counterpart, applying either the percentile or normal
approximation (estimated via resampling) method.  Using
the notation of our differential privacy definition~\eqref{eqn:diffp-def}
with private estimator $\wt{\param}$,
the new centered statistic becomes
\begin{equation}
  \label{eqn:private-unstudentized}
  \unstudstatpriv \defeq
  \sqrt{n}\left(\param(P) - \tparam(P_n, W_n)\right).
\end{equation}
Replacing $\param(P_n)$ with $\tparam(P_n, W_n)$ in the exact
confidence set $\exactconf(P_n)$ evidently yields a finite sample $1 -
\alpha$ confidence set, and so to develop private confidence sets, we
develop private analogues of the resampling
approximation~\eqref{eqn:key-blb-approximation}.

We describe algorithms for constructing
private confidence intervals using the percentile method
in Section~\ref{sec:percentile-bootstrap} and
to estimate squared error and variance
in Section~\ref{sec:moment-est}, which we may use for
normal approximation. The
key to both the theoretical and practical performance of the methods is to
use private median procedures to aggregate the information from bootstrap
resampling, which we turn to now.

\jcdcomment{Should we put the table here that captures accuracy? Or another
  summary place?}

\subsection{Private median algorithms}

The basic building block of our percentile algorithms is to estimate the
first index in a collection of vectors whose median passes a particular
threshold. For an increasing sequence of sets $I_1, I_2, \ldots,
I_\numqueries$, we use this to estimate the smallest index $t$ for which
$\P(\unstudstatpriv \in I_t) \ge 1 - \alpha$.
We develop a median-based private version of this process
as Algorithm~\ref{algorithm:above-threshold-median}.
In the algorithm, for a vector $y \in \R^k$ we let
\begin{equation*}
  \ordstat(y, \xi) = \begin{cases}
    y_{(\floor{\xi})} & \mbox{if}~ 1 \le \xi \le k \\
    -\infty \cdot \indic{\xi < 1}
    + \infty \cdot \indic{\xi > k}
    & \mbox{otherwise}
  \end{cases}
\end{equation*}
be the $\floor{\xi}$th order statistic with limiting extremes.

\begin{algorithm}[ht]
  \caption{\label{algorithm:above-threshold-median} Median
    Above Threshold ($\abthreshmed$)}
  \DontPrintSemicolon
  \SetKwInOut{Input}{Input}
  \SetKwInOut{Noise}{Noise}
  \Input{ data $y(t) \in \R^k$ for $t \in [\numqueries]$,
    threshold $\threshold \in \R$ }
  \Noise{ $\xi_t \in \R$ for $t \in [\numqueries]$}
  \For{$t \in [\numqueries]$}{
    $\est{v} \leftarrow \ordstat({y(t);\xi_0 + \xi_t})$\;
    \label{line:noisy-order-statistic}
    \If{$\est{v} \geq \threshold$}{
      \Return $t$\;
    }
  }
  \Return $\perp$\;
\end{algorithm}

Given a sequence of $\numqueries$ vectors $y(t) \in \R^k$,
Algorithm~\ref{algorithm:above-threshold-median} finds the index of the
first query whose noisy median exceeds a threshold $\threshold \in \R$,
returning $\perp$ if no such query exists. We choose the noise variables
$\xi_0$ and $\xi \in \R^\numqueries$ so that that $\abthreshmed$ satisfies
$\diffp$-DP and $\E[\xi_0 + \xi_t] = \frac{k}{2}$.

\begin{proposition}
  \label{proposition:privacy-threshold}
  Let $y(t)$ and $z(t) \in \R^k$, $t = 1, \ldots, T$,
  be sequences of vectors satisfying
  $\dham(y(t), z(t)) \le 1$ for each $t$. Let
  $\diffp > 0$ and
  $\xi_0 \sim \laplace(\frac{k}{2}, \frac{2}{\diffp})$ and
  $\xi_t \simiid \laplace(0, \frac{4}{\diffp})$ for $t \in \N$. Then
  for any $\threshold \in \R$,
  \begin{equation*}
    \abthreshmed(y, \threshold, \xi)
    \eqdiffp \abthreshmed(z, \threshold, \xi).
  \end{equation*}
\end{proposition}

When the vector $y \in \R^k$ appropriately concentrates, its order
statistics do as well, which we can leverage in our private percentile-based
BLB method to come.  We formalize this via the following intermediate lemma,
whose proof we defer to Appendix~\ref{sec:proof-noisy-median}.

\begin{lemma}
  \label{lemma:noisy-median}
  Let $Y \in \R^k$ have
  independent entries, and assume for some
  $\medacc < \infty$ there is $\mu \in \R$ such that
  $\P(|Y_i - \mu| \le \medacc) \ge \frac{3}{4}$.
  Let $\xi_0 \sim \laplace(\frac{k}{2}, b)$ and
  $\xi \sim \laplace(0, 2b)$ for some $b > 0$.
  Then
  \begin{equation*}
    \P\left(
    |\ordstat(Y, \xi_0 + \xi) - \mu| > \medacc
    \right)
    \le \frac{4}{3} \exp\left(-\frac{k}{16 b}\right)
    + \exp\left(-\frac{k}{32}\right).
  \end{equation*}
\end{lemma}

\subsection{A private percentile bootstrap}
\label{sec:percentile-bootstrap}

We now present an $\diffp$-differentially private algorithm for finding a
set providing a desired $1 - \alpha$ coverage level. The method uses an
increasing sequence of sets $I_1, I_2, \ldots, I_\numqueries$ and
(approximately) outputs the first in this sequence with estimated coverage
at least $1 - \alpha$.  Typical choices of the sets $I_t$ for an $\R^d$-valued
parameter $\param(P)$ are $I_t = \{u \in \R^d \mid \ltwo{u} \le h \cdot
t\}$ for a small resolution $h$, and recalling the
definition~\eqref{eqn:private-unstudentized} of $\unstudstatpriv$, we seek
the smallest $t$ with $\P(\unstudstatpriv \in I_t) \ge 1 - \alpha$.  The
algorithm, $\blbquant$, follows the BLB methodology to first partition the
data into $\nsub = n / \subsize$ disjoint subsets of size $\subsize$, then
uses bootstrap estimates of $\P(\unstudstatpriv \in I_t)$ calculated using
each partition independently by resampling $n$ points (with replacement)
from the subset.  It then uses $\abthreshmed$ and these estimates to
(approximately) find the first set $I_t$ in the sequence with coverage
probability at least $1-\alpha$.

\begin{algorithm}
  \caption{\label{algorithm:blb-quant} Private BLB quantile estimation
    (\blbquant)}
    \DontPrintSemicolon
    \SetKwInOut{Input}{Input}
    \Input{Sample $P_n$,
      sequence of sets of interest $I_{1:\numqueries}$,
      subset size $\subsize$, privacy parameter $\diffp$,
      estimators $\param(\cdot)$ and $\tparam(\cdot)$}
    Draw $\nsub = \lfloor n/\subsize \rfloor$ disjoint
    subsamples $P\subsampleind{1}_\subsize, \ldots, P\subsampleind{\nsub}_\subsize$
    of $P_n$ of size $\subsize$ \; \label{line:draw-subsamples-quant}
    \For{$i \in [s]$ and $t \in [\numqueries]$}{
      $\hparam \leftarrow \param(P_\subsize)$
      \emph{~~~~// saved plug-in estimator}\;
      \For{$j \in [\nmontecarlo]$}{
        Draw resample of $P_{\subsize,n}^*$ of size $n$
        i.i.d.\ with replacement from $P_\subsize\subsampleind{i}$ \;
        $\theta_j \gets \tparam(P_{\subsize,n}^*)$
      }
      $\what{p}_i(t) \gets \frac{1}{\nmontecarlo}
      \sum_{j = 1}^{\nmontecarlo} \indics{\sqrt{n} (\hparam
        - \theta_j) \in I_t}$
      ~~~~ \emph{// coverage estimate}\; \label{line:coverage-estimate}
    }
    Draw $\xi_0 \sim \laplace(n/2,2/\diffp)$
    and $\xi_{1:\numqueries} \simiid
    \laplace(0, 4/\diffp)$\; \label{line:noise-def}
    Set $\what{t} = \abthreshmed(\{\what{p}(1), \ldots, \what{p}(\numqueries)\},
    1 - \alpha, \xi)$\;
    \label{line:call-above-threshold}
    \Return $I_{\what{t}}$\;
\end{algorithm}

The privacy of $\blbquant$ is relatively straightforward once we
leverage privacy composition, but its utility requires more work.
Here, we make the assumption that the bootstrap CDF estimation
subroutine is accurate:
\begin{assumption}
  \label{assumption:bootcdf-accuracy}
  Let $\what{p}_i(t) \in [0,1]$, $t = 1, \ldots, \numqueries$ be
  the estimated coverage probabilities in Line~\ref{line:coverage-estimate}
  of Alg.~\ref{algorithm:blb-quant}.
  Let $\unstudstatpriv = \sqrt{n}(\param(P) - \tparam(P_n))$
  as in~\eqref{eqn:private-unstudentized}
  and $p(t) = \P(\unstudstatpriv \in I_t)$ be the true coverage
  probability.
  There exists $\cdfacc < \infty$ such that for any collection of
  sets $I_1, \ldots, I_\numqueries$,
  \begin{equation*}
    \P\left(\abs{p(t) - \what{p}_i(t)} \ge \cdfacc
    \right) \ge \frac{3}{4},
  \end{equation*}
  where the probability is over the i.i.d.\ sampling to construct
  $P_\subsize\subsampleind{i}$, resampling of $P_{\subsize,n}^*$,
  and any randomness in $\tparam(\cdot)$.
\end{assumption}

The next theorem gives both a privacy guarantee, which holds
unconditionally, and an accuracy guarantee under
Assumption~\ref{assumption:bootcdf-accuracy}.

\begin{theorem}\label{thm:blb-abthreshmed}
  Let $I_1, \ldots, I_\numqueries$ be an arbitrary sequence of sets.
  Then the output $I_{\what{t}}$ of $\blbquant$ is $\diffp$-differentially
  private. Additionally, let Assumption~\ref{assumption:bootcdf-accuracy}
  hold, assume the sets $\{I_t\}$ are nondecreasing, and assume
  there exists an index $t$ satisfying
  $1 - \alpha + \cdfacc \le \P(\unstudstatpriv \in I_t)
  \le 1 - \alpha + 2 \cdfacc$. Let $\failprob > 0$.
  Then if $\nsub >
  32\max\{\frac{1}{\diffp} \log\frac{8 \numqueries}{3\failprob},
  \log\frac{2 \numqueries}{\failprob}\}$,
  with probability at least $1 - \failprob$ the output
  $I_{\est{t}}$ of $\blbquant$ satisfies
  \begin{equation*}
    1 - \alpha - \cdfacc
    \le \P(\unstudstatpriv \in I_{\est{t}}) \le 1 - \alpha + 2 \cdfacc.
  \end{equation*}
\end{theorem}
\begin{proof}
  To demonstrate privacy, recognize that for each $t = 1, \ldots, T$, the
  vector $\what{p}(t) = (\what{p}_i(t))_{i = 1}^\nsub$ returns has entries
  computed on disjoint subsets of the sample $P_n$. For any samples $P_n,
  P_n'$ differing in only a single observation and corresponding outputs
  $\what{p}$ and $\what{p}'$, we thus have $\dham(\what{p}(t), \what{p}'(t))
  \le 1$ (we may use the same random seed to couple the sampling), and so
  $\abthreshmed$ guarantees $\diffp$-differential privacy
  by Proposition~\ref{proposition:privacy-threshold}.

  To provide an accuracy guarantee, we use Lemma~\ref{lemma:noisy-median}
  and Assumption~\ref{assumption:bootcdf-accuracy}.
  With probability at least $1 - \beta$,
  we have
  \begin{equation*}
    \abs{\ordstat\left([\what{p}_i(t)]_{i = 1}^\nsub, \xi_0 + \xi_t \right)
      - \P(\unstudstatpriv \in I_t)} \le \cdfacc
    ~~ \mbox{for~each}~ t \in T
  \end{equation*}
  by a union bound,
  because $\what{p}(t)$ has independent entries.
  By assumption in the theorem that
  there exists an index $t'$ such that
  $\P(\unstudstatpriv \in I_{t'}) \in [1 - \alpha + \cdfacc,
    1 - \alpha + 2 \cdfacc]$,
  on this event, $\ordstat(\what{p}(t'), \xi_0 + \xi_{t'}) \ge 1 - \alpha$.
  Then $\abthreshmed$ necessarily outputs some index
  $\what{t} \le t'$, and the preceding accuracy
  guarantee implies that
  $\P(\unstudstatpriv \in I_{\what{t}}) \in [1 - \alpha - \cdfacc,
    1 + \alpha + 2 \cdfacc]$.
\end{proof}

The output of $\blbquant$ directly provides a natural
private confidence set.
\begin{corollary}
  \label{corollary:coverage-from-percentile}
  Let the conditions of Theorem~\ref{thm:blb-abthreshmed} hold and
  $\tparam(\cdot)$ be an $(\diffp, \delta)$-differentially private estimator.
  Let $I_{\what{t}}$ be the output of $\blbquant$, and define the confidence
  set
  \begin{equation*}
    C_\alpha(P_n) \defeq \tparam(P_n) + \frac{1}{\sqrt{n}} I_{\what{t}}.
  \end{equation*}
  Then $C_\alpha(P_n)$ is $(2 \diffp, \delta)$-differentially private, and
  \begin{equation*}
    1 - \alpha - \cdfacc \le \P(\param(P) \in C_\alpha(P_n))
    \le 1 - \alpha + 2 \cdfacc.
  \end{equation*}
\end{corollary}

The scaling of the accuracy $\cdfacc$ in
Assumption~\ref{assumption:bootcdf-accuracy} thus determines the accuracy of
the confidence set $C_\alpha(P_n)$ that $\blbquant$ yields, and we address
this in the next section.  Previewing the results, so long as the
non-private and private estimators are close, we have at least $\cdfacc \to
0$ as the sample size grows (see
Corollary~\ref{corollary:basic-percentile-consistency}).  Under mild
regularity conditions on the non-private estimator $\param(\cdot)$, we
typically have the scaling $\cdfacc \lesssim \frac{1}{\sqrt{\subsize}} +
\frac{1}{\diffp \sqrt{n}}$ (see Example~\ref{example:resampling-mean-2} and
Corollary~\ref{corollary:typical-edgeworth-consistency}).

To make the types of sets we recover more concrete, consider the typical
asymptotically normal case that both $\unstudstatest \cd \normal(0, \Sigma)$
and $\unstudstatpriv \cd \normal(0, \Sigma)$ for some unknown covariance
$\Sigma$ (or, in the one-parameter case, a variance $\sigma^2$). In this
case, a natural choice for the sets $I_t$ is $I_t = \{u \mid \ltwo{u} \le h
\cdot t\}$, where we take $h \propto \frac{1}{\sqrt{n}}$ and $t \lesssim n$,
yielding $I_n = \{u \mid \ltwo{u} \le \sqrt{n}\}$, and
then the confidence set in Corollary~\ref{corollary:coverage-from-percentile}
gives
\begin{equation*}
  \P(\param(P) \in C_\alpha(P_n))
  = \P\bigg(\ltwobig{\param(P) - \tparam(P_n)} \le \frac{\what{t}}{\sqrt{n}}
  \bigg) = 1 - \alpha + o(1),
\end{equation*}
where the final equality holds so long as the distributional approximation
$\tblbstatistic \stackrel{\textup{dist}}{\approx} \unstudstatpriv$ holds,
as $t \mapsto \P(\ltwo{z} \ge t)$ is $\mc{C}^\infty$ for $Z \sim \normal(0,
\Sigma)$. The box-shaped sets $I_t = [-h t, h t]^d$ also provide
natural confidence sets.

Because Theorem~\ref{thm:blb-abthreshmed} shows the accuracy degrades only
logarithmically in the number $\numqueries$ of sets $\{I_t\}$, we typically
choose the number of subsamples $\nsubs = K\frac{\log n}{\diffp}$ for some
constant $K$.  We comment more on the hyperparameter selection and
sensitivity in \Cref{appen:expts}.
Theorem~\ref{thm:blb-abthreshmed} also highlights the importance of
using median-like aggregation and $\abthreshmed$, where
because the coverage estimates $\what{p}$ concentrate,
the failure probability decreases exponentially in the number
of subsamples $\nsubs$.

\subsection{Private error estimation and normal approximation}
\label{sec:moment-est}

To construct confidence sets using direct estimates of estimator error,
which often reduces to variance estimates we may use to make private normal
approximations, we present a data-driven bootstrap-based private estimator
for the mean-squared error (MSE) of $\unstudstatpriv$ in
\Cref{algorithm:blb-var}, $\blbvar$. We now use the bag of little bootstraps
(BLB)~\citep{KleinerTaSaJo14} to construct $\nsubs$ estimates of the MSE on
disjoint subsets of data and aggregates them using a private median
mechanism ($\privmed$, Alg.~\ref{algorithm:priv-median} in
Appendix~\ref{appendix:extra-algs}).  \citet{CovingtonHeHoKa21} investigate
a related algorithm, combining $\nsubs$ variance estimates using the
\textsc{CoinPress} private mean estimation
algorithm~\cite{BiswasDoKaUl20}. Because the mean is so challenging to make
private---it is inherently unstable---using the more naturally robust median
yields stronger performance.  The private median mechanism requires a
smoothing parameter $\smoothparam$ and interval
$[R_l, R_u]$ in which the median should lie to which it is
relatively insensitive; we usually take $\smoothparam = n^{-1}$.

\newcommand{\sigmamax}{\sigma_{\max}}

\begin{algorithm}
  \caption{\label{algorithm:blb-var} Private BLB variance estimation
    $\prn{\blbvar}$}
    \DontPrintSemicolon
    \SetKwInOut{Input}{Input}
    \Input{Estimators $\param(\cdot)$ and
      $\tparam(\cdot)$, sample $P_n$,
      subset size $\subsize$, gross upper bound $\sigmamax^2$ on
      second moment,
      smoothing parameter $\rho$, privacy budget $\diffp$}
    Draw $\nsub = \floor{n / \subsize}$ disjoint
    subsamples $P\subsampleind{1}_\subsize, \ldots, P\subsampleind{\nsub}_\subsize$
    of $P_n$ of size $\subsize$\;
    \For{$i \in [s]$}{
      $v_i \leftarrow \bootvaralg(\param(\cdot),
      \tparam(\cdot), P\subsampleind{i}_\subsize,n)$\;
    }
    \Return $\privmed(v_{1:\nsub}, \diffp, \smoothparam, [0, \sigmamax^2])$\;
    \label{line:return-private-median}
\end{algorithm}

\begin{algorithm}
  \caption{\label{algorithm:boot-var}Bootstrap for
    mean-square-error estimation $\prn{\bootvaralg}$}
  \DontPrintSemicolon
  \SetKwInOut{Input}{Input}
  \Input{Estimators $\param(\cdot)$ and $\tparam(\cdot)$,
    sample $P_\subsize$, resample size $n$}
  $\hparam\leftarrow \param(P_\subsize)$
  ~~~~ \emph{{\# saved plug-in estimator}}\;
  \For{$i \in [\nmontecarlo]$}{
    Draw resample $P_{\subsize,n}^*$ of
    size $n$ i.i.d.\ with replacement from $P_\subsize$\;
    $\wt{u}_i \leftarrow \sqrt{n}(\tparam(P_{n,\subsize}^*)
    - \hparam)$\;
  }
  \Return
  $\frac{1}{\nmontecarlo} \sum_{i = 1}^{\nmontecarlo} \ltwos{\wt{u}_i}^2$
\end{algorithm}

As we did in Theorem~\ref{thm:blb-abthreshmed}, we can provide
utility guarantees for $\blbvar$ whenever the bootstrap variance
estimate is reasonably accurate with constant probability.

\begin{assumption}\label{assumption:bootstd-accuracy}
  Let $\what{v}$ denote the output of
  $\bootvaralg$ on an input sample of size $\subsize$
  with resampling size $n$
  for the estimators $\param(\cdot)$ and
  $\tparam(\cdot)$, and let
  $\sigma^2_n = n \E[\ltwos{\tparam(P_n) - \param(P)}^2]$.
  There exists $\varacc < \infty$ such that
  \begin{equation*}
    \P\prn{\abs{\sigma^2_n - \what{v}} \le \varacc} \geq \frac{3}{4}.
  \end{equation*}
\end{assumption}
\noindent
In our more explicit accuracy analyses in Section~\ref{sec:val-rates} (see
the discussion following Proposition~\ref{proposition:edgeworth-moment}), we
show that it is frequently the case that $\varacc \lesssim
\frac{1}{\sqrt{\subsize}} + \frac{1}{\diffp\sqrt{n}}$.
More broadly, achieving $\varacc = o(1)$ is not difficult, so that
a normal approximation using $\what{v}$ provides an asymptotically
valid confidence set as a corollary of the next theorem
(see Corollary~\ref{corollary:coverage-from-variance}
in Section~\ref{sec:asymp-norm-consistency}).
Regardless, under
Assumption~\ref{assumption:bootstd-accuracy},
we have the following result, whose proof we defer to
Appendix~\ref{sec:proof-blb-var}.
\begin{theorem}
  \label{thm:blb-var}
  The output of $\blbvar$ is $\diffp$-differentially private.  Additionally,
  let Assumption~\ref{assumption:bootstd-accuracy} hold, and assume
  $\sigmamax \ge \sigma_n$. Let $\beta > 0$.
  If the number of
  subsamples $\nsubs = \floor{n/\subsize} \ge \max\{\frac{16}{\diffp} \log
  \frac{\sigmamax^2}{\failprob \smoothparam}, 32 \log \frac{1}{\failprob}\}$,
  then with probability at least $1 - \beta$
  \begin{equation*}
    \abs{\sigma^2_n - \blbvar(\param(\cdot), \tparam(\cdot), P_n,
      \subsize, \sigmamax, \smoothparam, \diffp)}
    \leq \varacc + \smoothparam.
  \end{equation*}
\end{theorem}

The choices of the particular constants Theorem~\ref{thm:blb-var} requires
to guarantee $\blbvar$'s performance are not particularly onerous. The bound
$\sigmamax^2$ on the possible variance $\sigma_n^2$ and the smoothing
parameter $\smoothparam$ that the private median algorithm require appear
only as logarithmic constraints on the number of subsamples; consequently,
choosing them to be polynomially small (or large) in $n$ has little
effect. Practically, the choices $\smoothparam = \frac{1}{n}$ and
$\sigmamax^2 = n^2$ are effective, with number of subsamples $\nsubs = K
\frac{\log n}{\diffp}$ for a constant $K$.  (In our experiments, we take $K
= 10$.)  We comment more on hyperparameters in Appendix~\ref{appen:expts}.

\newcommand{\privtononpriv}{\Delta_{\subsize,n}}

\section{Validity and Asymptotic Rates}\label{sec:val-rates}

The utility results in Section~\ref{sec:priv-conf-int} rely on the accuracy
of bootstrapped private estimators
(Assumptions~\ref{assumption:bootcdf-accuracy} and
\ref{assumption:bootstd-accuracy}). Here, we show that this holds for a wide
class of private estimators for which there exist bootstrappable non-private
estimators, so long as the (typical) distance between the private and
non-private estimators are small.  First, we show that such private
estimators are asymptotically normal and their bootstrapped distributions
are consistent estimates of their distribution
(Proposition~\ref{proposition:asymp-norm-consistency}). For problems in
which the non-private estimators admit Edgeworth expansions, we quantify
error rates for bootstrapped private estimators, which in turn allows us to
instantiate rates at which $\cdfacc$ and $\varacc \to 0$ as $n$ grows in
Assumptions~\ref{assumption:bootcdf-accuracy}
and~\ref{assumption:bootstd-accuracy}, thus closing the loop on the
performance of our proposed algorithms.

\subsection{Asymptotic Normality and Consistency}
\label{sec:asymp-norm-consistency}

We first define the consistency of resampling-based estimators of a
distribution, where we let $\unstudstatest^*$ denote an otherwise arbitrary
quantity that is a function of only the empirical sample $P_n$; the typical
choice will be a resampled estimate $\unstudstatest^* = \sqrt{n}(
\param(P_\subsize\subsample) - \param(P_{\subsize,n}^*))$ as in the
approximation~\eqref{eqn:key-blb-approximation} when we use the Bag of
Little Bootstraps (BLB). We make many statements
\emph{conditionally almost surely}, which means that
they hold conditional on the sequence $\{P_n\}_{n \in \N}$ of empirical
distributions, for almost all sequences $\{P_n\}_{n \in \N}$ when
$P_n$ is drawn i.i.d.\ $P$. (See~\citet[Ch.~23]{VanDerVaart98}
for a discussion of these modes of convergence.)
For vectors $u, v \in \R^d$, we write $u \le v$ to mean
$u_j \le v_j$ for each $j$.
\begin{definition}
  \label{def:ks-consistency}
  A resampling-based estimator $\unstudstatest^*$ is \emph{consistent}
  relative to the Kolmogorov-Smirnov distance
  \emph{(KS-consistent)}
  if
  \begin{equation*}
    \sup_{u}
    \left|\P(\unstudstatest \le u) - \P(\unstudstatest^* \le u \mid P_n)
    \right| \cp 0
  \end{equation*}
  conditionally almost surely. It is \emph{$L^2$-consistent} if
  for some $\gamma > 0$,
  \begin{equation*}
    \E\left[\norm{\unstudstatest^*}^{2 + \gamma} \mid P_n\right]
    = O_P(1)
  \end{equation*}
  conditionally almost surely.
\end{definition}
\noindent
The uniform integrability the second condition provides guarantees that
if $\unstudstatest \cd \normal(0, \Sigma)$, then
\begin{equation*}
  \E\left[\unstudstatest^* {\unstudstatest^*}^\top \mid P_n\right] \cp \Sigma
\end{equation*}
conditionally almost surely~\cite[Sec.~3.1.6]{ShaoTu95}.
As simple examples, BLB resampling
estimators~\eqref{eqn:key-blb-approximation} are frequently KS-consistent,
for example, for means and smooth functions of means, or the sampling
distribution of suitably smooth M-estimators
(see~\cite[Theorem~1]{KleinerTaSaJo14} and~\cite[Thm.~23.4]{VanDerVaart98},
as well as Section~\ref{sec:erm} to come), so long as the subsample size
$\subsize \to \infty$. We note in passing that
for the BLB, our subsampling means we have the distributional equalities
\begin{equation*}
  \blbstatistic \mid P_n
  \eqdist \blbstatistic \mid P_\subsize
  \eqdist \blbstatistic \mid P_\subsize\subsample,
\end{equation*}
so we use them interchangably. (And similarly for $\tblbstatistic$.)

In the next theorem, to obtain KS-consistency of resampling estimators we
require accuracy guarantees for the private estimators.
In this case, let $\noisevar_n$ be a sequence of random noise
variables, so that for a sample size $n$,
the private estimator takes the form
$\tparam(P_n) = \tparam(P_n, \noisevar_n)$.
We say that the estimator $\tparam$
is \emph{rate-$\sqrt{n}$-resampling consistent} for subsample
size $\subsize$ if
\begin{equation}
  \label{eqn:resample-consistency}
  \sqrt{n} \prn{\tparam(P_{\subsize,n}^*, \noisevar_n) - \param(P^*_{\subsize,n})}
  \cp 0
\end{equation}
conditionally almost surely
and $\tparam(P_n) - \param(P_n) = o_P(1/\sqrt{n})$,
and that it is \emph{$L^2$-resampling consistent} for subsample
size $\subsize$ if for some $\gamma > 0$,
\begin{equation}
  \label{eqn:l2-resample-consistency}
  \E\left[\norm{\sqrt{n}\big(\tparam(P_{\subsize, n}^*, \noisevar_n)
      - \param(P_{\subsize, n}^*)\big)}^{2 + \gamma} \mid P_\subsize \right]
  = O_P(1)
\end{equation}
conditionally almost surely.
The conditions~\eqref{eqn:resample-consistency}
and~\eqref{eqn:l2-resample-consistency} are abstract, so we provide one
example here, providing more elaborate discussion on empirical risk
minimization in Section~\ref{sec:erm}.

\begin{example}[Resampling consistency for mean estimation]
  \label{example:resampling-mean}
  Let $\param(P) = \E_P[X] = \int x dP(x)$ be the mean of its argument,
  and assume that the data has support $[-b, b]$. Then the
  standard Laplace mechanism, for which
  $\tparam(P_n, \noiseval) = P_n X + \noiseval = \wb{X}_n + \noiseval$, takes
  $W_n \sim \laplace(0, \frac{2b}{\diffp n})$
  and is $\diffp$-differentially
  private. Then
  \begin{equation*}
    \sqrt{n} \prn{\tparam(P_{\subsize,n}^*, \noisevar_n)
      - \param(P_{\subsize,n}^*)}
    = \sqrt{n} \noisevar_n \sim \laplace\left(0, \frac{2 b}{\diffp \sqrt{n}}
    \right).
  \end{equation*}
  This allows us to take $\diffp = \diffp_n$ decreasing to 0,
  and then so long as $\diffp \sqrt{n}$ increases to $\infty$ at a polynomial
  rate,
  the Borel-Cantelli lemmas immediately give that
  $\sqrt{n} \prninline{\param(P_{\subsize,n}^*, \noisevar_n)
    - \param(P_{\subsize,n}^*)} \cas 0$.

  For the $L^2$ consistency~\eqref{eqn:l2-resample-consistency},
  note that $\tparam(P_{\subsize,n}^*, \noisevar_n) - \param(P_{\subsize, n}^*)
  = \noisevar_n$ and for $\gamma > 0$,
  \begin{equation*}
    n \E\left[|\noisevar_n|^{2 + \gamma} \mid P_\subsize\right]
    = \frac{(2b)^{2 + \gamma}}{n^{1 + \gamma} \diffp^{2 + \gamma}}
    \Gamma(3 + \gamma).
  \end{equation*}
  The uniform
  integrability~\eqref{eqn:l2-resample-consistency} thus holds
  as well.
\end{example}

More generally, $\diffp$-DP estimators induce private error roughly scaling
as $\tparam(P_n) - \param(P_n) = O_P(\frac{1}{n \diffp})$. Additive noise
mechanisms using global sensitivity, such the Laplace mechanism in
Example~\ref{example:resampling-mean}, satisfy this, as do more
sophisticated mechanisms in ``typical'' cases, such as the
inverse-sensitivity mechanism~\cite{AsiDu20}.  The resampling
consistency~\eqref{eqn:resample-consistency} is then frequently a byproduct
of the privacy and accuracy analysis of a given mechanism (and see
Sec.~\ref{sec:erm} to come).  The assumptions on private and non-private
estimators in the next proposition, which gives sufficient conditions under
which our bootstrap-type methods yield asymptotically correct confidence
sets, thus hold in most common cases.

\begin{proposition}
  \label{proposition:asymp-norm-consistency}
  Let $\param(P_n)$ satisfy $\unstudstatest =
  \sqrt{n}(\param(P) - \param(P_n)) \cd \normal(0, \Sigma)$, and assume the
  BLB statistic $\blbstatistic$ is KS-consistent for $\unstudstatest$
  (Def.~\ref{def:ks-consistency}).  Let $\wt{\param}(\cdot)$ be a
  $\sqrt{n}$-subsampling-consistent~\eqref{eqn:resample-consistency}
  estimator for $\param(\cdot)$.  Then $\sqrt{n} (\tparam(P_n) - \param(P))
  \cd \normal(0, \Sigma)$, the resampled statistic
  \begin{equation*}
    \tblbstatistic
    \defeq \sqrt{n}\left(\param(P_\subsize\subsample)
    - \tparam(P_{\subsize,n}^*)\right) \cd \normal(0, \Sigma)
  \end{equation*}
  conditionally almost surely, and it is KS-consistent.
\end{proposition}
\begin{proof}
  The asymptotic normality of $\tparam(P_n)$ is immediate, because
  \begin{equation*}
    \sqrt{n} (\param(P) - \tparam(P_n))
    = \sqrt{n} (\param(P) - \param(P_n))
    + \underbrace{\sqrt{n} (\param(P_n) - \tparam(P_n))}_{\cp 0}
    \cd \normal(0, \Sigma_{\param})
  \end{equation*}
  by Slutsky's lemmas~\cite[Lemma 2.8]{VanDerVaart98}.
  To obtain the
  consistency of the resampled boostrap distributions,
  it is sufficient to show that
  $\tblbstatistic \cd \normal(0, \Sigma)$ conditionally on
  $\{P_n\}$ almost surely~\cite[Eq.~(23.2)]{VanDerVaart98}.
  For this, note that
  \begin{align*}
    \tblbstatistic
    = \sqrt{n} \left(\param(P_\subsize) - \tparam(P_{\subsize,n}^*)\right)
    & = \sqrt{n} \left(\param(P_\subsize)
    - \param(P_{\subsize,n}^*) \right)
    + \sqrt{n} \left(
    \param(P_{\subsize,n}^*) - \tparam(P_{\subsize,n}^*)\right) \\
    & = \blbstatistic + 
    \sqrt{n} \left(
    \param(P_{\subsize,n}^*) - \tparam(P_{\subsize,n}^*)\right).
  \end{align*}
  By assumption, $\blbstatistic \cd \normal(0, \Sigma_\param)$ conditionally
  almost surely, while the assumption that $\tparam$ is $\sqrt{n}$-resampling
  consistent~\eqref{eqn:resample-consistency}
  guarantees the rightmost term converges in probability
  to 0 conditionally almost surely.
\end{proof}

Coupling Proposition~\ref{proposition:asymp-norm-consistency} with
Theorem~\ref{thm:blb-abthreshmed} allows us to show the consistency of
private confidence intervals. In the corollary, we take
$I_t = \{u \mid \ltwo{u} \le t / \sqrt{n}\}$ and $t \le n$, though
other choices are possible.

\begin{corollary}
  \label{corollary:basic-percentile-consistency}
  Let $C_\alpha(P_n)$ be the confidence set in
  Corollary~\ref{corollary:coverage-from-percentile}
  and the conditions of Proposition~\ref{proposition:asymp-norm-consistency}
  hold. Then
  $\P(\param(P) \in C_\alpha(P_n)) \to 1 - \alpha$.
\end{corollary}
\begin{proof}
  By Proposition~\ref{proposition:asymp-norm-consistency}
  and assumption,
  we have both
  \begin{equation*}
    \sup_t \abs{\P(\blbstatistic \in I_t \mid P_\subsize)
      - \P(\tblbstatistic \in I_t \mid P_\subsize)}
    \cp 0,
    ~~
    \sup_t \abs{\P(\blbstatistic \in I_t \mid P_\subsize)
      - \P(\unstudstatest \in I_t)} \cp 0.
  \end{equation*}
  To show that the conditions of
  Assumption~\ref{assumption:bootcdf-accuracy} hold, then,
  take $\what{p} = \P(\tblbstatistic \in I_t \mid P_\subsize)$,
  the resampling estimate in Line~\ref{line:coverage-estimate} of
  Alg.~\ref{algorithm:blb-quant} (assuming as usual there is
  no Monte Carlo sampling error). Then by the triangle inequality,
  for any $\cdfacc > 0$ we have
  \begin{equation*}
    \P\prn{\abs{\what{p} - \P(\unstudstatest \in I_t)}
      \ge \cdfacc \mid P_n} \cp 0
  \end{equation*}
  conditionally almost surely, and Assumption~\ref{assumption:bootcdf-accuracy}
  holds. It only remains to show that there
  exists an accurate enough set as Theorem~\ref{thm:blb-abthreshmed}
  requires. But because $\unstudstatpriv \cd \normal(0, \Sigma)$,
  we know that there are sequences of $t_n = O(\sqrt{n})$ such that
  $\P(\ltwos{\unstudstatpriv} \le t_n / \sqrt{n}) \to 1 - \alpha$.
\end{proof}

We can also give consistency of the moment estimators.
\begin{proposition}
  \label{proposition:moment-consistency}
  Let the conditions of Proposition~\ref{proposition:asymp-norm-consistency}
  hold, and additionally assume that $\blbstatistic$ is
  $L^2$-consistent (Def.~\ref{def:ks-consistency}).  Assume that
  $\tparam(\cdot)$ is $L^2$-resampling
  consistent~\eqref{eqn:l2-resample-consistency}.  Then the boostrap
  second moment estimate
  \begin{equation*}
    \wt{\Sigma}_{\subsize,n}^* \defeq n\E\left[
      (\tparam(P_{\subsize,n}^*) - \param(P_\subsize))
      (\tparam(P_{\subsize,n}^*) - \param(P_\subsize))^\top
      \mid P_\subsize\right]
    \cp \Sigma
  \end{equation*}
  conditionally almost surely.
\end{proposition}
\begin{proof}
  Because we already have KS-consistency
  (Proposition~\ref{proposition:asymp-norm-consistency}),
  all we require is the uniform integrability of
  $n \norms{\tparam(P_{\subsize,n}^*) - \param(P_\subsize)}^2$
  under the sequence of (random) measures $P_\subsize$;
  see~\citet[Sec.~3.1.6]{ShaoTu95}. But this
  follows because for some $\gamma > 0$, we have
  \begin{align*}
    \lefteqn{\E\left[\norms{\sqrt{n} (\tparam(P_{\subsize,n}^*)
          - \param(P_\subsize))}^{2 + \gamma} \mid P_\subsize\right]} \\
    & \le 2^{1 + \gamma}
    \left(\E\left[\norms{\sqrt{n} (\tparam(P_{\subsize,n}^*)
        - \param(P_{\subsize,n}^*))}^{2 + \gamma} \mid P_\subsize\right]
    + \E\left[\norms{\sqrt{n} (\param(P_{\subsize,n}^*)
        - \param(P_\subsize))}^{2 + \gamma} \mid P_\subsize\right]
    \right)
  \end{align*}
  and each term on the right is $O_P(1)$ conditionally almost surely
  by assumption.
\end{proof}

In passing, we note a corollary of Theorem~\ref{thm:blb-var}
and Proposition~\ref{proposition:moment-consistency}
that is
analogous to Corollary~\ref{corollary:basic-percentile-consistency}.  In the
corollary, which focuses on the one parameter case---though this may be a
single parameter in a larger vector-valued model---we say that a private
estimator $\tparam(\cdot)$ has typical limiting behavior if for $\sigma_n^2$
and $\varacc$ in Assumption~\ref{assumption:bootstd-accuracy}, we have
\begin{equation*}
  \sqrt{n} \frac{\tparam(P_n) - \param(P)}{\sigma_n}
  \cd \normal(0, 1)
  ~~ \mbox{and} ~~
  \varacc \to 0
\end{equation*}
as $n \to \infty$. See, for example, \citet{Smith11} for the asymptotic
normality; Proposition~\ref{proposition:moment-consistency} guarantees that
$\varacc \to 0$.  Then taking $\smoothparam = \frac{1}{n}$ and $\sigmamax^2
= n^2$ in Algorithm~\ref{algorithm:blb-var}, we have the following
corollary.
\begin{corollary}
  \label{corollary:coverage-from-variance}
  Let $\tparam(\cdot)$ be an $(\diffp, \delta)$-differentially
  private estimator with typical limiting
  behavior for the population quantity $\param(P)$,
  and let $\subsize = \subsize(n)$ be any subset size
  for which $\subsize(n) \to \infty$ and
  $\subsize(n) \ll \frac{n \diffp}{\log n}$.
  Let $\wt{\sigma}_n^2$ be the output of $\blbvar$, and
  define the confidence set
  \begin{equation*}
    C_\alpha(P_n) \defeq \tparam(P_n) +
    \left[-n^{-1/2} \wt{\sigma}_n \cdot z_{1-\alpha/2},
      n^{-1/2} \wt{\sigma}_n \cdot z_{1-\alpha/2}\right].
  \end{equation*}
  Then $C_\alpha(P_n)$ is $(2 \diffp, \delta)$-differentially private, and
  \begin{equation*}
    \lim_{n \to \infty} \P(\param(P) \in C_\alpha(P_n)) = 1 - \alpha.
  \end{equation*}
\end{corollary}

\subsection{Edgeworth expansions and asymptotic rates}

Our goal in this section is to prove asymptotic rates for the error
of our BLB
private estimators. We work under the standard assumption that the
non-private estimators admit Edgeworth
expansions~\cite{Hall92,KleinerTaSaJo14}, a standard technical tool for
providing asymptotic guarantees for the bootstrap.  For this section, we
assume that the statistic $\param$ is one-dimensional, as
otherwise the notation becomes too cumbersome.
Recall that a statistic $\unstudstatest \cd \normal(0, \sigma^2)$
admits an
Edgeworth expansion of order $k$ if
uniformly in $t \in \R$,
\begin{subequations}
  \label{eqn:edgeworth}
  \begin{equation}
    \label{eqn:edgeworth-pop}
    \P\left(\frac{\unstudstatest}{\sigma} \le t\right)
    = \Phi\left(t\right)
    + \bigg(\sum_{i = 1}^k n^{-i/2} p_i\left(t \right)
    \bigg) \phi(t)
    + O\left(n^{-\frac{k+1}{2}}\right),
  \end{equation}
  where $p_i$ is a polynomial of degree $3i - 1$ depending on the first $i +
  2$ moments of $P$, which is even or odd as $i - 1$ is even or odd (see,
  e.g., \citet[Chapter 2]{Hall92}).  Similarly, the resampled statistic
  $\blbstatistic$ admits an Edgeworth expansion of order $k$ if there is a
  sequence of variances $\sigma^2(P_\subsize) = \sigma^2 + O_P(m^{-1/2})$
  with $\sigma^2(P_\subsize) > 0$ eventually (w.p.\ 1)\footnote{The accuracy
  assumption $\sigma^2(P_\subsize) = \sigma^2 + O_P(m^{-1/2})$ is unique,
  but typically holds, as we discuss.}  and for which $\blbstatistic /
  \sigma(P_\subsize) \cd \normal(0, 1)$ and
  \begin{equation}
    \label{eqn:edgeworth-blb}
    \P\left(\frac{\blbstatistic}{\sigma(P_\subsize)} \le t \mid P_\subsize\right)
    = \Phi(t) + \bigg(\sum_{i = 1}^k n^{-i/2}
    \what{p}_i(t) \bigg) \phi(t) +
    \remainder_{n,k}(t),
  \end{equation}
\end{subequations}
where $\remainder_{n,k}(t) = O_P(n^{-\frac{k + 1}{2}})$ uniformly in $t$.
See~\cite[Ch.~3.3]{Hall92} and \citet[Theorem 2]{KleinerTaSaJo14} for
justification of the expansion in power $n^{-i/2}$, which roughly follows
because $\blbstatistic$ is based on a sample of size $n$ drawn
i.i.d.\ $P_\subsize$. Here, $\what{p}_i$ depend implicitly on the subsampled
distribution $P_\subsize$ and are degree $3i-1$ polynomials.

\newcommand{\priverr}{\eta}  %
\newcommand{\privprob}{\nu}  %

If the private mechanisms $\wt{\param}$ have sufficient accuracy, we
can leverage Edgeworth expansions of $\param(P_n)$ to
obtain stronger convergence properties for the BLB-resampled private
estimator.  We say a (private) estimator $\tparam$ is $(\priverr,
\privprob)$-accurate for $\param(\cdot)$ if
\begin{equation*}
  \P\left(|\tparam(P_n) - \param(P_n)| \ge \priverr\right) \le \privprob
  ~~ \mbox{and} ~~
  \P\left(|\tparam(P_{\subsize,n}^*) - \param(P_{\subsize,n}^*)|
  \ge \priverr \mid P_\subsize\right) \le \privprob
  ~~ \mbox{eventually}.
\end{equation*}
Returning to Example~\ref{example:resampling-mean},
we discuss how the sample mean admits these expansions.

\begin{example}[The Laplace mechanism for the mean,
    Example~\ref{example:resampling-mean} continued]
  \label{example:resampling-mean-2}
  Let $X_i$ be data with $|X_i| \le b$.
  For the mean $\param(P) = \E_P[X]$,  $\param(P_n)
  = \wb{X}_n$, the Laplace mechanism
  sets $\tparam(P_n) = \wb{X}_n + \noisevar_n$
  for $\noisevar_n \sim \laplace(0, \frac{2b}{n \diffp})$.
  Then we see that $\P(|\noisevar_n| \ge \priverr)
  = 2 \exp(-\frac{n \diffp}{2b} \priverr)$,
  so for any $\privprob > 0$ the choice
  $\priverr = \frac{2b}{n \diffp} \log \frac{2}{\privprob}$
  yields $\P(|\tparam(P_n) - \param(P_n)| \ge \priverr)
  \le \privprob$.

  For the resampling distributions, given a sample $P_\subsize$, we can
  directly compute the variance $\sigma^2(P_\subsize)$ in the Edgeworth
  expansion~\eqref{eqn:edgeworth-blb}:
  we have $\sigma^2(P_\subsize) = \var_{P_\subsize}(X) = \frac{1}{\subsize}
  \sum_{i = 1}^\subsize X_i^2 - \wb{X}_\subsize^2$, and a calculation by the
  delta method shows that
  \begin{equation*}
    \sqrt{\subsize}(\sigma^2(P_\subsize) - \sigma^2)
    \cd \normal(0, \var(X^2) + 4 \E[X]^2 \sigma^2).
  \end{equation*}
  That is, $\sigma^2(P_\subsize) = \sigma^2 + O_P(\subsize^{-1/2})$.
\end{example}

\noindent
As in our discussion of the percentile bootstrap in the
previous section, we give a more elaborate example
for M-estimators in Section~\ref{sec:erm}
to come.

For sufficiently accurate private estimators, we then obtain the following
convergence guarantee on the percentile bootstrap method.
\begin{proposition}
  \label{proposition:edgeworth}
  Let $\unstudstatest$ be the (unstudentized)
  statistic of an estimator $\param(P_n)$
  admitting Edgeworth expansions~\eqref{eqn:edgeworth} of order $k = 1$.
  Let $\tparam$ be  $(\priverr, \privprob)$-accurate for $\param(P_n)$
  and $\tblbstatistic = \sqrt{n}(\param(P_\subsize) -
  \tparam(P_{\subsize,n}^*))$ as usual.
  Then for any interval $I \subset \R$,
  \begin{equation*}
    \abs{\P\left(\unstudstatpriv \in I\right)
      -
      \P\prn{\tblbstatistic \in I \mid P_\subsize}
    }
    \le O_P\left(\frac{1}{\sqrt{\subsize}}
    + \sqrt{n} \cdot \priverr\right) + \privprob.
  \end{equation*}
\end{proposition}
\noindent
See Appendix~\ref{sec:proof-edgeworth} for a proof of this proposition.

By considering the ``typical'' scaling of the privacy errors
$\priverr$ and probability of failure $\privprob$,
we can show how Proposition~\ref{proposition:edgeworth}
connects with the accuracy assumptions we have thus
far used to obtain consistency. In this case,
we say that $\tparam$ is has \emph{typical private error} if
\begin{equation}
  \label{eqn:typical-error}
  \priverr = O\left(\frac{\log n}{n \diffp}\right)
  ~~ \mbox{and} ~~
  \privprob = \frac{1}{n}.
\end{equation}
Example~\ref{example:resampling-mean-2} shows this holds for the mean;
Section~\ref{sec:erm} to come shows how this holds for empirical
risk minimization.

\begin{corollary}
  \label{corollary:typical-edgeworth-consistency}
  Let the conditions of Proposition~\ref{proposition:edgeworth}
  hold, $\tparam$ have typical private error~\eqref{eqn:typical-error},
  and $\diffp = O(1)$.
  Then
  the confidence set $C_\alpha(P_n)$ of
  Corollary~\ref{corollary:coverage-from-percentile}
  satisfies
  \begin{equation*}
    \P(\param(P) \in C_\alpha(P_n))
    = 1 - \alpha + O\left(\frac{\log n}{\diffp \sqrt{n}}
    \right).
  \end{equation*}
\end{corollary}
\begin{proof}
  Proposition~\ref{proposition:edgeworth} shows that
  Assumption~\ref{assumption:bootcdf-accuracy} holds
  with $\cdfacc = O(1 / \sqrt{\subsize} + \log n / \diffp \sqrt{n})$.
  In Theorem~\ref{thm:blb-abthreshmed},
  we take the number of subsamples $\nsubs = K \frac{\log n}{\diffp}$
  for some numerical constant $K$, meaning that the subsample
  size is $\subsize = \frac{n}{\nsubs}
  \gtrsim \frac{n \diffp}{\log n}$,
  then use that the Gaussian CDF is Lipschitz to verify its other conditions.
\end{proof}

We can also give accuracy guarantees for the bootstrap second
moment estimates.
In this case, we leverage a $k$th order moment-based Edgeworth expansion
in analogy to~\eqref{eqn:edgeworth}, where
for $\unstudstatest \cd \normal(0, \sigma^2)$
\begin{subequations}
  \label{eqn:edgeworth-variance}
  \begin{equation}
    \label{eqn:edgeworth-variance-pop}
    \E[\unstudstatest^2]
    = \sigma^2
    \Big(1 - 4 \sum_{j = 1}^{k-1} n^{-j}
    \int_0^\infty t p_j(t) \phi(t) dt
    + O(n^{-k})\Big),
  \end{equation}
  where $p_j$ are polynomials, and
  for $\sigma(P_\subsize)$ as in the equalities~\eqref{eqn:edgeworth},
  we have
  \begin{equation}
    \label{eqn:edgeworth-variance-blb}
    \E\left[(\blbstatistic)^2 \mid P_\subsize\right]
    = \sigma^2(P_\subsize)
    \Big(1 - 4 \sum_{j = 1}^{k-1} n^{-j} \int_0^\infty t \what{p}_j(t) \phi(t)
    dt + O_P(n^{-k})\Big),
  \end{equation}
\end{subequations}
conditionally almost surely (where $\what{p}_j$ are empirical polynomials).

\begin{proposition}
  \label{proposition:edgeworth-moment}
  Let $\unstudstatest$ be the (unstudentized) statistic
  of an estimator $\param(P_n)$ admitting Edgeworth variance
  expansions~\eqref{eqn:edgeworth-variance} of order $k = 1$. Let
  $\tparam$ be any other estimator, and define
  the error
  $\privtononpriv \defeq \tparam(P_{\subsize,n}^*) - \param(P_{\subsize,n}^*)$.
  Then conditionally almost surely,
  \begin{equation*}
    \left|\sigma^2(P_\subsize)
    - \bootvaralg(\param(\cdot), \tparam(\cdot), P_\subsize, n)\right|
    = O_P\left(\frac{1}{\sqrt{\subsize}} + \sqrt{n} \sqrt{\E[
        \privtononpriv^2 \mid P_\subsize]}\right).
  \end{equation*}
\end{proposition}
\noindent
See Appendix~\ref{sec:proof-edgeworth-moment} for a proof.

As in the case of Proposition~\ref{proposition:edgeworth},
we can control the typical scaling of the error $\privtononpriv$,
which in turn allows more concrete accuracy guarantees.
In analogy to Corollary~\ref{corollary:typical-edgeworth-consistency}
and the conditions~\eqref{eqn:typical-error},
we say that $\tparam$ has \emph{typical squared error}
if
\begin{equation*}
  \E\left[\big(\tparam(P_{\subsize,n}^*) - \param(P_{\subsize,n}^*)\big)^2
    \mid P_\subsize\right] = O_P\left(\frac{1}{n \diffp}\right).
\end{equation*}
Example~\ref{example:resampling-mean-2} makes clear that the Laplace
mechanism for the sample mean has this scaling; see
also Section~\ref{sec:erm} to come on empirical risk minimization.
By connecting this condition with
Assumption~\ref{assumption:bootstd-accuracy} and applying
Theorem~\ref{thm:blb-var}, we can then obtain a concrete
accuracy guarantee.
\begin{corollary}
  \label{corollary:typical-edgeworth-moment-consistency}
  Let $\wt{\param}$ be have typical squared error and
  take $\sigmamax^2 = n^2$. Then
  there exists $N$ such that $n \ge N$ implies
  with probability at least $1 - 1/n^2$,
  \begin{equation*}
    \abs{\sigma_n^2 - \blbvar(\param(\cdot),
      \tparam(\cdot), P_n, \subsize, \sigmamax, \rho, \diffp)}
    \le \sqrt{\frac{\log n}{n \diffp}}.
  \end{equation*}
\end{corollary}
\begin{proof}
  Let
  $v^2(P_\subsize)$ be the output of $\bootvaralg$ as
  in Proposition~\ref{proposition:edgeworth-moment}.
  Then by the proposition and
  the typical squared error condition, we see that for
  $\varacc = O(1/\sqrt{\subsize} + 1 / \sqrt{n \diffp})$,
  there is some $N$ such that $n \ge N$ implies
  $\P(|v^2(P_\subsize) - \sigma^2(P_\subsize)| \ge \varacc) \le \frac{1}{4}$.
  We also have $\sigma^2(P_\subsize) = \sigma^2 + O_P(1/\sqrt{\subsize})$
  by the assumed Edgeworth expansions,
  so taking $\nsubs = K \frac{\log n}{\diffp}$ for a numerical
  constant $K$ as in
  Theorem~\ref{thm:blb-var}
  yields $\subsize = \frac{n}{\nsubs} = \frac{n \diffp}{\log n}$,
  and Assumption~\ref{assumption:bootstd-accuracy} holds.
  Substitute.
\end{proof}

\subsubsection{On higher-order accuracy and studentization}

As a brief remark for readers familiar with the classical literature on
Edgeworth expansions and bootstrap accuracy~\cite{Hall92},
when we seek symmetric
sets $I_t = [-ht, ht]$, the classical bootstrap with studentization achieves
coverage error $O(1/n)$ rather than the $O(1/\sqrt{n})$ above. The failure
to achieve the faster rate follows because we do not assume we have a
consistent variance estimate $\wt{\sigma}$ of $\tparam$.  Thus, comparing
the estimates~\eqref{eqn:edgeworth}, we achieve roughly
\begin{align*}
  \P(\tblbstatistic \in [-t, t] \mid P_\subsize)
  - \P(\unstudstatpriv \in [-t, t])
  & = 2 \Phi(t \sigma(P_\subsize))
  - 2 \Phi(t \sigma)
  + O(1/n) \\
  & = 2 \phi(t \sigma) t \cdot (\sigma(P_m) - \sigma)
  + O\left((\sigma(P_\subsize) - \sigma)^2 + 1/n\right)
\end{align*}
because $p_1$ is even. So the error necessarily scales with
the difference $\sigma(P_\subsize) - \sigma$, which typically is
roughly of order $1 / \sqrt{\subsize}$. One could
in principal use a private variance estimates $\wt{\sigma}^2(P_\subsize)$
for private estimators applied to $P_\subsize\subsampleind{i}$, plausibly
computed via an adaptation
of
Algorithm~\ref{algorithm:blb-var}, and resample
the studentized statistic
\begin{equation*}
  \wt{U}_{\subsize,n}^{*,\textup{student}}
  \defeq \sqrt{n} \frac{\tparam(P_{\subsize,n}^*) - \param(P_\subsize)}{
    \wt{\sigma}(P_\subsize)}
  \cd \normal(0, 1).
\end{equation*}
But this requires more care in the analysis, as we have multiple
levels of variance estimation and resampling, so we leave it to
future work.

\subsection{Empirical risk minimization, objective perturbation, and accuracy}
\label{sec:erm}

In this section, we provide a more sophisticated set of examples that depart
from simple mean and other moment-based estimators, considering M-estimators
(or empirical risk minimization).  Here, for a convex loss $\loss_\param(x)$
measuring the performance of a parameter $\param$ for predicting on data $x
\in \mc{X}$, the goal is to find $\param(P)$ minimizing the population risk
\begin{equation*}
  \poploss_P(\param) \defeq \E_P[\loss_\param(X)] = P \loss_\param,
\end{equation*}
where we use the empirical process notation $P f
= \int f(x) dP(x)$, so that for the empirical measure $P_n$, $P_n f =
\frac{1}{n} \sum_{i = 1}^n f(X_i)$.  The classical M-estimator choses
$\param(P_n) = \argmin_\param P_n \loss_\theta$. When $\loss_\theta$ is
suitably smooth in $\theta$, then classical
asymptotics~\cite[Ch.~5.6]{VanDerVaart98}
give
\begin{equation*}
  \sqrt{n}\left(\param(P_n) - \param(P)\right)
  = -\sqrt{n} (P \ddot{\loss}_{\param(P)})^{-1}
  P_n \dot{\loss}_{\theta(P)}
  + o_P(1)
\end{equation*}
where $\dot{\loss}$ and $\ddot{\loss}$ denote gradient and Hessian,
respectively. For $H = \nabla^2 \poploss_P(\param(P))$,
we then have $\sqrt{n}(\param(P_n) - \param(P))
\cd \normal(0, H^{-1} \cov(\dot{\loss}_{\param(P)}) H^{-1})$,
the optimal asymptotic convergence~\cite{DuchiRu21}.

\providecommand{\radius}{\textup{rad}_2}

Because we develop results under privacy constraints,
we require somewhat explicit smoothness assumptions.
\begin{assumption}
  \label{assumption:loss-stuff}
  For each $x \in \mc{X}$, the loss $\theta \mapsto \loss_\theta(x)$ is
  convex, $\mc{C}^2$, $\lipobj$-Lipschitz continuous, and has
  $\lipconst_i$-Lipschitz $i$th derivatives, $i = 1, 2$. Additionally, $\nabla^2
  \loss_\theta(x)$ is at most rank $r$.
\end{assumption}
\noindent
Any generalized linear model loss has rank-1 Hessian, while
smoothness properties follow on a per-loss basis.
Logistic regression (which has inputs $x \in \R^d$ and
$y \in \{0, 1\}$) has loss
\begin{equation*}
  \loss_\theta(x, y) = \log(1 + \exp(\<x, \theta\>)) - y \<x, \theta\>,
\end{equation*}
and when $x \in \mc{X} \subset \R^d$ has finite
radius $\radius(\mc{X}) = \sup_{x \in \mc{X}} \ltwo{x}$,
the loss has Lipschitz constants
$\lipobj = \radius(\mc{X})$, $\lipgrad = \frac{1}{4} \radius(\mc{X})^2$,
and $\liphess \lesssim \radius(\mc{X})^3$.

As the prototypical private estimator, we consider objective
perturbation~\cite{ChaudhuriMoSa11}, which chooses $\wt{\param}$ to minimize
a randomly tilted (and regularized) version of the empirical objective $P_n
\loss_\param$.  To make dependence on the noise clear and allow us to
provide a unified analysis, for a regularization $\lambdareg \ge 0$ and
tilting vector $\noiseval$, define
\begin{equation}
  \label{eqn:obj-pert-param}
  \param_{\lambdareg}(P_n, \noiseval) \defeq \argmin_\param
  \left\{P_n \loss_\param + \frac{\lambdareg}{2} \ltwo{\param}^2
  + \<\noiseval, \param\> \right\}.
\end{equation}
We let $\param_\lambda(P_n) = \param_\lambda(P_n, \zeros)$ for shorthand.
Then under appropriate conditions on the loss $\loss$, regularization
$\lambdareg$, and noise $\noisevar$, $\param_{\lambdareg}(P_n, \noisevar_n)$
is $(\diffp, \delta)$-differentially private. A recent analysis gives the
following.
\begin{corollary}[Lemmas 7 and 24, \cite{AgarwalKaSiTh23}]
  \label{corollary:objective-perturbation-privacy}
  Let Assumption~\ref{assumption:loss-stuff} hold
  and $\diffp > 0$, and $\delta \in (0, 1)$.
  Then there is a 
  numerical constant $C$ such that if
  \begin{equation*}
    \sigma_n^2 \ge C \frac{\lipobj^2(d + \log \frac{1}{\delta})}{n^2 \diffp^2},
    ~~ \noisevar_n \sim \normal(0, \sigma_n^2 I_d),
    ~~ \mbox{and} ~~
    \lambdareg[n] \ge C \cdot \min\{r, d\} \frac{\lipgrad}{n \diffp}
  \end{equation*}
  the objective
  perturbation estimate $\param_{\lambdareg[n]}(P_n, \noisevar_n)$ is $(\diffp,
  \delta)$-differentially private.
\end{corollary}

In this case, we perform bootstrap resampling on
the statistic $\param_\lambda(P_{\subsize,n}^*, \noisevar_n)
- \param_\lambda(P_\subsize)$.
The following proposition captures the asymptotic properties of the
subsampling/resampling estimator~\eqref{eqn:obj-pert-param}, which also
includes classical M-estimators with $\noiseval = \zeros$ and $\lambdareg =
0$ as a special case.
\begin{proposition}
  \label{proposition:objective-perturbation-subsample}
  Assume the subsample size $\subsize  \to \infty$
  as $n \to \infty$. Then for any sequence of vectors $\noiseval_n
  \to 0$ and any bounded
  sequence of regularization values
  $\lambda = \lambdareg[n] \ge 0$,
  \begin{align*}
    \theta_\lambda(P_n, \noiseval_n)
    - \theta_\lambda(P)
    & = -(\nabla^2 \poploss_P(\theta_\lambda(P))
    + \lambda I)^{-1}
    (P_n \dot{\loss}_{\theta_\lambda(P)} + \lambda
    \param_\lambda(P) + \noiseval_n)
    + R_n,
  \end{align*}
  where for a problem-dependent constant $C < \infty$
  and numerical constant $c > 0$, for all $t \le c n$ the
  remainder $R_n$ satisfies
  \begin{equation*}
    \P\left(\norm{R_n}
    \ge C
    \sqrt{\frac{t}{n}}
    \bigg(\ltwo{\noiseval_n} + \sqrt{\frac{t}{n}}
    \bigg)\right)
    \le e^{-t}.
  \end{equation*}
  Additionally, we have the (conditional on $P_n$) almost sure
  expansions
  \begin{align*}
    \param_\lambda(P_{n,m}^*, \noiseval_n)
    - \param_\lambda(P_m\subsample)
    & = -(H_m + \lambdareg I)^{-1} (P_{n,m}^* \dot{\loss}_{\param_\lambda(P_m\subsample)}
    + \lambdareg \param_\lambda(P_\subsize)
    + \noiseval_n) + R_{n,\subsize}^* ~~ \mbox{and} \\
    \param_\lambda(P_{n,m}^*, \noiseval_n)
    - \param_\lambda(P_{n,m}^*)
    & = -(H_m + \lambdareg I)^{-1} \noiseval_n + R_n^*,
  \end{align*}
  where $H_\subsize = P_\subsize \ddot{\loss}_{\param_\lambda(P_\subsize)} \to
  \nabla^2 \poploss_P(\param_\lambda(P))$ conditionally almost surely.  The
  remainders $R_n^*$ and $R_{n,\subsize}^*$ satisfy the same probabilistic
  guarantees as $R_n$ conditional on $P_\subsize$ asymptotically almost
  surely.
\end{proposition}
\noindent
As Proposition~\ref{proposition:objective-perturbation-subsample}
is not the main focus, we defer its proof to
Appendix~\ref{sec:proof-objective-perturbation-subsample}.

\subsubsection{Subsampling consistency for objective perturbation}

We leverage Proposition~\ref{proposition:objective-perturbation-subsample}
to show how objective perturbation meets the conditions
of Propositions~\ref{proposition:asymp-norm-consistency}
and~\ref{proposition:edgeworth}.
Throughout this argument, we let the privacy levels $\diffp = \diffp_n$
decrease, but assume they satisfy $\diffp_n \ge n^{-\beta}$ for
some $\beta < \half$.
We first argue that so long as the regularization
$\lambdareg[n] \to 0$ quickly enough, then the regularized
(population) minimizers are equivalent to $\param(P)$.
A quick calculation
with the implicit function theorem shows that
if $\theta_\lambda = \argmin_\theta \{\poploss_P(\theta) + \lambda
\ltwo{\theta}^2\}$, then
\begin{equation}
  \theta_\lambda = \theta(P) - \lambda \nabla^2 \poploss_P(\theta(P))^{-1}
  \theta(P) + O(\lambda^2)
  \label{eqn:implicit-theta-lambda}
\end{equation}
as $\lambda \to 0$.
Define the asymptotic covariance
\begin{equation*}
  \Sigma_P \defeq \nabla^2 \poploss_P(\param(P))^{-1}
  \cov_P(\dot{\loss}_{\param(P)})
  \nabla^2 \poploss_P(\param(P))^{-1}.
\end{equation*}
Then the choice $\lambda = \lambdareg[n] = O(1/n\diffp)$ in
Corollary~\ref{corollary:objective-perturbation-privacy} guarantees
the asymptotic normality
$\sqrt{n}\left(\param_{\lambdareg[n]}(P_n, \noisevar_n)
- \param(P)\right)
\cd \normal(0, \Sigma_P)$.
Additionally, so long as $\lambdareg[n] \ll 1/\sqrt{n}$,
the resampled estimator
\begin{equation*}
  \blbstatistic
  \defeq \sqrt{n}\left(\param_\lambda(P_{\subsize,n}^*) - \param_\lambda(P_\subsize)\right)
  = -\sqrt{n} (H_\subsize + \lambda)^{-1} P_{n,m}^* \dot{\loss}_{\param_\lambda(P_m)}
  + O_P(\lambda \sqrt{n})
  + \remainder_{\subsize,n}^*
\end{equation*}
similarly satisfies $\blbstatistic \cd \normal(0, \Sigma_P)$ conditionally
almost surely by Slutsky's lemmas and the Lindeberg central limit
theorem~\cite[cf.][Ch.~23]{VanDerVaart98}.  That is, $\blbstatistic$ is
KS-consistent (Def.~\ref{def:ks-consistency}) for $\unstudstatest =
\sqrt{n}(\param_\lambda(P) - \param_\lambda(P_n))$, as
Proposition~\ref{proposition:asymp-norm-consistency} requires. That
$\param_\lambda - \param_0 = O(\lambda) = O(\frac{1}{n \diffp})$ by the
expansion~\eqref{eqn:implicit-theta-lambda} shows that so long as $\diffp
\gg 1/\sqrt{n}$, the regularization has negligible asymptotic effect.  We
more explicitly give the resampling
consistency~\eqref{eqn:resample-consistency} as an example:

\begin{example}[Resampling consistency of objective perturbation]
  \label{example:obj-pert-resampling}
  Let Assumption~\ref{assumption:loss-stuff} hold on the losses, and assume
  the subsample size $\subsize$ is at least polynomial in $n$. Assume as
  above that the privacy $\diffp = \diffp_n \ge
  n^{-\beta}$ for some $\beta < \half$. Then as in
  Corollary~\ref{corollary:objective-perturbation-privacy}, we have
  $\sigma_n^2 = O(n^{-2(1 - \beta)})$, and so $\sqrt{n} \ltwo{\noisevar_n}
  \cas 0$.  In particular, the remainder terms $R_n$ and $R_n^*$ in
  Proposition~\ref{proposition:objective-perturbation-subsample} satisfy
  $\sqrt{n} \max\{R_n, R_n^*\} \cas 0$, and so the objective perturbation
  estimator is rate-$\sqrt{n}$ subsampling
  consistent~\eqref{eqn:resample-consistency}.

  Recalling Corollary~\ref{corollary:basic-percentile-consistency}, we see
  that if $\what{t}$ is the index $\blbquant$ returns, then with the choices
  $\lambdareg[n]$ and $\sigma_n^2$ in
  Corollary~\ref{corollary:objective-perturbation-privacy},
  the confidence set
  \begin{equation*}
    C_\alpha(P_n) = \param_{\lambdareg[n]}(P_n, \noisevar_n)
    + \frac{1}{\sqrt{n}} I_{\what{t}}
  \end{equation*}
  is $(2 \diffp, \delta)$-differentially private, and
  satisfies $\P(\param(P) \in C_\alpha(P_n)) \to 1 - \alpha$.
\end{example}

\subsubsection{Consistency via Edgeworth expansions and objective
  perturbation}

The Edgeworth expansions that Proposition~\ref{proposition:edgeworth}
requires follow for M-estimators from the asymptotics in
Proposition~\ref{proposition:objective-perturbation-subsample}, so that
$\param(P_n)$ (taking $\noiseval_n = \zeros$) has an Edgeworth
expansion~\cite{Hall92}.
We can then provide the following high-probability guarantees
on deviation, where we are not explicit about constants.

\begin{example}[Accuracy of objective perturbation]
  Let $\tparam(P_n) = \param_{\lambdareg[n]}(P_n, \noisevar_n)$ for
  $\noisevar_n \sim \normal(0, \sigma_n^2)$ and regularization
  $\lambdareg[n]$ as in
  Corollary~\ref{corollary:objective-perturbation-privacy}, and let
  $\param(P_n) = \param_0(P_n, \zeros)$ be the standard M-estimator. We give
  a completely asymptotic expansion, taking $\lambdareg[n] = O(1 / n
  \diffp)$ and $\sigma_n^2 = O(1 / n^2 \diffp^2)$, ignoring the privacy
  parameter $\delta > 0$ and other problem dependent constants. Define
  $\theta = \param(P)$ and $\theta_\lambda = \argmin_\param
  \{\poploss_P(\param) + \frac{\lambda}{2} \ltwo{\theta}^2\}$, so that
  $\theta_\lambda = \theta + O(\lambda)$ as a consequence of the implicit
  function theorem, as above.  Similarly, let $H = \nabla^2
  \poploss_P(\param)$ and $H_\lambda = \nabla^2 \poploss_P(\param_\lambda) =
  H + O(\lambda)$, where we use the smoothness of $\ddot{\loss}_\theta$ from
  Assumption~\ref{assumption:loss-stuff}.
  
  We claim that for some $C < \infty$ depending only on the problem
  parameters in Assumption~\ref{assumption:loss-stuff},
  \begin{equation*}
    \P\left(\norms{\tparam(P_n) - \param(P_n)}
    \ge
    C \left(\frac{1}{n \diffp}
    + \frac{1}{n}\right) (1 + t)
    \right) \le e^{-t}
  \end{equation*}
  for all $0 \le t \le n$.
  As a consequence, the private estimator $\tparam(P_n)$
  is $(C \cdot \frac{\log n}{n \diffp}, \frac{1}{n^2})$-accurate for
  $\param(P_n)$.  
  To see this claim, note that 
  by Proposition~\ref{proposition:objective-perturbation-subsample},
  \begin{equation*}
    \tparam(P_n) - \param(P_n)
    = -(H_{\lambdareg[n]} + \lambdareg[n] I)^{-1}
    (P_n \dot{\loss}_{\theta_{\lambdareg[n]}}
    + \lambdareg[n] \theta_{\lambdareg[n]}
    + \noisevar_n)
    + H^{-1} P_n \dot{\loss}_\theta
    + \remainder_n,
  \end{equation*}
  where $\norm{R_n} \le C \sqrt{t / n} (\norm{\noisevar_n} + \sqrt{t/n})$
  with probability at least $1 - e^{-t}$.  Performing an asymptotic
  expansion that $(H_\lambda + \lambda I)^{-1} = (H + O(\lambda) + \lambda
  I)^{-1} = H^{-1} + O(\lambda)$ as $\lambda \to 0$, we therefore obtain
  \begin{align*}
    \tparam(P_n) - \param(P_n)
    & = H^{-1} (P_n \dot{\loss}_\theta
    - P_n \dot{\loss}_{\theta_{\lambdareg[n]}} - \noisevar_n
    - \lambdareg[n] \theta_{\lambdareg[n]})
    + O(\lambdareg[n]) (P_n \dot{\loss}_{\theta_{\lambdareg[n]}}
    + \noisevar_n) + \remainder_n \\
    & = -H^{-1} \noisevar_n
    + O(\lambdareg[n])(1 + \noisevar_n) + \remainder_n,
  \end{align*}
  where we use that $\ltwos{\dot{\loss}_{\param_{\lambdareg[n]}}
    - \dot{\loss}_\param}
  \le \lipgrad \ltwos{\param_{\lambdareg[n]}
    - \param}
  = O(\lambdareg[n])$.
  In particular, because
  $\noisevar_n \sim \normal(0, \sigma_n^2)$,
  there is some $C$ such that
  $\P(\ltwos{\noisevar_n} \ge C \frac{1 + t}{n \diffp})
  \le e^{-t^2}$ (cf.~\cite[Lemma 1]{LaurentMa00})
  for $t \ge 0$. Applying the triangle inequality
  a few times then gives the original claimed accuracy guarantee.
\end{example}

The Edgeworth expansions~\eqref{eqn:edgeworth-blb}
also require accuracy of a variance proxy $\sigma(P_\subsize)$. For this,
we have Proposition~\ref{proposition:objective-perturbation-subsample}
again: let
\begin{equation*}
  \Sigma_{P_\subsize}
  = (P_\subsize \ddot{\loss}_{\param_{\lambdareg[n]}(P_\subsize)}
  + \lambdareg[n] I)^{-1}
  \cov_{P_\subsize}(\dot{\loss}_{\param_{\lambdareg[n]}(P_\subsize)})
  (P_\subsize \ddot{\loss}_{\param_{\lambdareg[n]}(P_\subsize)}
  + \lambdareg[n] I)^{-1}.
\end{equation*}
Then $\blbstatistic = \sqrt{n} (\param(P_\subsize) -
\param(P_{\subsize,n}^*))$ satisfies $\Sigma_{P_\subsize}^{-1/2}
\blbstatistic \to \normal(0, I)$ by an argument analogous to that preceding
Example~\ref{example:obj-pert-resampling}, and we certainly have
$\sqrt{\subsize}(\param(P_\subsize) - \param(P)) \cd \normal(0, \Sigma_P)$,
so that $\Sigma_{P_\subsize} = \Sigma_P + O_P(1/\sqrt{\subsize})$.  In
particular, Corollary~\ref{corollary:typical-edgeworth-consistency} applies,
because the objective perturbation estimator $\param_{\lambdareg[n]}$ has
typical private error~\eqref{eqn:typical-error}, and so the confidence set
$C_\alpha(P_n)$ centered around
$\param_{\lambdareg[n]}(P_n, \noisevar_n)$
of Corollary~\ref{corollary:coverage-from-percentile}
satisfies
\begin{equation*}
  \P(\param(P) \in C_\alpha(P_n)) = 1 - \alpha + O\left(\frac{\log n}{
    \diffp \sqrt{n}}\right).
\end{equation*}

\section{Experiments}\label{sec:expts}

\begin{figure*} %
  \begin{subfigure}{0.33\textwidth}
    \includegraphics[width=\linewidth]{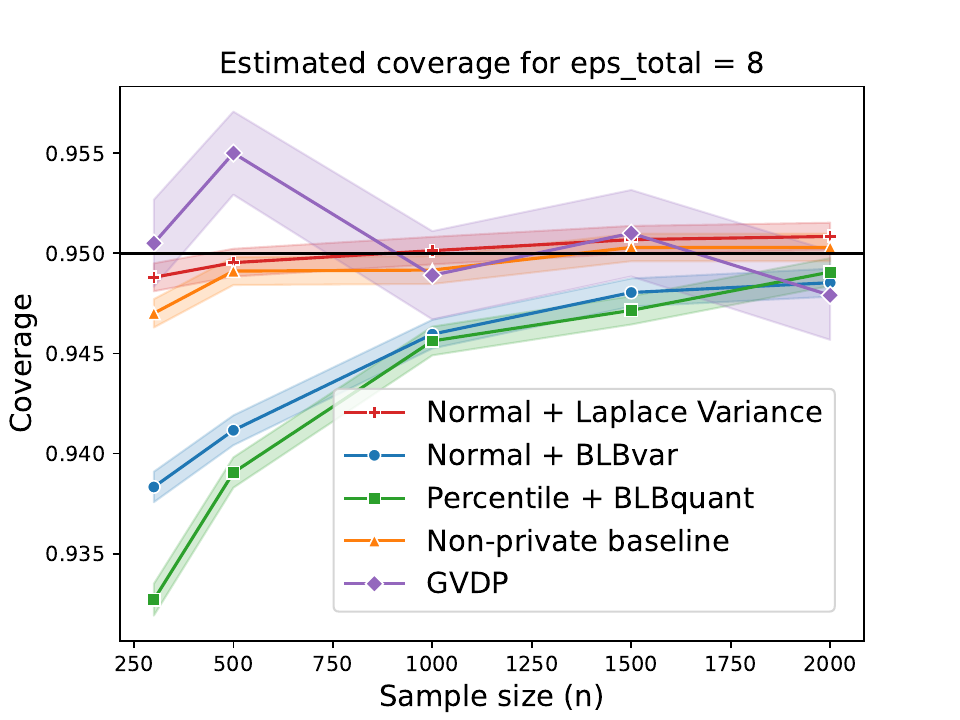}
    \caption{\small Mean estimation coverage} \label{fig:mean-coverage}
  \end{subfigure}
  \begin{subfigure}{0.33\textwidth}
    \includegraphics[width=\linewidth]{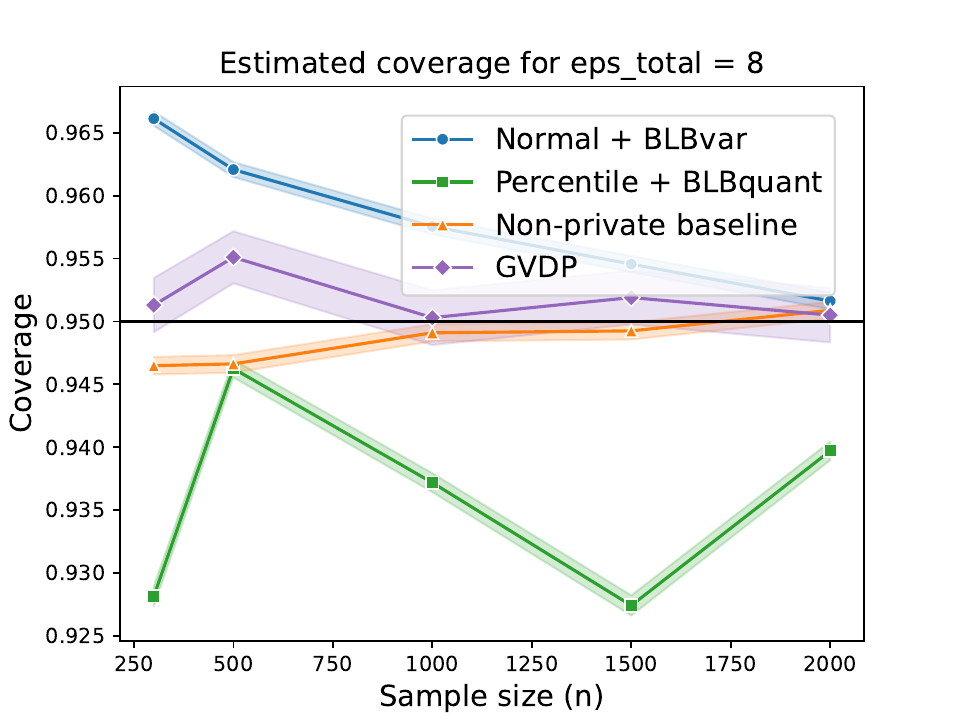}
    \caption{\small Median estimation coverage} \label{fig:median-coverage}
  \end{subfigure}\hspace*{\fill}
  \begin{subfigure}{0.33\textwidth}
    \includegraphics[width=\linewidth]{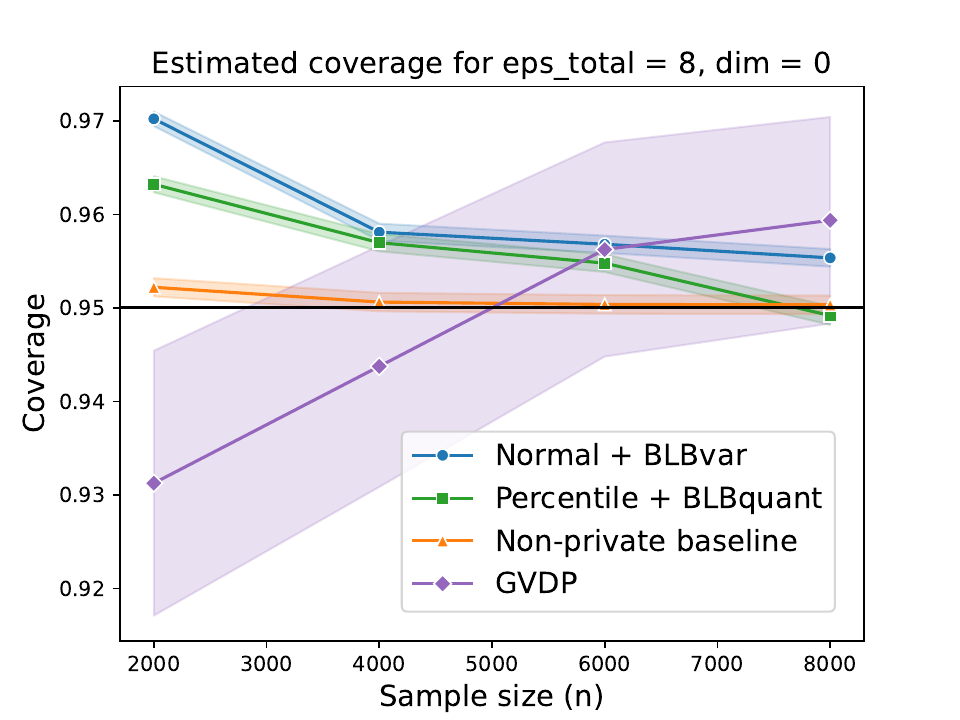}
    \caption{\small Logistic regression coverage} \label{fig:logreg-coverage}
  \end{subfigure}

  \medskip
  \begin{subfigure}{0.33\textwidth}
    \includegraphics[width=\linewidth]{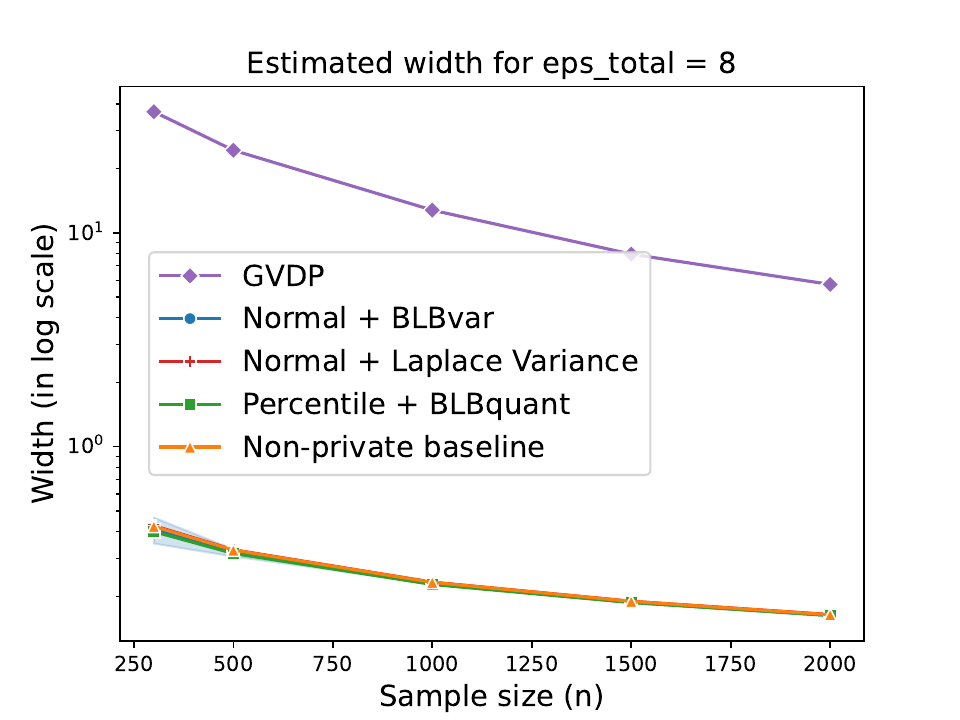}
    \caption{\small Mean estimation widths} \label{fig:mean-widths}
  \end{subfigure}
  \begin{subfigure}{0.33\textwidth}
    \includegraphics[width=\linewidth]{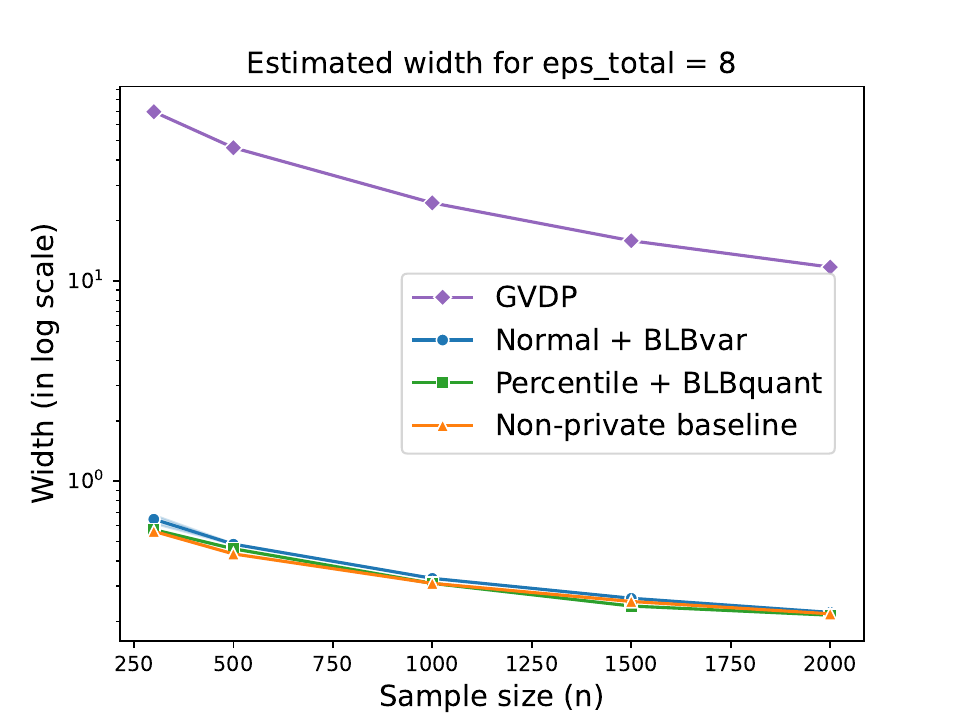}
    \caption{\small Median estimation width} \label{fig:median-width}
  \end{subfigure}\hspace*{\fill}
  \begin{subfigure}{0.33\textwidth}
    \includegraphics[width=\linewidth]{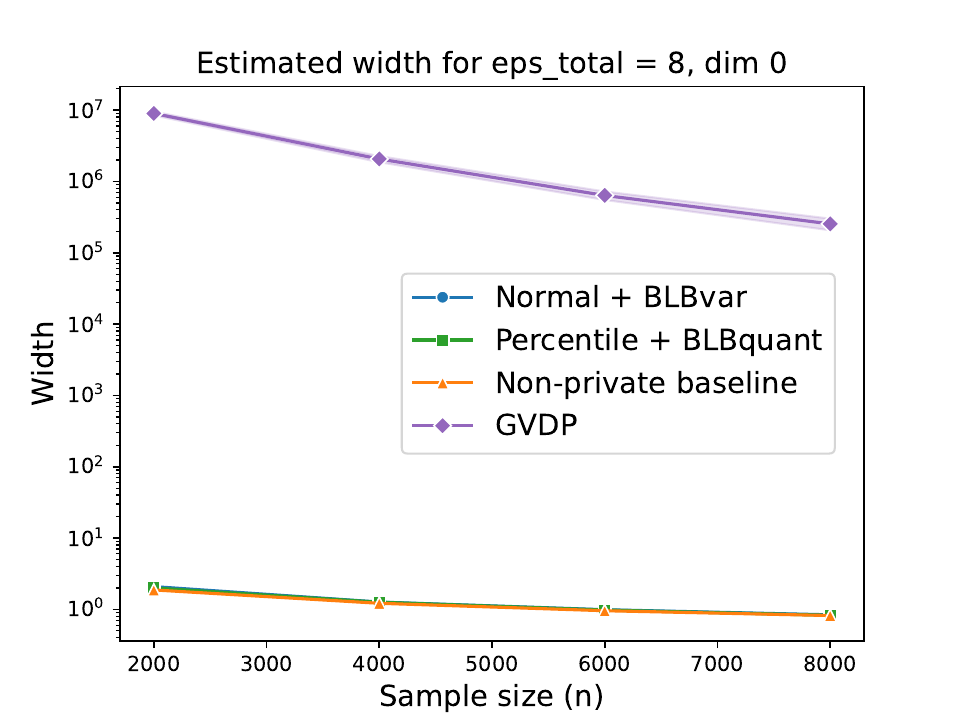}
    \caption{\small Logistic regression width} \label{fig:logreg-width}
  \end{subfigure}

  \looseness=-1
  \caption{Coverage and confidence interval
    widths for all experiments. Each uses $\totdiffp =  8$.}
  \label{fig:1}
\end{figure*}

We study the empirical performance of the proposed methods by constructing
confidence intervals and estimating their coverage and widths on the
following three tasks, which we choose to elucidate the performance
of private estimators for estimators with different characteristics.
As the confidence sets always center around a private estimator
$\tparam(P_n)$, we use a total privacy budget $\totdiffp$ that
we divide evenly between $\tparam(P_n)$ and the resampling methods.
Appendix~\ref{appen:expts} includes more detail and additional ablation
studies. 

\begin{itemize}[leftmargin=*]
  \setlength{\itemsep}{0mm}
\item \textbf{Mean estimation}: (\Cref{fig:mean-coverage,fig:mean-widths})
  We generate synthetic data from a Truncated Gaussian and use
  the sample mean with additive Laplace noise as the private estimator.
\item \textbf{Median estimation}:
  (\Cref{fig:median-coverage,fig:median-width}) We generate synthetic data
  from a Truncated Gaussian distribution, and use the sample median with the
  inverse sensitivity mechanism \citep{AsiDu20} as the private
  estimator. The median fails to satisfy the Edgeworth expansion
  assumption~\eqref{eqn:edgeworth-blb}, but it is resampling consistent.
\item \textbf{Logistic regression}:
  (\Cref{fig:logreg-coverage,fig:logreg-width}) We use the Adult income
  dataset \citep{DingHaMiSc21}, where the task is to predict whether the
  income of individuals is greater than a threshold given some of
  their features. To simulate sampling from a
  population distribution, we consider the ERM solution
  on the whole dataset to be the true parameter, and
  we draw samples of size $n$ from the whole dataset.
  We use the inverse sensitivity mechanism \citep[Sec
    5.2]{AsiDu20} as the private estimator, estimating
  the multiplier on the \texttt{Sex} variable.
\end{itemize}

\newcommand{\Ntrial}{N_{\textup{trials}}}
\newcommand{\Nresamp}{N_{\textup{resamp}}}

\textbf{Methodology}: We evaluate each method on each task by repeating the
following individual trial $\Ntrial$ times. Within each trial, we generate a
confidence interval for the private statistic
$\unstudstatpriv$~\eqref{eqn:private-unstudentized}, then record whether the
true parameter lies in it and its width.
To evaluate $\P(\param(P) \in C_\alpha(P_n))$, we employ a computational
shortcut. Recognizing that $C_\alpha(P_n) = \tparam(P_n) + I_{\what{t}}$,
and $\tparam$ and the index $\what{t}$ use independent randomness and
so are conditionally independent given the sample $P_n$,
we have
\begin{equation*}
  \P(\param(P) \in C_\alpha(P_n) \mid P_n)
  = \P\left(\param(P) \in \tparam(P_n, \noisevar_n^{(*)}) + n^{-1/2} I_{\what{t}}
  \mid P_n\right),
\end{equation*}
where $\noisevar_n^{(*)}$ is a resampled independent random variable that
$\tparam$ uses to guarantee privacy. Thus, for a given interval
$I_{\what{t}}$, we draw $\Nresamp$ independent copies of $\noisevar_n^{(*)}$
and compute the frequency with which $\param(P) \in \tparam(P_n,
\noisevar_n^{(*)}) + n^{-1/2}I_{\what{t}}$.
Figures~\ref{fig:mean-coverage}--\ref{fig:logreg-coverage} plot
this coverage, while
Figures~\ref{fig:mean-widths}--\ref{fig:logreg-width} plot average
confidence interval widths.

We plot the performance of the following algorithms:

\textsc{NonprivateBaseline}: We use the non-private percentile method to
approximate the distribution of the private statistic $\tparam(P_n) -
\param(P)$ via directly bootstrapping $\tparam(P_n^*) - \param(P_n)$.
Across all tasks, this has the lowest coverage error and its width serves as
a baseline for the performance of the private methods.

$\textsc{Percentile}+\blbquant$ and $\textsc{Normal}+\blbvar$: We use the
private percentile method and private normal approximation, with the private
quantile and variance estimators from \Cref{algorithm:blb-quant} and
\Cref{algorithm:blb-var} respectively, as
Corollaries~\ref{corollary:coverage-from-percentile}
and~\ref{corollary:coverage-from-variance} describe. Across the three tasks,
we observe that both the coverage and the average width for these are
is similar to the
$\textsc{NonprivateBaseline}$.
Typically, the normal approximation via $\blbvar$ provides higher
coverage and a larger confidence set than
$\blbquant$, though there is no clear winner.

\Cref{fig:mean-hist} focuses on mean estimation and plots a histogram of
confidence interval widths across trials.  Both $\blbquant$ and $\blbvar$
return confidence intervals with widths typically between $.3$ and $.9$, and
the spread of reported confidence intervals is wider than the
\textsc{NonprivateBaseline}, highlighting a cost to privacy.  $\blbvar$ has
an infrequent failure mode where it returns a grossly too large confidence
set; by Theorem~\ref{thm:blb-var} this failure occurs with probability at
most $(\sigmamax^2 / \smoothparam) \exp(-c \cdot \nsubs \diffp)$, where
$\nsubs$ is the number of subsamples, so slightly increasing this can reduce
this failure mode significantly. (See \Cref{appen:expts} for more details.)

\newcommand{\gvdp}{\textsc{Gvdp}\xspace}

\gvdp (General Valid Differential Privacy) \citep{CovingtonHeHoKa21}: \gvdp
is a fully black box algorithm that uses the
Sample-Aggregate~\cite{NissimRaSm07} method to privatize a (potentially)
non-private estimator, then uses \textsc{CoinPress}~\cite{BiswasDoKaUl20} to
aggregate BLB-subsampled results.  We use the implementation of \gvdp the
authors provide.  For mean and median estimation, the coverage of \gvdp is
as good as our proposed methods, but logistic regression
(Fig.~\ref{fig:logreg-coverage}) proves challenging. Likely because it is
black box and because mean (as opposed to median) aggregation is inherently
unstable, the confidence intervals it outputs are at least $10\times$ wider
than other procedures. While \gvdp's coverage is strong, the
more robust median-based aggregation of Algorithms~\ref{algorithm:blb-quant}
and~\ref{algorithm:blb-var}, which privately adapt
to estimator accuracy, allows sharper confidence intervals, especially
when paired with a known private estimator.

$\algname{Normal+LaplaceVariance}$: In the
special case of mean estimation with $\tparam(P_n, \noiseval)
= P_n X + \noiseval$ (see Example~\ref{example:resampling-mean})
for $x \in [-b, b]$,
the sample variance
\begin{equation*}
  \sigma_n^2 \defeq P_n X^2 - (P_n X)^2
  = \frac{1}{n^2} \sum_{i < j} (X_i - X_j)^2
\end{equation*}
is stable, with global sensitivity at most $\frac{4 b^2}{n}$, so that
$\wt{\sigma}_n^2 = \sigma_n^2 + \noisevar_n^{(2)}$
with $\noisevar_n^{(2)} \sim \laplace(0, \frac{8 b^2}{n \diffp})$
is $\diffp$-differentially private. Thus,
we may directly and privately estimate variance and give
a normal confidence interval.
The better performance of this ad-hoc method highlights that
if it is possible to directly estimate estimator variance---which
requires bespoke approaches---one can achieve stronger confidence
sets.

\begin{figure}
  \centering
  \includegraphics[width=0.9\linewidth]{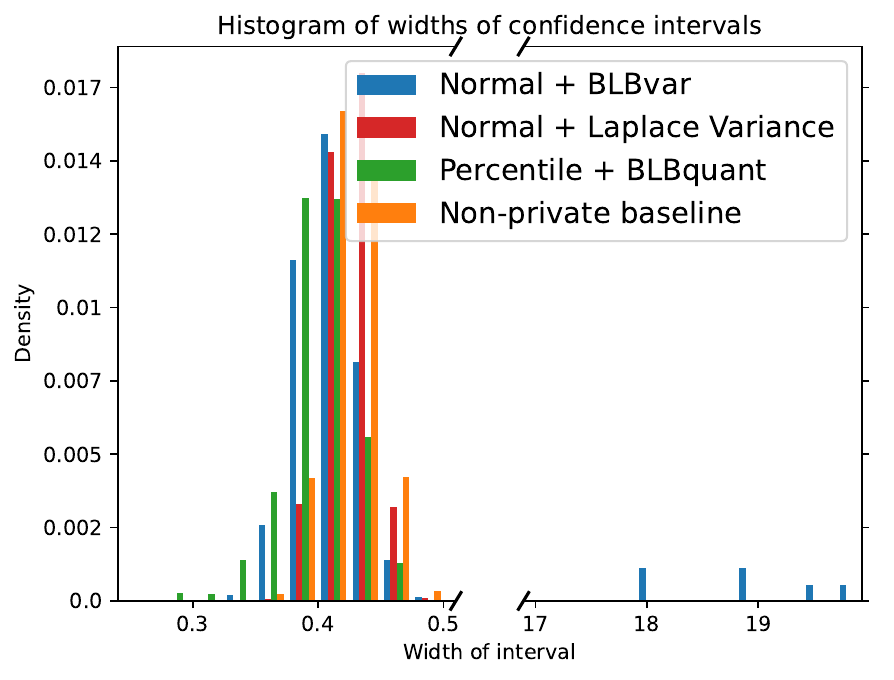}
  \caption{Histogram of widths of confidence intervals constructed for mean estimation for $n = 300$}
  \label{fig:mean-hist}
\end{figure}

\section{Discussion and conclusions}

We have introduced two general purpose data-driven methods for constructing
private confidence sets and intervals given a private estimator for a
parameter of interest. One challenge our experiments highlight is the
challenge of obtaining useful confidence intervals under privacy.  The more
accurate a \emph{private} estimator we have available, the tighter
confidence intervals we can obtain. For example, our experiments with mean
estimation and median achieve confidence set widths (at level $\alpha =
.05$) of roughly $.2$ at sample size $n = 1000$, but for logistic
regression, where effective private estimators are more challenging, the
confidence intervals are of width roughly $1$. Without a strong private
estimator (see Figure~\ref{fig:logreg-width}), confidence interval widths
remain unusably large. At some level, then, our results highlight the need
for practicable private algorithms for basic statistical estimation tasks,
especially generalized linear models, for us to be able to tackle real-world
problems.

It may be interesting to develop higher-order convergence guarantees for
bootstrap and other resampling algorithms for privacy. It is plausible that
by modifying Assumption~\ref{assumption:bootcdf-accuracy} to something along
the lines of $\P(|p - \what{p}| \ge \cdfacc \mid P_\subsize) \ge \half +
\eta$ for some $\eta > 0$, which should hold with high probability over the
subsample $P_\subsize$, we might achieve higher-order accuracy in Edgeworth
expansions, though this would require some more sophisticated conditioning
arguments on a ``good'' event that subsamples are accurate enough.  Our
consistency results also rely on private estimators achieving rates of
convergence comparable to non-private estimators; moving to analysis regimes
where privacy induces substantial additional error could provide new
insights. In scenarios where the private estimator simply adds noise
independent of the underlying statistic $\param(P_n)$, such as basic mean
estimation problems (Examples~\ref{example:resampling-mean}
and~\ref{example:resampling-mean-2}), we can fairly trivially analyze
coverage directly, as the estimator's error simply decomposes, but moving
beyond this will likely require new tools.
We look forward to more work bringing privacy and the
inferential goals of modern statistics together.

\subsection*{Acknowledgements}

This work was partially supported by the Office of Naval Research
Grant N00014-22-1-2669, the
Stanford DAWN Consortium and National Science Foundation Grant IIS-2006777.

\newpage
\bibliography{bib}
\bibliographystyle{abbrvnat}
\newpage
\appendix

\section{The private median algorithm}
\label{appendix:extra-algs}

We recall Asi and Duchi's private median algorithm~\cite{AsiDu20}, providing
pseudocode for completeness in Alg.~\ref{algorithm:priv-median}. The method
takes as input a vector $v_{1:n} \in \R^n$, a privacy budget $\diffp$, and a
smoothing constant $\smoothparam$ (which Asi and Duchi recommend taking to
be $\smoothparam \ll 1/n$), and upper and lower bounds $R_u$ and $R_l$ on
the range of the data, and it outputs an $\diffp$-DP estimate of the median
of $v_{1:n}$. The algorithm instantiates the inverse sensitivity
mechanism~\cite{AsiDu20}, and uses the $\smoothparam$-smoothed length
function $\len^\smoothparam$ (defined in
\Cref{line:len-def,line:smoothlen-def} in \Cref{algorithm:priv-median}) as
the score function to sample a private median estimate using the exponential
mechanism \citep{McSherryTa07}.

\begin{algorithm}
    \caption{\label{algorithm:priv-median} $\privmed$ (from \citet{AsiDu20})}
    \DontPrintSemicolon
    \SetKwInOut{Input}{Input}
    \SetKwInOut{Output}{Output}
    \Input{ data $v_{1:n}$, privacy budget $\diffp$,
      smoothing parameter $\smoothparam$, interval
      $[R_l, R_u]$}
    \Output{ $\diffp$-DP estimate of the median of $v_{1:n}$}
    Set $\what{v} \gets \textup{Med}(v_{1:n})$\;
    Define $\len(y; v_{1:n}) \defeq
    \card\{i \in [n] \mid v_i \in \openleft{y}{\what{v}}
      \cup \openright{\what{v}}{y}\}$\;\label{line:len-def}
      Define $\len^\smoothparam(y; v_{1:n}) \defeq
      \inf_{\abs{z - y} < \smoothparam}\len(z; v_{1:n})$\;\label{line:smoothlen-def}
    \For{$\ell \in 1,\dots,\lceil n/2 \rceil$}{
      Define $\mc{I}_\ell = \{y \in [R_l,R_u] : \len^\smoothparam(y; v_{1:n})
      = \ell \}$\;
    }
    Sample $\ell$
    with probability proportional to $\abs{I_\ell}\exp(-\ell\varepsilon/2)$\;
    \Return $v \sim \uniform\prn{\mc{I}_\ell}$\;
\end{algorithm}

\section{Proofs related to percentile methods}

\subsection{Proof of Proposition~\ref{proposition:privacy-threshold}}
\label{sec:proof-privacy-threshold}

We demonstrate that for any
$t \in [T] \cup \{\perp\}$,
\begin{align}
  \P(\abthreshmed(y, \threshold, \xi) = t)
  \leq e^\diffp \P(\abthreshmed(y', \threshold, \xi) =
  t). \label{eqn:claim-helper}
\end{align}
We first show the result for $t \in [T]$, treating $t = \perp$
afterwards as a special case.
We provide a particular coupling between noise values
$\xi$ and $\xi' = \pi(\xi)$, where
$\xi' \eqdiffp \xi$
and $\abthreshmed(y, \threshold, \xi) = \abthreshmed(y', \threshold, \xi')$,
which then
implies the result~\eqref{eqn:claim-helper}, as
\begin{equation*}
  \P(\abthreshmed(y, \threshold, \xi) = t)
  = \P(\abthreshmed(y', \threshold, \xi') = t)
  \le e^\diffp \P(\abthreshmed(y', \threshold, \xi) = t).
\end{equation*}
The coupling follows here:

\begin{observation}\label{obs:helper}
  Define $\eta \in \R^{T + 1}$ by
  \begin{align*}
    \eta_i :=
    \begin{cases}
      -1 & \mbox{if}~ i = 0    \\
      2 & \mbox{if}~ i = t    \\
      0 & \mbox{otherwise},
    \end{cases}
  \end{align*}
  and define $\pi(\xi) \defeq \xi + \eta$.
  Let $y(t), y'(t) \in \R^k$ satisfy
  $\dham(y(t), y'(t)) \le 1$ for each $t$.
  Then $\abthreshmed(x, \threshold, \xi) = t$ implies 
  $\abthreshmed(x', \threshold, \pi(\xi)) = t$ for all $t \in [T]$.
\end{observation}
\begin{proof}
  Observe that
  $\abthreshmed(y, \threshold, \xi) = t$
  if and only if the following conditions hold:
  \begin{enumerate}[leftmargin=*,label=(C\arabic*)]
  \item \label{item:reg-cond-1}
    For each index $u < t$,
    $y(u)_{(\floors{\xi_0 + \xi_u})} < \threshold$ or $\floors{\xi_0
    + \xi_u} < 1$, and
  \item \label{item:reg-cond-2} $y(t)_{(\floors{\xi_0 + \xi_t})} \geq \threshold$ or
    $\floors{\xi_0 + \xi_t} > k$.
  \end{enumerate}
  Similarly $\abthreshmed(y', \threshold, \xi') = t$ if and only if 
  \begin{enumerate}[leftmargin=*,label=(C\arabic*')]
  \item \label{item:prime-cond-1} For each index $u < t$,
    $y'(u)_{(\floors{\xi_0' + \xi_u'})} < \threshold$
    or $\floors{\xi_0' + \xi_u'} < 1$, and
  \item \label{item:prime-cond-2} $y'(t)_{(\floors{\xi_0' + \xi_t'})} \geq
    \threshold$ or $\floors{\xi_0' + \xi_t'} > k$.
  \end{enumerate}
  We show that condition \ref{item:reg-cond-1} implies
  \ref{item:prime-cond-1} and \ref{item:reg-cond-2} implies
  \ref{item:prime-cond-2}, from which the claim follows.  For the first
  condition, we have by construction for all $u < t$ that $\lfloor \xi_0' +
  \xi_u' \rfloor = \lfloor \xi_0 + \xi_u \rfloor - 1$.
  So if $\floors{\xi_0 + \xi_u} < 1$ then
  certainly $\floors{\xi'_0 + \xi'_u} < 1$. Otherwise,
  if $y(u)_{(\floors{\xi_0 + \xi_u})} < \threshold$, then
  because $\dham(y(u), y'(u)) \le 1$ we have
  \begin{equation*}
    y'(u)_{(\floors{\xi_0' + \xi_u'})}
    = y'(u)_{(\floors{\xi_0 + \xi_u} - 1)}
    \le y(u)_{(\floors{\xi_0 + \xi_u})}
    < \threshold.
  \end{equation*}

  Similarly, for the second condition~\ref{item:reg-cond-2}, we have by
  construction that $\lfloor \xi_0' + \xi_t' \rfloor = \lfloor \xi_0 + \xi_t
  \rfloor + 1$ and so, given
  condition~\ref{item:reg-cond-2},
  either $\lfloor \xi_0' + \xi_t' \rfloor > k$ or
  the assumption that $\dham(y(t), y'(t)) \leq
  1$ implies
  $y(t)_{(\lfloor \xi_0'
    + \xi_t'\rfloor)} \ge y(t)_{(\lfloor \xi_0 + \xi_t\rfloor)} > \threshold$.
\end{proof}

Observation~\ref{obs:helper} explicitly constructs $\eta \in \R^{T + 1}$
$\lone{\eta} \le 3$ such that $\abthreshmed(y, \threshold, \xi) = t$ implies
$\abthreshmed(y', \threshold, \xi + \eta) = t$.  Letting $\pi(\xi) = \xi +
\eta$, computing ratios of Laplace densities immediately implies that $\xi
\eqdiffp \xi'$ for $\xi$ constructed as in the statement of
Proposition~\ref{proposition:privacy-threshold}, and thus the
claim~\eqref{eqn:claim-helper} holds for $t \neq \perp$.

The special case $t = \perp$ follows from a similar argument:
defining $\eta \in \R^{T + 1}$ by
\begin{align*}
  \eta_i \defeq 
  \begin{cases}
    -1 & \mbox{if} ~ i = 0    \\
    0 & \mbox{otherwise},
  \end{cases}
\end{align*}
then
$\abthreshmed(y, \threshold, \xi)
= \perp$ implies $\abthreshmed(y', \threshold, \xi + \eta) = \perp$.
Letting $\pi(\xi) = \xi + \eta$, we see again
that
$\xi \eqdiffp \xi'$.

\subsection{Proof of Lemma~\ref{lemma:noisy-median}}
\label{sec:proof-noisy-median}

The following two technical lemmas form the keys for the proof.
\begin{lemma}
  \label{lemma:med-concentration}
  Let $\{Y_i\}_{i=1}^k$ be independent
  and satisfy
  $\P(|Y_i - \mu| \le c) \geq \half + \eta$ for all $i$, where $\eta >
  0$ and $c \in \R$. Then for $\alpha \ge 0$,
  \begin{equation*}
    \P\left(|Y_{(i)} - \mu| \leq c
    ~ \mbox{for~all}~ \frac{1 - 2 \alpha}{2} k \le i
    \le \frac{1 + 2 \alpha}{2} k \right)
    \geq 1 - \exp\left(-2 k \hinge{\eta - \alpha}^2\right).
  \end{equation*}
\end{lemma}
\begin{proof}
  Let $\eta_i = \P(|Y_i - \mu| \le c) - \half$, so that $\eta_i \ge \eta > 0$,
  let $Z_i = \indic{|Y_i - \mu| > c}$ so that $Z_i \sim \bernoulli(\half -
  \eta_i)$ independently, and let $\mc{I} = \{i \in \N \mid \frac{1 -
    2\alpha}{2} k \le i \le \frac{1 + 2 \alpha}{2} k\}$. Then
  \begin{align*}
    \P\left(|Y_{(i)} - \mu| > c ~ \mbox{for~some~} i \in \mc{I}\right)
    & \le \P\left(\sum_{i = 1}^k \indic{|Y_i - \mu| > c}
    \ge \frac{1 - 2 \alpha}{2} k
    \right) \\
    & = \P\left(\sum_{i = 1}^k Z_i \ge \frac{1 - 2 \alpha}{2} k \right) \\
    & = \P\left(\frac{1}{k} \sum_{i = 1}^k \left(Z_i - \left(\half - \eta_i
    \right)\right)
    \ge \frac{1}{k}
    \sum_{i = 1}^k \eta_i - \alpha\right).
  \end{align*}
  Because $\frac{1}{k} \sum_{i = 1}^k \eta_i \ge \eta$,
  applying Hoeffding's concentration inequality gives the result.
\end{proof}

\begin{lemma}
  \label{lemma:sum-laplace}
  Let $X_1 \sim \laplace(0,b_1)$ and $X_2 \sim \laplace(0,b_2)$. Then
  $X_1 + X_2$ has density
  \begin{align*}
    f_{X_1 + X_2}(x) = \frac{b_1^2}{b_1^2 - b_2^2} \frac{1}{2b_1}\exp
    \prn{-\frac{\abs{x}}{b_1}} + \frac{b_2^2}{b_2^2 - b_1^2} \frac{1}{2b_2}\exp\prn{-\frac{\abs{x}}{b_2}},
    \end{align*}
    and for $t \ge 0$,
    \begin{align*}
        \P\prn{\abs{X_1 + X_2} > t}  = \frac{b_1^2}{b_1^2 - b_2^2}\exp\prn{-\frac{t}{b_1}} + \frac{b_2^2}{b_2^2 - b_1^2}\exp\prn{-\frac{t}{b_2}}.
    \end{align*}
\end{lemma}
\begin{proof}
  Let $f_{X_i}$ denote the density of $X_i$, so $f_{X_1}(x) =
  \frac{1}{2b_i}\exp(\frac{-\absinline{x}}{b_i})$.
  Computing the convolution of $f_{X_1}$ and $f_{X_2}$ for $x \ge 0$ thus
  yields
  \begin{align*}
    f_{X_1 + X_2}(x) &= \int_{-\infty}^{\infty} f_{X_1}(y)f_{X_2}(x - y)dy\\
    &= \frac{1}{4b_1b_2}\int_{-\infty}^{\infty}\exp\prn{\frac{-\absinline{y}}{b_1}}\exp\prn{\frac{-\absinline{x-y}}{b_2}}dy\\
    &= \frac{1}{4b_1b_2}\prn{\int_{-\infty}^{0}\exp\prn{\frac{y}{b_1} + \frac{y-x}{b_2}}dy + \int_{0}^{x}\exp\prn{\frac{-y}{b_1} + \frac{y-x}{b_2}}dy}\\
    & + \frac{1}{4b_1b_2}\prn{\int_{x}^{\infty}\exp\prn{\frac{-y}{b_1} + \frac{x-y}{b_2}}dy}\\
    &= \frac{b_1^2}{b_1^2 - b_2^2} \frac{1}{2b_1}\exp\prn{-\frac{x}{b_1}} + \frac{b_2^2}{b_2^2 - b_1^2} \frac{1}{2b_2}\exp\prn{-\frac{x}{b_2}}.
  \end{align*}
  Symmetry considerations give the lemma's first identity. The
  second follows by integrating.
\end{proof}

We use Lemma~\ref{lemma:med-concentration} coupled with the
concentration guarantees for adding independent Laplace random
variables in Lemma~\ref{lemma:sum-laplace}. For
a value $\alpha \ge 0$ to be determined,
define the event
\begin{equation*}
  E_1 \defeq \left\{|\xi_0 + \xi - \E[\xi]| \le
  \alpha k \right\}.
\end{equation*}
Then because $\xi_0 - \E[\xi_0] \sim \laplace(0, b)$ and
$\xi \sim \laplace(0, 2b)$, Lemma~\ref{lemma:sum-laplace}
yields
\begin{equation*}
  \P(E_1^c) \le \frac{4}{3} \exp\left(-\frac{k \alpha}{2b}
  \right) - \frac{1}{3} \exp\left(-\frac{k \alpha}{b} \right)
  < \frac{4}{3} \exp\left(-\frac{k \alpha}{2 b}\right).
\end{equation*}
Because $Y \in \R^k$ has independent
entries and is independent of $E_1$,
\begin{align*}
  \P(|\ordstat(Y, \xi_0 + \xi) - \mu| > \medacc)
  & \le \P(E_1^c) + \P(|\ordstat(Y, \xi_0 + \xi) - \mu| > \medacc
  \mid E_1) \\
  & \le \P(E_1^c) +
  \exp\left(-2 k \hinge{1/4 - \alpha}^2\right)
\end{align*}
by Lemma~\ref{lemma:med-concentration}.
Bound $\P(E_1^c) \le \exp(-\frac{k \alpha}{2b})$
and set $\alpha = \frac{1}{8}$.

\subsection{Proof of Theorem~\ref{thm:blb-var}}
\label{sec:proof-blb-var}

The privacy of $\blbvar$ is immediate:
the algorithm $\privmed$ is $\diffp$-differentially private,
and because at most one of its inputs may change when
we change a single example (because at most a single
subsample $P\subsampleind{i}_\subsize$ changes),
$\blbvar$ is $\diffp$-differentially private.

To prove the utility of $\blbvar$, we analyze the performance
of the private median algorithm. First recall the equalities
\begin{equation*}
  \len(v, v_{1:\nsubs})
  = \card\left\{i \mid
  \mbox{Med}(v_{1:\nsubs}) \le v_i < v
  ~ \mbox{or} ~
  v < v_i \le \mbox{Med}(v_{1:\nsubs})\right\}
\end{equation*}
and
\begin{equation*}
  \len^\smoothparam(v; v_{1:\nsubs})
  \defeq \inf_{|z - v| \le \smoothparam}
  \len(v, v_{1:\nsubs}).
\end{equation*}
Let
$v$ be the output of $\privmed$ on
inputs and for $k \in\N$ to be chosen,
define the event
$\mc{E} \defeq \{\len^\smoothparam(v; v_{1:\nsubs}) \le k\}$.
Then
following \citet[Sec.~5.1]{AsiDu20}, we have
\begin{align}
  \nonumber
  \P(\mc{E}^c) = \P(\len^\smoothparam(v, v_{1:\nsubs}) > k)
  & = \frac{\int_{\len^\smoothparam(z;y_{1:s}) > k}
    \exp\prninline{-\frac{\diffp \len^\smoothparam(z; v_{1:\nsubs})}{2}} dz}{
    \int \exp\prninline{\frac{-\diffp \len^\smoothparam(z; v_{1:\nsubs})}{2}} dz} \\
  & \le \frac{\sum_{j = k + 1}^\infty \int_{\len^\smoothparam(z; v_{1:\nsubs})
      = j} e^{-\frac{\diffp}{2} j} dz}{
    \rho}
  \le \frac{\sigmamax^2}{\rho} \exp\left(-\frac{\diffp}{2} (k + 1)\right),
  \label{eqn:bad-event-upper-bound}
\end{align}
where the first inequality follows because
the denominator
lower bound $\smoothparam$ as $\len^\smoothparam(z; v_{1:\nsubs}) = 0$
for $|z - \mbox{Med}(v_{1:\nsubs})| \le \smoothparam$, while the second
follows because the support of $z$ has width $\sigmamax^2$.

Now, take $k \le \frac{\nsubs}{8}$. On the event $\mc{E}$, we have
$v_{(\nsubs/2 - k)} - \smoothparam \le
v \le v_{(\nsubs/2 + k)} + \smoothparam$, and
so applying Lemma~\ref{lemma:med-concentration}
(along with Assumption~\ref{assumption:bootstd-accuracy}
that $\P(|\sigma_n^2 - \bootvaralg| \ge \varacc) \ge \frac{3}{4}$)
with $c = \varacc$, $\alpha = \frac{k}{\nsubs} \le \frac{1}{8}$,
and $\eta = \frac{3}{4}$ and
noting that the values
$v_i$ that $\bootvaralg$ returns are independent
under the subsampling, we obtain
\begin{align*}
  \lefteqn{\P\left(\left| \sigma^2_n - \privmed(v_{1:\nsubs}, \diffp, \smoothparam,
    \sigmamax^2)\right|
    > \varacc + \smoothparam, \mc{E} \right)} \\
  & \le \P\left(\left|\sigma^2_n
  - v_{(j)}\right| > \varacc + \smoothparam
  ~ \mbox{for~each~} \frac{\nsubs}{2} - k \le j
  \le \frac{\nsubs}{2} + k, \mc{E}\right) \\
  & \le \P\left(\left|\sigma^2_n
  - v_{(j)}\right| > \varacc + \smoothparam
  ~ \mbox{for~each~} \frac{\nsubs}{2} - k \le j
  \le \frac{\nsubs}{2} + k\right)
  \le \exp\left(-\frac{\nsubs}{32}\right).
\end{align*}
Using the shorthand $\wt{\sigma}_n^2 = \blbvar(\param(\cdot),
\tparam(\cdot), P_n, \subsize, \sigmamax^2, \smoothparam, \diffp)$ and combining
this inequality with inequality~\eqref{eqn:bad-event-upper-bound}, we
therefore obtain
\begin{align*}
  \P\left(\left|\sigma_n^2 - \wt{\sigma}_n^2\right| \ge
  \varacc + \rho\right) & \le
  \P\left(\left|\sigma_n^2 - \wt{\sigma}_n^2\right| \ge
  \varacc + \rho, \mc{E}\right)
  + \P(\mc{E}^c) \\
  & \le \exp\left(-\frac{\nsubs}{32}\right) + \frac{\sigmamax^2}{\smoothparam}
  \exp\left(-\frac{\diffp}{2} k\right)
  \le \exp\left(-\frac{\nsubs}{32}\right)
  + \frac{\sigmamax^2}{\smoothparam} \exp\left(-\frac{\diffp \nsubs}{16}\right).
\end{align*}
Solving for the probability $\beta > 0$ gives the result.

\section{Proof of
  Proposition~\ref{proposition:objective-perturbation-subsample}
  and related results}
\label{sec:proof-objective-perturbation-subsample}

Before we proceed with the proof of
Proposition~\ref{proposition:objective-perturbation-subsample} proper
(see Section~\ref{sec:finally-proof-perturbation} below),
we present a few necessary results on subsampling and matrix
concentration, and we
also develop results on perturbations of solutions
of convex optimization problems (see Sec.~\ref{sec:perturbation-solutions}).

\subsection{Results on matrix concentration with subsampling}

Let $\mathbb{H}^d$ denote the collection of $d \times d$ Hermitian matrices.
Let $A_1, A_2, \ldots, A_n \in \mathbb{H}^d$
satisfy $\sum_{i = 1}^n A_i = 0$, and let $\pi$ be a
permutation of $n$. Define the random matrix
\begin{equation*}
  X \defeq \sum_{i = 1}^\subsize A_{\pi(i)},
\end{equation*}
that is, the sum of a random subset of $\subsize$ of the $n$ matrices.
We leverage the method of exchangeable pairs to show that
$X$ concentrates:
\begin{proposition}
  \label{proposition:concentration-permutation-matrices}
  Let $\{A_i\}_{i = 1}^n \subset \mathbb{H}^d$ and $X
  = \sum_{i = 1}^\subsize A_{\pi(i)}$ be as above, and
  additionally assume that $\opnorm{A_i - A_j} \le b$ for all $i, j$.
  Then
  \begin{equation*}
    \P(\lambdamin(X) \le -t)
    \le d \exp\left(-\frac{t^2}{\subsize b^2}\right)
    ~~ \mbox{and} ~~
    \P(\lambdamax(X) \ge t)
    \le d \exp\left(-\frac{t^2}{\subsize b^2}\right).
  \end{equation*}
\end{proposition}
\noindent
By rescaling the matrices, the following result is evidently equivalent
to the proposition: given
any set of $d \times d$
Hermitian matrices $A_1, \ldots, A_n$ summing to zero and with
$\opnorm{A_i - A_j} \le b$ for any pair $i, j$, we have
\begin{equation*}
  \P\left(\opnormbigg{\frac{1}{\subsize} \sum_{i = 1}^\subsize A_{\pi(i)}} \ge t
  \right) \le 2 d \exp\left(-\frac{\subsize t^2}{b^2}\right).
\end{equation*}
We also always have the Matrix-Hoeffding concentration inequality,
a simple version of which we state here.
\begin{corollary}[Matrix-Hoeffding, Cor.~4.2~\cite{MackeyJoChFaTr12}]
  \label{corollary:matrix-hoeffding}
  Let $X_i \in \mathbb{H}^d$ be independent, mean zero,
  and
  satisfy $X_i^2 \preceq A_i^2$ for each $i$. Let
  $S_n = \sum_{i = 1}^n X_i$. Then
  for all $t \ge 0$,
  \begin{equation*}
    \P\left(\opnorm{S_n} \ge t \right)
    \le 2d \exp\left(-\frac{t^2}{2 \sigma^2}\right)
    ~~ \mbox{where} ~~
    \sigma^2 = \opnormbigg{\sum_{i = 1}^n A_i^2}.
  \end{equation*}
\end{corollary}
In the context of subsampled Hessians in an M-estimation problem,
so long as $\nabla \loss$ is $\lipgrad$-Lipschitz, several
consequences follow this result. First,
observe that because $0 \preceq \ddot{\loss}_\theta(x) \preceq
\lipgrad I_d$ for any $\theta$ and $x$, then no matter the distribution
$P$,
\begin{equation*}
  \P\left(\opnorm{P_n \ddot{\loss}_\theta - P \ddot{\loss}_\theta}
  \ge t\right) \le 2d \exp\left(-\frac{m t^2}{2 \lipgrad^2}\right).
\end{equation*}
Additionally, we obtain that
\begin{equation*}
  \P\left(\opnorm{P\subsample_m \ddot{\loss}_\theta
    - P_n \ddot{\loss}_\theta} \ge t \mid P_n \right)
  \le 2d \exp\left(-\frac{m t^2}{\lipgrad^2}\right),
\end{equation*}
because $0 \preceq \ddot{\loss}_\theta(x) \preceq \lipgrad I_d$ for any
$\theta, x$. Applying
and applying the triangle inequality, Matrix-Hoeffding, and a union bound gives
\begin{equation*}
  \P\left(\opnorm{P_{n,m}^* \ddot{\loss}_\theta
    - P_n \ddot{\loss}_\theta} \ge t \mid P_n \right)
  \le 2d \left[\exp\left(-\frac{m t^2}{4 \lipgrad^2}\right)
    + \exp\left(-\frac{n t^2}{4 \lipgrad^2}\right)
    \right].
\end{equation*}
Said differently, conditional on $P_n$,
for any $\theta$ and for any $\delta > 0$ we have
\begin{equation*}
  \opnorm{P_{n,m}^* \ddot{\loss}_\theta - P_n \ddot{\loss}_\theta}
  \le O(1) \frac{\lipgrad}{\sqrt{m}}
  \sqrt{\log \frac{d}{\delta}}
\end{equation*}
with probability at least $1 - \delta$
(over only the resampling $P_{n,m}^*$). Note that because this statement
is conditional on $P_n$, we may take $\theta = \theta(P_n)$.

We can apply the same argument to subsampled gradient vectors,
where by a dilation argument (cf.~\cite{MackeyJoChFaTr12}), we obtain
for any $\theta$, $t \ge 0$, and $\delta > 0$, that
\begin{align*}
  \P\left(\ltwo{P_m\subsample \dot{\loss}_\theta - P_n \dot{\loss}_\theta}
  \ge t \mid P_n \right)
  & \le 2d \exp\left(-\frac{m t^2}{2 \lipobj^2}\right)
  \\
  \P\left(\ltwo{P_{n,\subsize}^* \dot{\loss}_\theta - P_n \dot{\loss}_\theta}
  \ge t \mid P_n\right)
  & \le 4d \exp\left(-\frac{\subsize t^2}{4 \lipobj^2}\right).
\end{align*}

\paragraph{Proof of
  Proposition~\ref{proposition:concentration-permutation-matrices}:}
We now turn to the proof of the proposition.
We leverage the method of exchangeable pairs for concentration
inequalitites, developed by \citet{Chatterjee07} and
\citet{MackeyJoChFaTr12}.
Recall that a pair $(X, X')$ of Hermitian matrices in $\R^{d \times d}$ is a
matrix Stein pair if for exchangeable random variables $(Z, Z')$, that is, a
pair satisfying $(Z, Z') \eqdist (Z', Z)$, we have $X = f(Z)$ and $X' =
f(Z')$ for some function $f$ and for some constant $0 < \alpha \le 1$ we
have
\begin{equation*}
  \E[X - X' \mid Z] = \alpha X
  ~~ \mbox{w.p.} ~ 1.
\end{equation*}
The conditional variance
$\Delta_X(Z) \defeq \frac{1}{2 \alpha} \E[(X - X')^2 \mid Z]$.
We have
\begin{corollary}[Mackey et al.~\cite{MackeyJoChFaTr12}, Theorem 4.1]
  \label{corollary:matrix-concentration}
  Let $(X, X')$ be a matrix Stein pair and assume that
  for some nonnegative $c, v$,
  \begin{equation*}
    \Delta_X(Z) \preceq c X + v I_d.
  \end{equation*}
  Then for all $t > 0$,
  \begin{equation*}
    \P(\lambdamin(X) \le -t) \le
    d \exp\left(-\frac{t^2}{2v}\right)
    ~~ \mbox{and} ~~
    \P(\lambdamax(X) \ge t)
    \le d \exp\left(-\frac{t^2}{2v + 2 c t}\right).
  \end{equation*}
\end{corollary}

\begin{lemma}
  \label{lemma:stein-pair-matrix}
  Let $\pi'$ be a permutation that transposes two elements of $\pi$, chosen
  at random. Define $X' = \sum_{i = 1}^m A_{\pi'(i)}$. Then the subsampling
  sum matrices $X, X'$ form a matrix Stein pair with $\alpha = \frac{2}{n -
    1}$, i.e.,
  \begin{equation*}
    \E[X - X' \mid \pi] = \frac{2}{n - 1} X.
  \end{equation*}
\end{lemma}
\begin{proof}
  Because we condition on $\pi$, without loss of generality we may assume
  that $\pi$ is the identity mapping with $\pi(i) = i$. Let
  $J$ and $K$ be the indices swapped in $\pi'$. Then we have
  $X = X'$ unless $J \le m$ and $K > m$ or $J > m$ and $K \le m$, as swapping
  two elements in the first $m$ or the last $n - m$ indices changes nothing.
  Thus by symmetry
  \begin{align*}
    \E[X - X']
    & = 2 \P(J \le m, K > m) \E[X - X' \mid J \le m, K > m] \\
    & = 2 \P(J \le m, K > m)
    \E\left[\sum_{i = 1}^m A_i - \bigg(\sum_{i = 1}^m A_i - A_J + A_K\bigg)
      \mid J \le m, K > m\right] \\
    & = 2 \P(J \le m, K > m)
    \E[A_J - A_K \mid J \le m, K > m].
  \end{align*}
  Now note that $\E[A_J \mid J \le m] = (1/m) \sum_{i = 1}^m A_i = (1/m) X$,
  while
  \begin{equation*}
    \E[A_K \mid K > m] = \frac{1}{n - m} \sum_{i > m} A_i
    = \frac{1}{n - m} \sum_{i = 1}^n A_i
    - \frac{1}{n - m} \sum_{i = 1}^m A_i
    = -\frac{1}{n - m} \sum_{i = 1}^m A_i
    = -\frac{1}{n - m} X,
  \end{equation*}
  because $\sum_{i = 1}^n A_i = 0$ by assumption.
  Substituting above yields
  \begin{equation*}
    \E[X - X'] = 2 \P(J \le m, K > m)
    \left(\frac{1}{m} X + \frac{1}{n - m} X \right)
    = 2 \P(J \le m, K > m)
    \frac{n}{n(n - m)} X.
  \end{equation*}
  Finally, we recognize that $\P(J \le m, K > m)
  = \frac{m}{n} \frac{n - m}{n - 1}$ and simplify.
\end{proof}
\begin{lemma}
  Assume that $\opnorm{A_i - A_j} \le b$ for all $i, j$. Then
  the conditional variance of $X$ satisfies
  \begin{equation*}
    \Delta_X
    = \frac{1}{2 \alpha} \E[(X - X')^2 \mid \pi]
    \preceq \frac{m(n - m)}{2 n} b^2 I_d
    \preceq \frac{m}{2} b^2 I_d.
  \end{equation*}
\end{lemma}
\begin{proof}
  We compute the conditional variance of $X$, noting that
  as in the proof of Lemma~\ref{lemma:stein-pair-matrix}
  we may w.l.o.g.\ assume $\pi$ is the identity permutation.
  Then again letting $J$ and $K$ denote the swapped indices,
  \begin{align*}
    \E[(X - X')^2]
    & = 2 \P(J \le m, K > m)
    \E[(A_J - A_K)^2 \mid J \le m, K > m] \\
    & = \frac{2 m(n - m)}{n(n - 1)}
    \E[(A_J - A_K)^2 \mid J \le m, K > m].
  \end{align*}
  Because $u^\top(A_j - A_k)^2 u
  = \ltwo{(A_j - A_k) u}^2 \le
  \opnorm{A_j - A_k}^2$ for $\ltwo{u} = 1$,
  we obtain
  \begin{equation*}
    \E[(X - X')^2]
    \preceq \frac{2 m(n - m)}{n(n - 1)}
    b^2 I_d.
  \end{equation*}
  Dividing by $2 \alpha = \frac{4}{n - 1}$ gives the result.
\end{proof}

To finally prove
Proposition~\ref{proposition:concentration-permutation-matrices}, apply
Corollary~\ref{corollary:matrix-concentration} with $v = \frac{m b^2}{2}$.

\subsection{Perturbation of convex minimizers}
\label{sec:perturbation-solutions}

We first present a lemma on the localization of minimizers
of general convex functions.
\begin{lemma}
  \label{lemma:convex-minimizers}
  Let $\poploss$ be a convex function with $\liphess$-Lipschitz
  continuous Hessian. Let $\theta$ be a point for which
  $\nabla^2 \poploss(\theta) \succeq \lambda I$, and
  let $w$ be any vector such that
  $\ltwo{\nabla \poploss(\theta) + w} < \frac{\lambda^2}{6 \liphess}$.
  Then $\theta_w = \argmin_\theta \{\poploss(\theta) + \<\theta, w\>\}$
  satisfies
  \begin{equation*}
    \ltwo{\theta - \theta_w}
    \le \frac{2 \ltwo{\nabla \poploss(\theta) + w}}{\lambda}.
  \end{equation*}
\end{lemma}
\begin{proof}
  Let $\gamma = \ltwo{\nabla \poploss(\theta) + w}$ for shorthand and
  $u$ be any unit vector.
  By a Taylor approximation, for $t \ge 0$ we have
  \begin{align*}
    \poploss(\theta + t u)
    + t \<w, u\>
    & \ge \poploss(\theta) + t \<\nabla \poploss(\theta) + w, u\>
    + \frac{\lambda - \liphess t}{2}
    t^2
    \ge \poploss(\theta) - t \gamma
    + \frac{\lambda - \liphess t}{2} t^2.
  \end{align*}
  By convexity, if we can show for some $t_0 > 0$ that
  \begin{equation*}
    -\gamma t_0
    + \frac{\lambda - \liphess t_0}{2} t_0^2 > 0,
  \end{equation*}
  then we immediately have $\poploss(\theta + t u)
  + t \<w, u\> > \poploss(\theta)$
  for all $t \ge t_0$, and so
  $\theta + t u$ cannot minimize
  $\poploss(\cdot) + \<w, \cdot\>$, that is,
  $\ltwo{\theta_w - \theta} \le t_0$.
  To that end, let us temporarily assume that
  $t \le \frac{2 \lambda}{3 \liphess}$, so that
  \begin{equation*}
    - \gamma t
    + \frac{\lambda - \liphess t}{2} t^2
    \ge - \gamma t + \frac{\lambda}{6} t^2.
  \end{equation*}
  Evidently, any $t > \frac{6 \gamma}{\lambda}$ implies that the
  right side is positive, so
  if there is some such $t$ also satisfying
  $t \le \frac{\lambda}{2 \liphess}$, that is,
  $\gamma < \frac{\lambda^2}{6 \liphess}$.
  Then we necessarily have
  $\ltwo{\theta_w - \theta} \le \frac{6 \gamma}{\lambda}$.

  Now we give the sharper guarantee.  Let $\Delta = \theta_w - \theta$. Then
  the optimality conditions for convex optimization guarantee
  \begin{equation*}
    \<\nabla \poploss(\theta_w) + w, \theta - \theta_w\> \ge 0,
    ~~ \mbox{i.e.} ~~
    \<\nabla \poploss(\theta_w) - \nabla \poploss(\theta),
    \theta - \theta_w\> \ge \<\nabla \poploss(\theta) + w, \theta_w - \theta\>.
  \end{equation*}
  Define $h(t) =
  \<\nabla \poploss(\theta + t \Delta), \Delta\>$, so that
  $h'(t) = \<\Delta, \nabla^2 \poploss(\theta + t \Delta), \Delta\>$.
  Then the preceding inequality is equivalent to
  \begin{equation*}
    \int_0^1 h'(t) dt = h(1) - h(0)
    = \<\nabla \poploss(\theta_w) - \nabla \poploss(\theta),
    \theta_w - \theta\> \le \<\nabla \poploss(\theta) + w, \theta - \theta_w\>.
  \end{equation*}
  But inspecting the last integral, we have
  \begin{align*}
    \int_0^1 h'(t) dt
    & = \<\Delta, \nabla^2 \poploss(\theta) \Delta\>
    + \left\<\Delta, \int_0^1
    \left(\nabla^2 \poploss(\theta + t \Delta)
    - \nabla^2 \poploss(\theta)\right) dt
    \Delta\right\> \\
    & \ge \lambda \ltwo{\Delta}^2
    - \liphess \ltwo{\Delta}^3
    \int_0^1 t dt
    = \lambda \ltwo{\Delta}^2
    \left(1 - \frac{\liphess \ltwo{\Delta}}{2} \right).
  \end{align*}
  Then either $\ltwo{\Delta} \ge \frac{\lambda}{\liphess}$, or we have
  \begin{equation*}
    \frac{\lambda}{2} \ltwo{\Delta}^2
    \le \<\nabla \poploss(\theta) + w, \theta - \theta_w\>
    \le \gamma \ltwo{\Delta},
  \end{equation*}
  implying $\ltwo{\Delta} \le \frac{2 \gamma}{\lambda}$.
  In the former case, we alread have that
  $\ltwo{\Delta} \le \frac{6 \gamma}{\lambda}
  < \frac{\lambda^2}{\liphess \lambda}
  = \frac{\lambda}{\liphess}$, a contradiction.
\end{proof}

An alternative perturbation result is also sometimes useful.
\begin{lemma}
  Let the conditions of Lemma~\ref{lemma:convex-minimizers}
  hold, and let $\theta\opt = \argmin \poploss(\theta)$.
  If $\nabla^2 \poploss(\theta\opt) \succeq \lambda I$
  and $\ltwo{w} < \frac{\lambda^2}{6 \liphess}$,
  then $\theta_w = \argmin_\theta \{\poploss(\theta) + \<w, \theta\>\}$
  satisfies
  \begin{equation*}
    \ltwo{\theta_w - \theta\opt} \le \frac{2 \ltwo{w}}{\lambda}.
  \end{equation*}
  If $\poploss$ is $\lambda$-strongly
  convex (not necessarily differentiable), then
  $\ltwo{\theta_w - \theta\opt} \le \frac{\ltwo{w}}{\lambda}$
  regardless.
\end{lemma}
\begin{proof}
  Observe that $\nabla \poploss(\theta\opt) = 0$, so that
  Lemma~\ref{lemma:convex-minimizers} gives the first claim.
  For the final inequality, note that there exist elements of the
  subdifferential (which we denote by
  $\nabla \poploss(\theta) \in \partial \poploss(\theta)$) such that
  \begin{equation*}
    \<\nabla \poploss(\theta_w) + w, \theta\opt - \theta_w\> \ge 0
    ~~ \mbox{and} ~~
    \<\nabla \poploss(\theta\opt), \theta_w - \theta\opt\> \ge 0.
  \end{equation*}
  Adding these and rearranging gives
  \begin{equation*}
    \<w, \theta\opt - \theta_w\>
    \ge  \<\nabla \poploss(\theta_w) - \nabla \poploss(\theta\opt),
    \theta_w - \theta\opt\>
    \ge \lambda \ltwo{\theta_w - \theta\opt}^2,
  \end{equation*}
  where we used strong convexity.
  Divide by $\ltwo{\theta_w - \theta\opt}$ and apply Cauchy-Schwarz.
\end{proof}

Using Lemma~\ref{lemma:convex-minimizers}, we can obtain
various finite-sample-valid expansions of the minimizers.
To do so, we give a general condition under which minimizers
of tilted and sampled losses are close to minimizers of a loss
based on a distribution $P$.
\begin{definition}
  \label{definition:accurate-sample}
  A sample-vector pair $(P_n, w)$ is
  \emph{$(\gamma, \eta, \lambda)$-accurate} for
  $P$ if for $\theta = \theta(P) = \argmin_\theta P\loss_\theta$, we have
  both
  \begin{enumerate}[leftmargin=*,label=(\roman*)]
  \item the well-conditioning $P \ddot{\loss}_\theta \succeq \lambda I$,
  \item the accuracy guarantees
    $\ltwos{P_n \dot{\loss}_\theta + w} \le \gamma$,
    $\opnorms{(P_n - P) \ddot{\loss}_\theta} \le \eta$, and
  \item \label{item:accuracy-parameter-relations}
    the parameter relations
    $\eta \le \frac{\lambda}{6}$,
    and
    $\gamma \le \frac{\lambda^2}{6 \liphess}$.
  \end{enumerate}
\end{definition}
\noindent
This definition, combined with Lemma~\ref{lemma:convex-minimizers},
gives a straightforward condition to check for accuracy of sample
minimizers.
\begin{lemma}
  Let Assumption~\ref{assumption:loss-stuff} hold and assume that
  $(P_n, w)$ be $(\gamma, \eta, \lambda)$-accurate for $P$. Then
  $\theta(P_n, w) \defeq \argmin_\theta \{P_n \loss_\theta + \<w, \theta\>\}$
  satisfies
  \begin{equation*}
    \ltwo{\theta(P_n, w) - \theta(P)} \le \frac{2 \gamma}{\lambda}
  \end{equation*}
  and
  \begin{equation*}
    \theta(P_n, w) - \theta(P)
    = -(P \ddot{\loss}_\theta)^{-1} (P_n \dot{\loss}_\theta + w)
    + E_w  (P_n \dot{\loss}_\theta + w),
  \end{equation*}
  where the error matrix $E_w$
  satisfies $\opnorm{E_w} \le
  \frac{\eta}{\lambda^2} + \frac{2 \liphess \gamma}{\lambda^3}$.
\end{lemma}
\begin{proof}
  Lemma~\ref{lemma:convex-minimizers} gives the first result.
  For the second let $\theta = \theta(P)$ and
  $\theta_w = \theta(P_n, w)$ for shorthand.
  As we are guaranteed the minimizer exists,
  we may perform a Taylor expansion to obtain
  \begin{equation*}
    0 = P_n \dot{\loss}_{\theta_w} + w
    = P_n \dot{\loss}_\theta + w + (P_n \ddot{\loss}_\theta + E)
    (\theta_w - \theta)
  \end{equation*}
  where $\opnorm{E} \le \liphess \norm{\theta - \theta_w}
  \le \frac{2 \liphess \gamma}{\lambda}
  \le \frac{1}{3} \lambda$
  by definition~\ref{definition:accurate-sample}.
  Rearranging, we find that
  \begin{equation}
    \label{eqn:hawaii}
    \theta_w - \theta
    = -(P_n \ddot{\loss}_\theta + E)^{-1} (P_n \dot{\loss}_\theta + w)
    = -(P \ddot{\loss}_\theta + (P_n - P) \ddot{\loss}_\theta + E)^{-1}
    (P_n \dot{\loss}_\theta + w).
  \end{equation}
  
  Now, we use a standard matrix identity, that for a matrix $H \succeq
  0$ and $D$ with $\opnorm{D} < \lambdamin(H)$,
  we have
  \begin{equation*}
    (H + D)^{-1}
    = H^{-1}
    + \sum_{i = 1}^\infty (-1)^i (H^{-1} D)^i H^{-1}
  \end{equation*}
  Under the stronger condition
  that $\opnorm{D} \le \half \lambdamin(H)$, then
  \begin{equation*}
    \opnorm{(H^{-1} D)^i} \le (\opnorm{D} / \lambdamin(H))^i
    \le 2^{-i},
  \end{equation*}
  and so
  \begin{equation*}
    (H + D)^{-1} = H^{-1} + E' H^{-1}
    ~~ \mbox{where} ~~
    \opnorm{E'} \le \frac{\opnorm{D} / \lambdamin(H)}{1
      - \opnorm{D} / \lambdamin(H)}
    \le \frac{2 \opnorm{D}}{\lambdamin(H)}.
  \end{equation*}
  Substituting this into the identity~\eqref{eqn:hawaii}
  with the definitions $H = P \ddot{\loss}_\theta$
  and $D = (P_n - P) \ddot{\loss}_\theta + E$,
  because $\opnorms{(P_n - P) \ddot{\loss}_\theta + E}
  \le \eta + \frac{\lambda}{3}
  \le \half \lambda$ and $P \ddot{\loss}_\theta \succeq \lambda I$
  by definition,
  we have
  \begin{equation*}
    \theta_w - \theta
    = -(P \ddot{\loss}_\theta)^{-1} (P_n \dot{\loss}_\theta + w)
    + E_w (P_n \dot{\loss}_\theta + w),
  \end{equation*}
  where
  \begin{equation*}
    \opnorm{E_w} \le
    \frac{2 \opnorms{(P_n - P) \ddot{\loss}_\theta + E}}{
      \lambdamin(H)} \frac{1}{\lambdamin(H)}
    \le 
    \frac{\eta}{\lambda^2} + \frac{2 \liphess \gamma}{\lambda^3},
  \end{equation*}
  as desired.
\end{proof}

The remainder of the proof simply demonstrates that
the sampled and resampled empirical distributions are all
accurate (Definition~\ref{definition:accurate-sample}) for
their respective ``initial'' distributions, that is, we show that
each of the following holds with high probability:
\begin{enumerate}[leftmargin=*,label=\arabic*.]
\item $P_n$ is accurate for $P$
\item $P_\subsize\subsample$ is accurate for $P_n$
\item $P_{\subsize,n}^*$ is accurate for $P_\subsize\subsample$
\end{enumerate}
Once each of the three of these holds,
Proposition~\ref{proposition:objective-perturbation-subsample}
is then more or less immediate.

\begin{lemma}
  \label{lemma:from-P-to-Pn}
  Let $\delta > 0$ and $w \in \R^d$
  and $\lambda = \lambdamin(P)$. Then with probability at least $1 -
  \delta$ over the
  draw of $P_n$, the pair
  $(P_n, w)$ is $(\gamma, \eta, \lambda)$-accurate for $P$,
  where
  \begin{equation*}
    \gamma = O(1)
    \left(\frac{\lipobj \sqrt{\log \frac{d}{\delta}}}{\sqrt{n}}
    + \ltwo{w} \right)
    ~~ \mbox{and} ~~
    \eta = O(1) \frac{\lipgrad \sqrt{\log \frac{d}{\delta}}}{\sqrt{n}},
  \end{equation*}
  so long as $\gamma$ and $\eta$ satisfy the parameter
  relationships~\ref{item:accuracy-parameter-relations} in
  Definition~\ref{definition:accurate-sample}.
\end{lemma}
\begin{proof}
  Apply Corollary~\ref{corollary:matrix-hoeffding}
  and the discussion following it.
\end{proof}
\noindent
If we let $\mc{E}_1$ be the event
that $(P_n, w)$ is $(\gamma, \eta, \lambda)$-accurate
for $\lambda = \lambdamin(P)$, then
we have
\begin{equation*}
  \lambdamin(P_n)
  \ge \lambdamin(P_n \ddot{\loss}_{\theta(P)})
  - \frac{2 \lipgrad \gamma}{\lambda}
  \ge \lambdamin(P) - \eta
  - \frac{2 \lipgrad \gamma}{\lambda}
  \ge
  \half \lambdamin(P)
\end{equation*}
by the inequalities~\ref{item:accuracy-parameter-relations},
and $\P(\mc{E}_1) \ge 1 - \delta$
so long as $\ltwo{w}$ is small enough that $\gamma, \eta$
satisfy these inequalities.

\begin{lemma}
  \label{lemma:from-Pn-to-Psub}
  Let $\delta > 0$ and $w \in \R^d$
  and $\lambda = \lambdamin(P_n)$. Then
  with probability at least $1 - \delta$ over the subsampling
  $P_\subsize\subsample$,
  the pair $(P_\subsize\subsample, w)$ is
  $(\gamma, \eta, \lambda)$-accurate
  `for $P$, where
  \begin{equation*}
    \gamma = O(1)
    \left(\frac{\lipobj \sqrt{\log \frac{d}{\delta}}}{\sqrt{\subsize}}
    + \ltwo{w} \right)
    ~~ \mbox{and} ~~
    \eta = O(1) \frac{\lipgrad \sqrt{\log \frac{d}{\delta}}}{\sqrt{\subsize}},
  \end{equation*}
  so long as $\gamma$ and $\eta$ satisfy the parameter
  relationships~\ref{item:accuracy-parameter-relations} in
  Definition~\ref{definition:accurate-sample}.
\end{lemma}
\begin{proof}
  Apply Proposition~\ref{proposition:concentration-permutation-matrices}.
\end{proof}
\noindent
As in the discussion following Lemma~\ref{lemma:from-P-to-Pn},
we see that if we let $\mc{E}_2$ be the event that
$(P_n, w)$ is $(\gamma,\eta,\lambda)$-accurate
for $\lambda = \lambdamin(P)$, then on $\mc{E}_2$ we have
\begin{equation*}
  \lambdamin(P_\subsize\subsample) \ge \half \lambdamin(P_n)
\end{equation*}
in complete parallel with event $\mc{E}_1$. Finally,
note that Lemma~\ref{lemma:from-P-to-Pn} applies
equally to the sampling of $P_{\subsize,n}^*$ from $P_\subsize\subsample$,
and we therefore have the following
summary lemma.

\begin{lemma}
  \label{lemma:summary-lemma}
  There exists a numerical constant $C < \infty$ such that
  the following holds. Let $\delta > 0$. Assume that the vector
  $w$ is small enough and sample sizes $\subsize, n$ are large
  enough that
  \begin{equation*}
    \gamma \defeq
    \ltwo{w} + \frac{\lipobj \sqrt{\log \frac{d}{\delta}}}{\sqrt{\subsize}}
    \le C \cdot \frac{\lambdamin^2(P)}{\liphess}.
  \end{equation*}
  Then
  with probability at least $1 - 3 \delta$
  over the draw of $P_n$ from $P$, $P_\subsize\subsample$ from $P_n$,
  and $P_{\subsize,n}^*$ from $P_\subsize\subsample$,
  there are remainders $R_n(w)$ and $R^*_{\subsize,n}(w)$ satisfying
  \begin{equation*}
    \max\left\{\ltwo{R_n(w)},
    \ltwo{R_{\subsize,n}^*(w)}\right\}
    \le O(1) \cdot
    \max\left\{\frac{\lipgrad}{\lambda^2},
    \frac{\lipobj \liphess}{\lambda^3}\right\}
    \cdot \frac{\sqrt{\log \frac{d}{\delta}}}{\sqrt{n}}
    \cdot \left(\ltwo{w}
    + \frac{\lipobj \sqrt{\log \frac{d}{\delta}}}{\sqrt{n}}\right),
  \end{equation*}
  such that
  the following equalities hold:
  \begin{align*}
      \theta(P_n, w) - \theta(P)
      & = -(P \ddot{\loss}_{\theta(P)})^{-1} (P_n \dot{\loss}_{\theta(P)}
      + w)
      + R_n(w) \\
    \theta(P_{\subsize,n}^*, w) - \theta(P_\subsize\subsample)
    & = -(P_\subsize\subsample \ddot{\loss}_{\theta(P_\subsize\subsample)})^{-1}
    (P_{\subsize,n}^* \dot{\loss}_{\theta(P_\subsize\subsample))} + w)
    + R_{\subsize,n}^*(w).
  \end{align*}
  Additionally, $P_n \ddot{\loss}_{\theta(P)}
  \succeq \frac{3}{4} \lambdamin(P)$, and
  $P_\subsize\subsample \ddot{\loss}_{\theta(P_\subsize\subsample)}
  \succeq \half \lambdamin(P)$.
\end{lemma}
\begin{proof}
  We simply track definitions and substitute constants
  from Lemmas~\ref{lemma:from-P-to-Pn} and~\ref{lemma:from-Pn-to-Psub}.
\end{proof}

\subsection{Proof of
  Proposition~\ref{proposition:objective-perturbation-subsample}}
\label{sec:finally-proof-perturbation}

Apply Lemma~\ref{lemma:summary-lemma}, noting that it is uniform over $w$,
and replace the losses $\loss_\theta$ with $\theta \mapsto \loss_\theta +
\frac{\lambdareg}{2} \ltwo{\theta}^2$.  By the implicit function theorem,
$\theta_\lambda(P) \defeq \argmin_\theta \{\poploss_P(\theta) +
\frac{\lambda}{2} \ltwo{\theta}^2\}$ is locally $\mc{C}^1$ in $\lambda$, and
so $\theta_\lambda(P)$ is bounded and the result follows.

\section{Proofs of results related
  to Edgeworth expansions}

\subsection{Proof of Proposition~\ref{proposition:edgeworth}}
\label{sec:proof-edgeworth}

Before proving the proposition proper, we provide a few
auxiliary helper lemmas.

\begin{lemma}
  \label{lemma:poly-phi-lipschitz}
  Let $p$ be a polynomial of finite degree and $c > 0$.  Then $\sup_t p(t)
  e^{-c t^2} < \infty$.  If $p_i$ are polynomials of finite degree and $c_i
  > 0$, then $f(t) = \Phi(c_0 t) + \sum_{i = 1}^k p_i(t) \phi(c_i t)$ is
  bounded, Lipschitz continuous, and has Lipschitz derivatives of all
  orders.
\end{lemma}
\begin{proof}
  Using the Taylor series for $e^{c t^2} = 1 + \sum_{i = 1}^\infty \frac{c^i
    x^{2i}}{i!}  \ge 1 + \frac{c^d x^{2d}}{d!}$, valid for any degree $d \in
  \N$, we have $p(t) e^{-t^2} \le \frac{c^{-d} d! p(t)}{c^{-d} d! +
    t^{2d}}$.  Let $d$ be any value greater than half the degree of $p$.
  For the second claim, note that $\Phi'(c t) = c \phi(c t)$, and $\phi'(c
  t) = -c t \phi(ct)$ so that $\phi^{(k)}(ct) = \mbox{poly}_k(t) \phi(c t)$,
  where $\mbox{poly}_k$ denotes a polynomial of degree $k$.  Then $\sup_t
  |f^{(k)}(t)| < \infty$ for all $k \in \N$ by the first claim of the lemma,
  giving the result.
\end{proof}

\begin{lemma}
  \label{lemma:intervals-to-intervals}
  Let $X, Y$ be random variables and $a \in \R$. Then
  $\P(|X - a| \le |Y|, |Y| \le \eta)
  \le \P(X \in [a - \eta, a + \eta])$.
\end{lemma}
\begin{proof}
  The event that $|X - a| \le \eta$ contains
  the event $|X - a| \le |Y|$ and $|Y| \le \eta$.
\end{proof}

\begin{lemma}
  \label{lemma:from-non-private-edgeworth-to-private}
  Let $\unstudstatest
  = \sqrt{n}(\param(P) - \param(P_n))$
  be asymptotically normal, with $\unstudstatest \cd \normal(0, \sigma^2)$.
  Assume $\unstudstatest$ and $\blbstatistic$ admit
  Edgeworth expansions~\eqref{eqn:edgeworth} of order $k$,
  and assume that $\tparam(P_n)$ is $(\priverr, \privprob)$-accurate
  for $\param(P_n)$. Let $\unstudstatpriv = \sqrt{n}(\param(P)
  - \tparam(P_n))$ and
  $\tblbstatistic = \sqrt{n}(\param(P_\subsize) - \tparam(P_{\subsize,n}^*))$.
  Then there exists a constant $C$ depending only
  on the polynomials $p_i$ defining the Edgeworth
  expansion~\eqref{eqn:edgeworth-pop} such that for all $t \in \R$,
  \begin{equation*}
    \abs{\P\bigg(\frac{\unstudstatpriv}{\sigma} \le t\bigg)
      - \Phi(t) - \bigg(\sum_{i = 1}^k
      n^{-i/2} p_i(t) \bigg) \phi(t)}
    \le \privprob +
    C \left(\sqrt{n} \cdot \priverr + n^{-\frac{k+1}{2}}\right).
  \end{equation*}
  Additionally, there exists a constant $\what{C}$
  depending only on the polynomials
  $\what{p}_i$ defining the Edgeworth
  expansion~\eqref{eqn:edgeworth-blb} such that
  for a remainder $\remainder_{n,k}(t) =
  O_P(n^{-\frac{k + 1}{2}})$ uniformly in $t$,
  \begin{equation*}
    \abs{\P\bigg(\frac{\tblbstatistic}{\sigma(P_\subsize)} \le t
      \mid P_\subsize\bigg)
      - \Phi(t) - \bigg(\sum_{i = 1}^k
      n^{-i/2} \what{p}_i(t) \bigg) \phi(t)}
    \le \remainder_{n,k}(t) + \privprob + C \sqrt{n} \cdot \priverr.
  \end{equation*}
\end{lemma}
\begin{proof}
  We prove the first expansion first. Without loss of generality,
  we assume $\sigma = 1$ to simplify notation.
  Adding and subtracting, we observe that
  \begin{equation*}
    \P(\unstudstatpriv \le t)
    = \P(\unstudstatest \le t) +
    \left(\P(\unstudstatpriv \le t) - \P(\unstudstatest \le t)\right),
  \end{equation*}
  so that it suffices to bound the latter probability.
  For this, observe that for $W = \tparam(P_n) - \param(P_n)$,
  we have
  \begin{align*}
    \lefteqn{\P(\unstudstatpriv \le t) - \P(\unstudstatest \le t)
      = \P(\unstudstatest \le t + W \sqrt{n})
      - \P(\unstudstatest \le t)} \\
    & = \P(\unstudstatest \le t + W \sqrt{n},
    |W|  \le \priverr)
    - \P(\unstudstatest \le t, |W| \le \priverr)
    + \P(\unstudstatest \le t + W \sqrt{n},
    |W| > \priverr)
    - \P(\unstudstatest \le t, |W| > \priverr) \\
    & = \P(\unstudstatest \in [t, t + W \sqrt{n}],
    |W| \le \priverr)
    + \P(\unstudstatest \in [t, t + W \sqrt{n}],
    |W| > \priverr).
  \end{align*}
  Now we use Lemma~\ref{lemma:intervals-to-intervals} to
  observe that
  \begin{equation*}
    \P(\unstudstatest \in [t, t + W \sqrt{n}], |W| \le \priverr)
    \le \P(\unstudstatest \in [t - \priverr \sqrt{n},
      t + \priverr \sqrt{n}]),
  \end{equation*}
  and by assumption $\P(|W| \ge \priverr) \le \privprob$. So
  \begin{equation*}
    \left|\P(\unstudstatpriv \le t) - \P(\unstudstatest \le t)
    \right|
    \le \P(\unstudstatest \in [t - \priverr \sqrt{n}, t + \priverr \sqrt{n}])
    + \privprob.
  \end{equation*}
  Lastly, observe that
  by Lemma~\ref{lemma:poly-phi-lipschitz}
  there is some constant $C$, depending on the polynomials
  defining the Edgeworth expansion~\eqref{eqn:edgeworth-pop},
  for which
  $\P(\unstudstatest \in [t - \delta, t + \delta])
  \le C (\delta + n^{-\frac{k + 1}{2}})$ simultaneously
  for all $t \in \R$ and $\delta \ge 0$.

  For the second equality, the proof is \emph{mutatis mutandis} completely
  analogous, except that we rely instead on the
  expansion~\eqref{eqn:edgeworth-blb}, and give the argument only on the
  (eventual) event that $\sigma^2(P_\subsize) > 0$.
\end{proof}

We can now give the proof of Proposition~\ref{proposition:edgeworth}
proper. By Lemma~\ref{lemma:from-non-private-edgeworth-to-private},
we have
\begin{equation*}
  \left|\P\left(\unstudstatpriv \le t \right)
  - \P\left(\tblbstatistic \le t \mid P_\subsize\right) \right|
  \le \left|\Phi\left(\frac{t}{\sigma}\right)
  - \Phi\left(\frac{t}{\sigma(P_\subsize)}\right)\right|
  + O_P\left(n^{-1/2} + \sqrt{n} \cdot \priverr\right) + \privprob.
\end{equation*}
Because $\Phi$ has Lipschitz derivatives of all orders, that
$\sigma^2(P_\subsize) = \sigma^2 + O_P(\subsize^{-1/2})$ then yields
the proposition.

\subsection{Proof of Proposition~\ref{proposition:edgeworth-moment}}
\label{sec:proof-edgeworth-moment}

Throughout this argument, we understand all steps to hold
asymptotically conditionally (on the subsample $P_\subsize$) almost surely,
so we leave it tacit.
Let $\tblbstatistic = \sqrt{n}(\tparam(P_{\subsize,n}^*) -
\param(P_\subsize))$ and $\blbstatistic = \sqrt{n}(\param(P_{\subsize,n}^*)
- \param(P_\subsize))$ as usual.  Then
(ignoring Monte Carlo error), we have
\begin{equation*}
  \bootvaralg(\param(\cdot), \tparam(\cdot), P_\subsize, n)
  = \E[(\tblbstatistic)^2 \mid P_\subsize],
\end{equation*}
so it is sufficient to control
$\sigma^2(P_\subsize) - \E[(\tblbstatistic)^2 \mid P_\subsize]$.
We first observe that
$\tblbstatistic = \blbstatistic + \sqrt{n} \privtononpriv$, so
that
\begin{align*}
  \E[(\tblbstatistic)^2 - (\blbstatistic)^2 \mid P_\subsize]
  & = n\E[\privtononpriv^2 \mid P_\subsize]
  + 2 \sqrt{n} \E[\blbstatistic \privtononpriv \mid P_\subsize] \\
  & \le n\E[\privtononpriv^2 \mid P_\subsize]
  + 2 \sqrt{n} \E[(\blbstatistic)^2 \mid P_\subsize]^{1/2}
  \E[\privtononpriv^2 \mid P_\subsize]^{1/2}
\end{align*}
by Cauchy-Schwarz. By the assumed Edgeworth
expansions~\eqref{eqn:edgeworth-variance}, we have $\E[(\blbstatistic)^2
  \mid P_\subsize] = O_P(1)$, and as we assume
$n \E[\privtononpriv^2 \mid P_\subsize] \cp 0$, we obtain
\begin{equation*}
  \left|\E[(\tblbstatistic)^2 - (\blbstatistic)^2 \mid P_\subsize]\right|
  \le O_P(1)
  \sqrt{n \E[\privtononpriv^2 \mid P_\subsize]}.
\end{equation*}
Now we use the resampling Edgeworth
expansions~\eqref{eqn:edgeworth-variance-blb}, which give that
\begin{equation*}
  \E[(\blbstatistic)^2 \mid P_\subsize]
  = \sigma^2(P_\subsize) + O_P(1/n),
\end{equation*}
and using the standing Edgeworth assumption that
$\sigma^2(P_\subsize) = \sigma^2 + O_P(1/\sqrt{\subsize})$,
we obtain the result.

\section{Experimental Details}\label{appen:expts}

\newcommand{\paramdiffp}{\diffp_\theta}
\newcommand{\sigmadiffp}{\diffp_\sigma}
\newcommand{\quantdiffp}{\diffp_\textup{q}}

We detail our
experimental methodology. We
describe algorithms and hyperparameters in \Cref{appen:alg-expts},
and we describe the synthetic and real datasets in
\Cref{appen:datasets}, providing additional ablation plots
in Section~\ref{sec:ablations}.
We choose $\alpha = .05$ to target 95\% confidence intervals for each.

\subsection{Algorithms}\label{appen:alg-expts}

As we mention in the experiments section, each trial uses a total
privacy budget $\totdiffp = 2 \diffp$, where we $\tparam(P_n)$ is
$\diffp$-differentially private, and the interval $I_{\what{t}}$ is
$\diffp$-differentially private.

\subsubsection{\textsc{NonprivateBaseline}}

We bootstrap $\wt{U}_n^* \defeq \sqrt{n}(\param(P_n) - \tparam(P_n^*))$, the
standard bootstrap resample of $P_n$, reporting the actual percentiles of
$\wt{U}_n^*$ conditional on $P_n$ as the confidence set.  We perform
$\nmontecarlo= \min\{10000,\max\{100,\frac{n^{1.5}}{\log(n)}\}\}$ Monte
Carlo iterations; increasing the number iterations beyond this value changes
coverage little.

\subsubsection{$\blbvar$}

We use $\blbvar$ to estimate the second moment of the
statistic $\unstudstatpriv = \sqrt{n}(\param - \tparam(P_n))$ and use normal
approximation (Corollary~\ref{corollary:coverage-from-variance}) to
construct the confidence interval.  We use $\nmontecarlo =
\min\{10000,\max\{100,\frac{n^{1.5}}{\nsubs \log(n)}\}\}$ Monte Carlo
iterations for each little bootstrap.

To instantiate the algorithm, we need to set two other hyperparameters: the
number of partitions of the dataset $\nsubs$ and the upper bound on the
variance $\sigmamax^2$; we let $R^2 = \sigmamax^2 / \sigma^2$, where
$\sigma^2$ is the true variance so that $R^2$ is the ratio of the upper
bound used in the algorithm and the true variance. We choose $\nsubs \asymp
\frac{1}{\diffp} \log n$ as Theorems~\ref{thm:blb-abthreshmed}
and~\ref{thm:blb-var} suggest, so that we vary the particular multiplier $K$
in $\nsubs = \lfloor \frac{K \log n}{\diffp}\rfloor$.  We investigate the
effect of varying $K$ in \Cref{fig:blbvar-R-10,fig:blbvar-R-300}. For a
ratio $R = \sigmamax / \sigma$ (10 in \Cref{fig:blbvar-R-10}, indicating a
$100\times$ overestimate of the variance), all values of $K$ have similar
performance, with performance dropping slightly as we increase the value of
$K$. For a larger $R$ (300 in \Cref{fig:blbvar-R-300}), we see that smaller
values of $K$ overestimate the variance and output very wide intervals with
higher than required coverage. We also vary the upper bound $\sigmamax^2 =
R^2\sigma^2$ in \Cref{fig:blbvar-K-6,fig:blbvar-K-10}. While a smaller value
of $R$ is obviously better for the performance of the algorithm at all
values of $K$, the performance of the algorithm is stable for higher values
of $K$ (i.e., more subsamples $\nsubs$).

Because the optimal value of multiplier $K$ varies, we default to $K=10$ and
$R \approx 50$ for illustration of results. The probability of catastrophic
failure (outputting very wide confidence sets) of the private median
subroutine has bound $\failprob \le n R^2 \sigma^2 e^{-K \log(n)/(2\diffp)}$
when we set $\nsubs = \frac{K \log n}{\diffp}$ and use smoothing parameter
$\smoothparam = n^{-1}$ (again, see Theorem~\ref{thm:blb-var}).  Choosing
higher $K$ reduces the probability of catastrophic failure but
simultaneously decreases the subsample size $\subsize = \frac{n}{\nsubs}$ in
each little bootstrap, which may reduce accuracy. Setting
$K$ slightly higher than 10 (say
12 or 14) allows commpensation for poor bounds on the true variance.

\begin{figure}
  \begin{subfigure}{0.5\textwidth}
    \includegraphics[width=\linewidth]{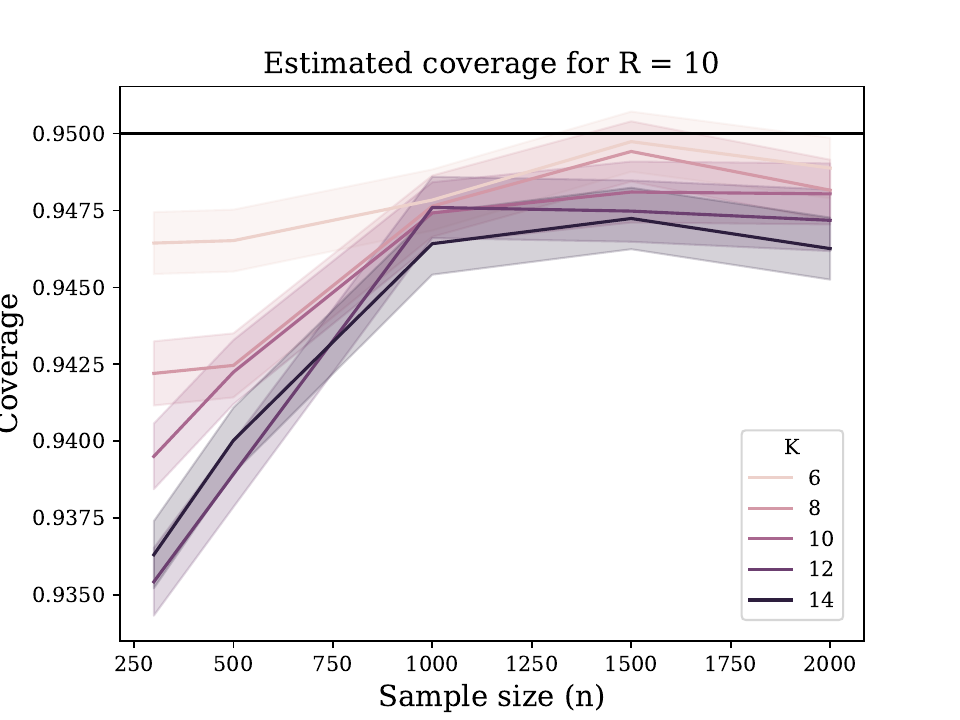}
    \caption{Coverage varying $K$, fixing $R = 10$} \label{fig:blbvar-R-10}
  \end{subfigure}\hspace*{\fill}
  \begin{subfigure}{0.5\textwidth}
    \includegraphics[width=\linewidth]{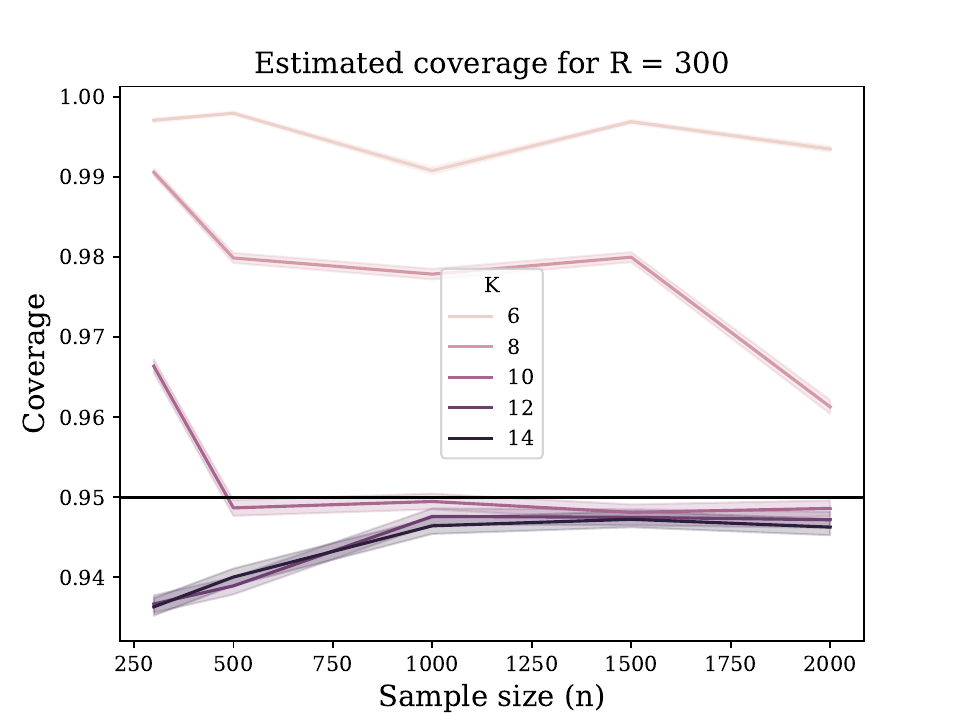}
    \caption{Coverage varying $K$, fixing $R = 300$} \label{fig:blbvar-R-300}
	\end{subfigure}
    \caption{Hyperparameter sensitivity of $\blbvar$ on mean estimation with $\totdiffp = 8$}
\end{figure}

\begin{figure}
  \begin{subfigure}{0.5\textwidth}
    \includegraphics[width=\linewidth]{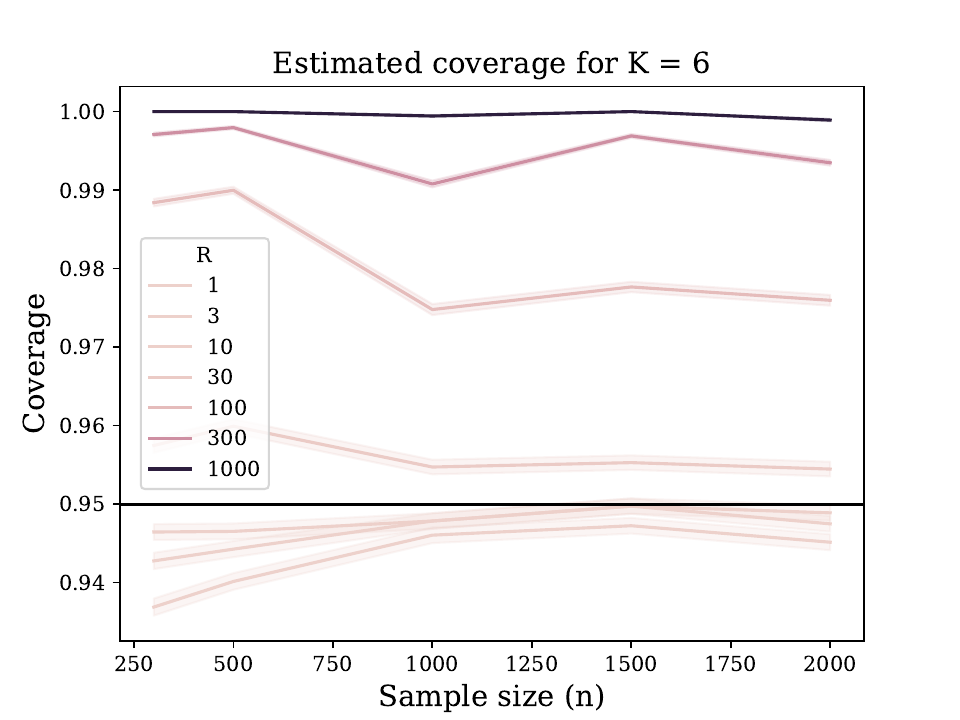}
    \caption{Coverage for varying $R$, fixing $K = 6$} \label{fig:blbvar-K-6}
  \end{subfigure}\hspace*{\fill}
  \begin{subfigure}{0.5\textwidth}
    \includegraphics[width=\linewidth]{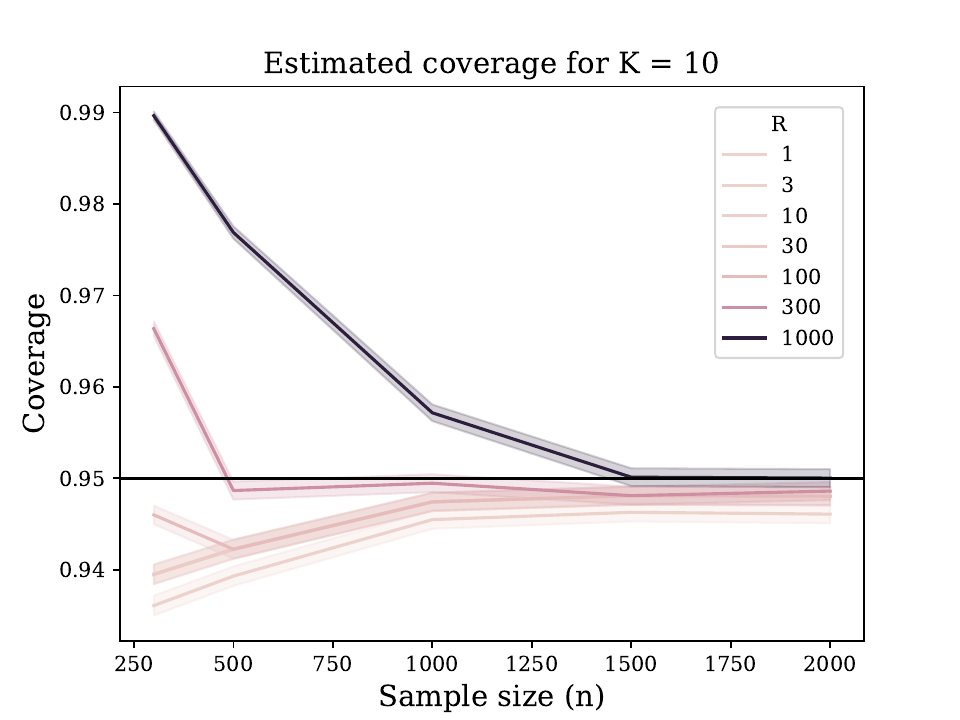}
    \caption{Coverage varying $R$, fixing $K = 10$} \label{fig:blbvar-K-10}
  \end{subfigure}
  \caption{Hyperparameter sensitivity of $\blbvar$ on mean estimation with $\totdiffp = 8$}
\end{figure}

\subsubsection{$\blbquant$}

We use $\blbquant$ to estimate the quantiles of the distribution of the
centered statistic $\unstudstatpriv = \sqrt{n}(\param - \tparam(P_n))$ and
use the private percentile method to construct the confidence interval. We
use $\nmontecarlo = \min\{10000,\max\{100,\frac{n^{1.5}}{\nsubs
  \log n}\}\}$ Monte Carlo iterations for each little bootstrap.

To instantiate the algorithm, we need to set the number of partitions
$\nsubs$ of the dataset and the sequence of sets $I_t$ amongst which we find
the first set with coverage more than $1 - \alpha$. We choose $I_t =
[-th,th]$ for $t = 1,2,\dots$ and $h = c/\sqrt{n}$ as we discuss after
Theorem~\ref{thm:blb-abthreshmed} and
Corollary~\ref{corollary:coverage-from-percentile}, studying the effect of
varying $c$ in \cref{fig:blbquant-K-6,fig:blbquant-K-10,fig:blbquant-K-14}.
Increasing $c$ means that fewer sets are required to achieve coverage
$\P(\unstudstatpriv \in I_t) \ge 1 - \alpha$, though too large of a $c$
increases the increment between $\P(\unstudstatpriv \in I_t)$ and
$\P(\unstudstatpriv \in I_{t + 1})$, yielding overcoverage. As in
the experiments with $\blbvar$, larger values of the muliplier $K$ in the
choice $\nsubs = \floor{\frac{K \log n }{\diffp}}$ yield more accurate
coverage (see the vertical axes in Figure~\ref{fig:hyperparameter-blb-quant}).
We conduct all mean and median estimation experiments with $c=1$ to
remain agnostic as the plots suggest it is a reasonable default.

As we do for $\blbvar$, we choose the number of partitions $\nsubs = \lfloor
\frac{K \log n}{\diffp}\rfloor$, varying $K$ in
\Cref{fig:blbquant-c-0.3,fig:blbquant-c-1}. As in the case of $\blbvar$,
beyond a threshold value for $K$, we find no significant improvement in the
coverage error and below that value the sensitivity to the granularity $c$
of the intervals $I_t$ is much higher. We choose $K = 10$ for all problem
settings.

\begin{figure}
  \begin{subfigure}{0.33\textwidth}
    \includegraphics[width=\linewidth]{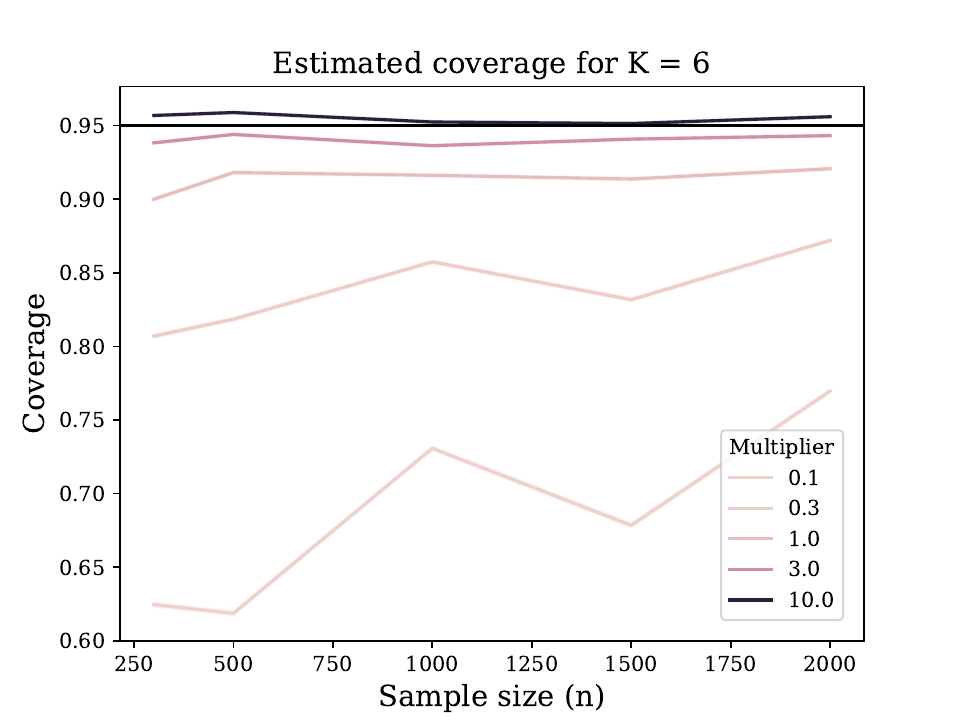}
    \caption{\small Varying $c$, fixing $K = 6$} \label{fig:blbquant-K-6}
  \end{subfigure}\hspace*{\fill}
  \begin{subfigure}{0.33\textwidth}
    \includegraphics[width=\linewidth]{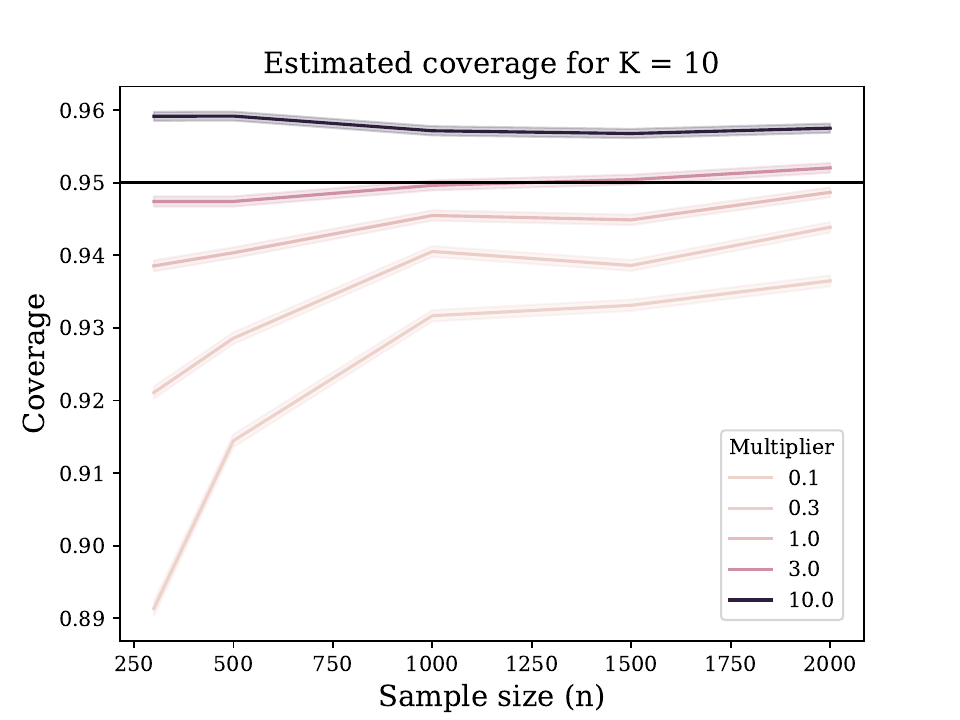}
    \caption{\small Varying $c$, fixing $K = 10$} \label{fig:blbquant-K-10}
  \end{subfigure}\hspace*{\fill}
  \begin{subfigure}{0.33\textwidth}
    \includegraphics[width=\linewidth]{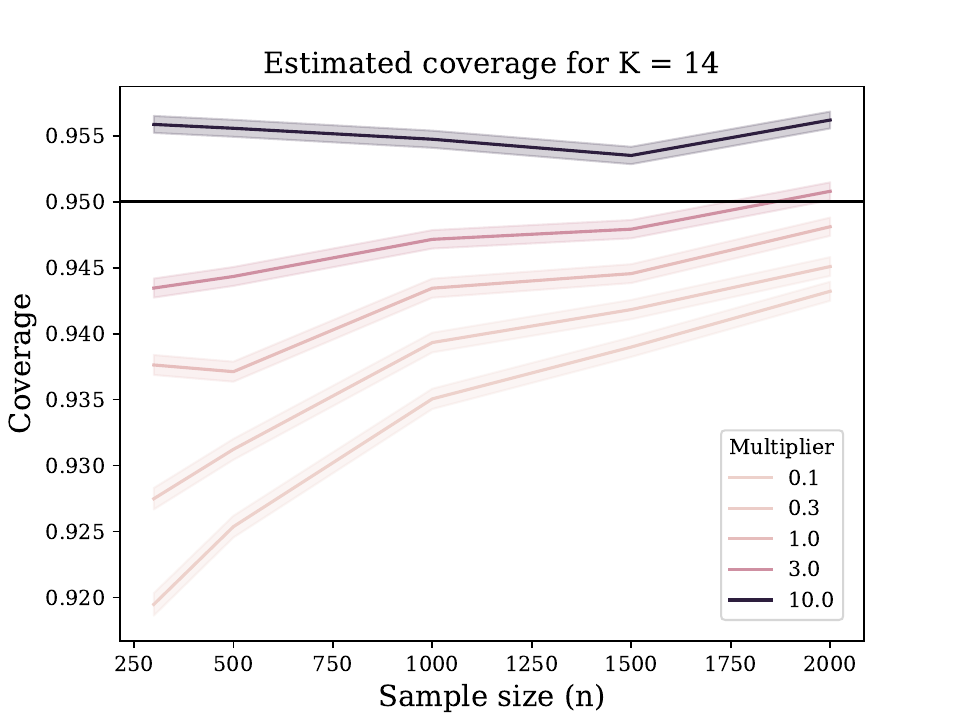}
    \caption{\small Varying $c$, fixing $K = 14$} \label{fig:blbquant-K-14}
  \end{subfigure}
  \caption{\label{fig:hyperparameter-blb-quant} Coverage rates
    for $\blbquant$ on mean estimation with $\totdiffp = 8$.
    Each plot varies $c$ in the interval sets $I_t = [-c t /\sqrt{n},
    ct / \sqrt{n}]$, fixing multiplier $K$.}
\end{figure}

\begin{figure}
	\begin{subfigure}{0.45\textwidth}
		\includegraphics[width=\linewidth]{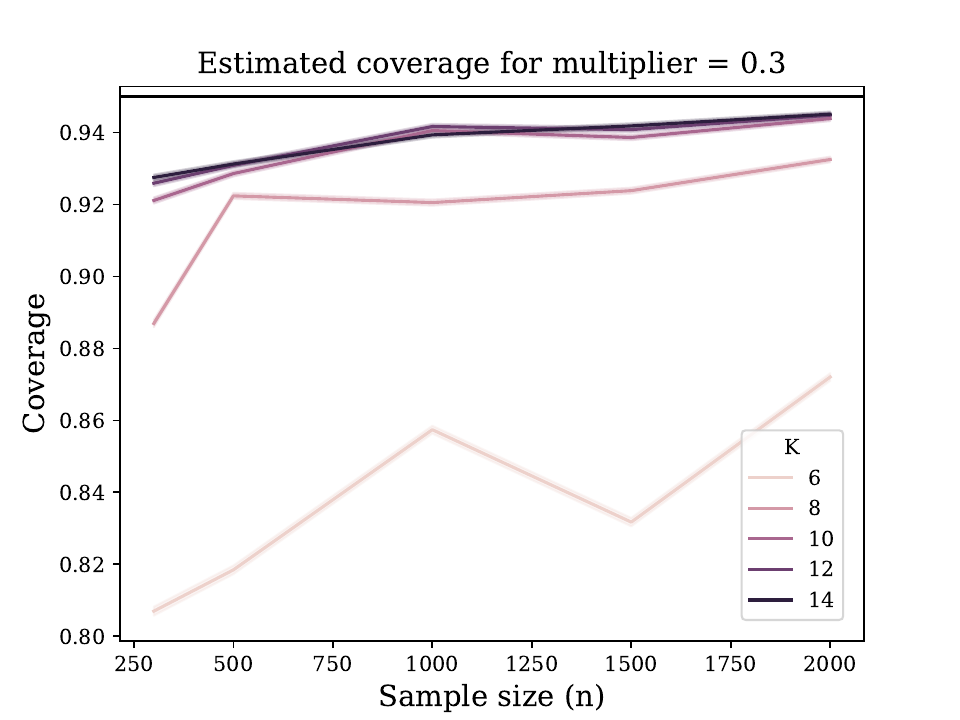}
		\caption{\small Coverage for different K and R = 10} \label{fig:blbquant-c-0.3}
	\end{subfigure}\hspace*{\fill}
	\begin{subfigure}{0.45\textwidth}
		\includegraphics[width=\linewidth]{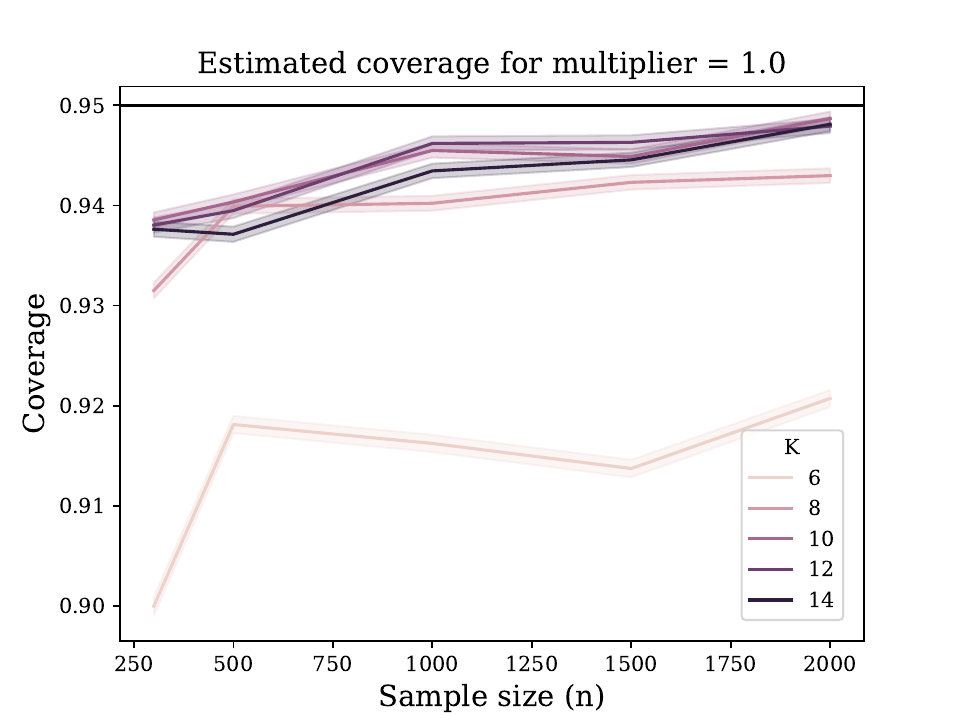}
		\caption{\small Coverage for different K and R = 300} \label{fig:blbquant-c-1}
	\end{subfigure}
    \caption{Hyperparameter sensitivity of $\blbquant$ on mean estimation with $\totdiffp = 8$}
\end{figure}

\subsubsection{\gvdp}

We also compare against the ``Generally Valid Differential Privacy'' (\gvdp)
algorithm \citet{CovingtonHeHoKa21} propose. We use a slight modification of
the authors' implementation. We use the same
hyperparameters as suggested in that code and run the mean estimation
algorithm CoinPress for $t = 5$ iterations. We tune the number of subsets
$\nsubs$ and find that although the algorithm is sensitive to the number of
subsets at small $n$ values (in the range we consider), using $s = \sqrt{n}$
works well for the mean estimation setting. Since the algorithm is defined
for $\rho$-concentrated DP we find a $\rho_{\rm tot}$ such that an algorithm
satisfying $(\totdiffp,\delta)$ also satisfies $\rho_{\rm tot}$-concentrated
DP for $\delta = 1/n^{1.1}$. Lastly, since \gvdp also uses an
upper bound on the variance of the statistic, we set it to be 10 times the
true variance.

\subsection{Experimental Settings}\label{appen:datasets}

\paragraph{Mean Estimation:}
For mean estimation, we do experiments on a synthetic dataset for which we
sample from a truncated Gaussian distribution. We truncate a gaussian
distribution with mean $0$ and variance $4$ at $-6$ and $4$, which induces
some skewness in the distribution, and the resulting distribution has mean
$\approx-0.1$ and variance $3.49$. We repeated these experiments for other
truncated Gaussian distributions with more skewness, for symmetric truncated
gaussian distributions with no skewness, and for truncated distributions
with a triangular probability density function (see
\texttt{scipy.stats.triang}) and obtained similar results to those for the
truncated Gaussian distribution. We used an independent Laplace noise
mechanism as a private estimator (Example~\ref{example:resampling-mean}).
We use $\Ntrial = 160$ and $\Nresamp = 625$.

\begin{figure}[t!] %
	\begin{subfigure}{0.45\textwidth}
		\includegraphics[width=\linewidth]{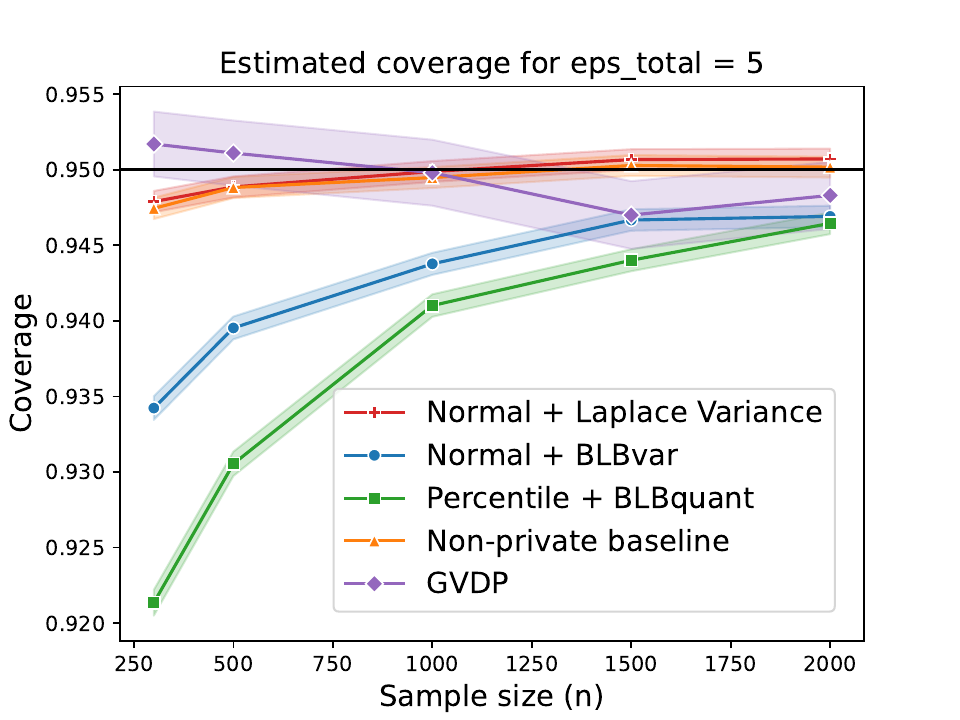}
		\caption{\small Mean estimation coverage} \label{fig:mean-coverage5}
	\end{subfigure}\hspace*{\fill}
	\begin{subfigure}{0.45\textwidth}
		\includegraphics[width=\linewidth]{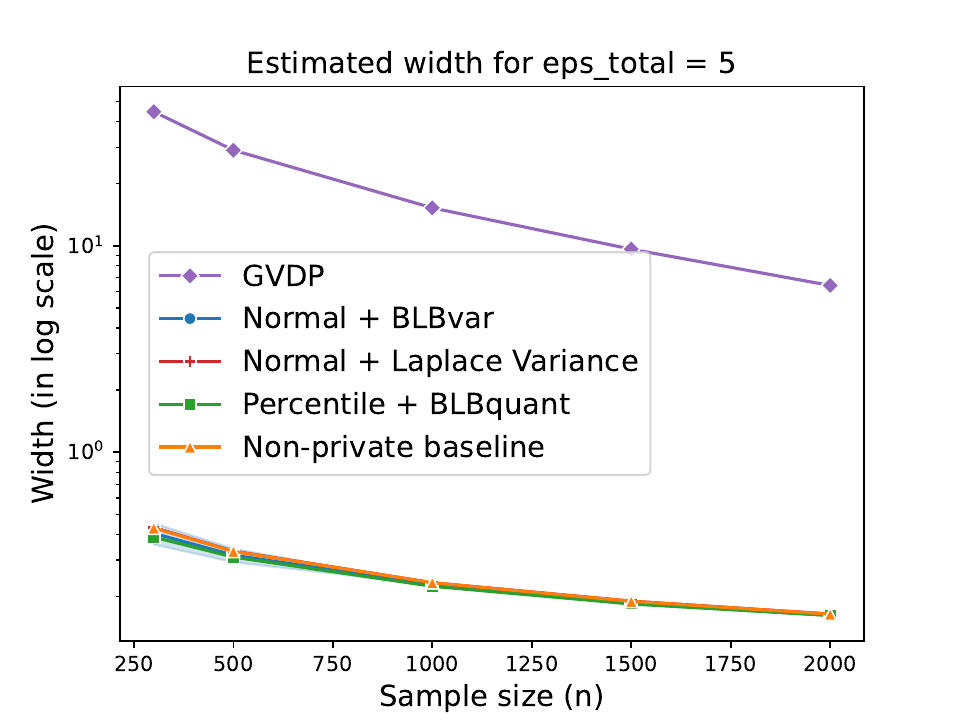}
		\caption{\small Mean estimation width} \label{fig:mean-width5}
	\end{subfigure}
    \caption{Mean estimation with $\totdiffp = 5$}
\end{figure}

\begin{figure}[t!] %
  \begin{subfigure}{0.45\textwidth}
    \includegraphics[width=\linewidth]{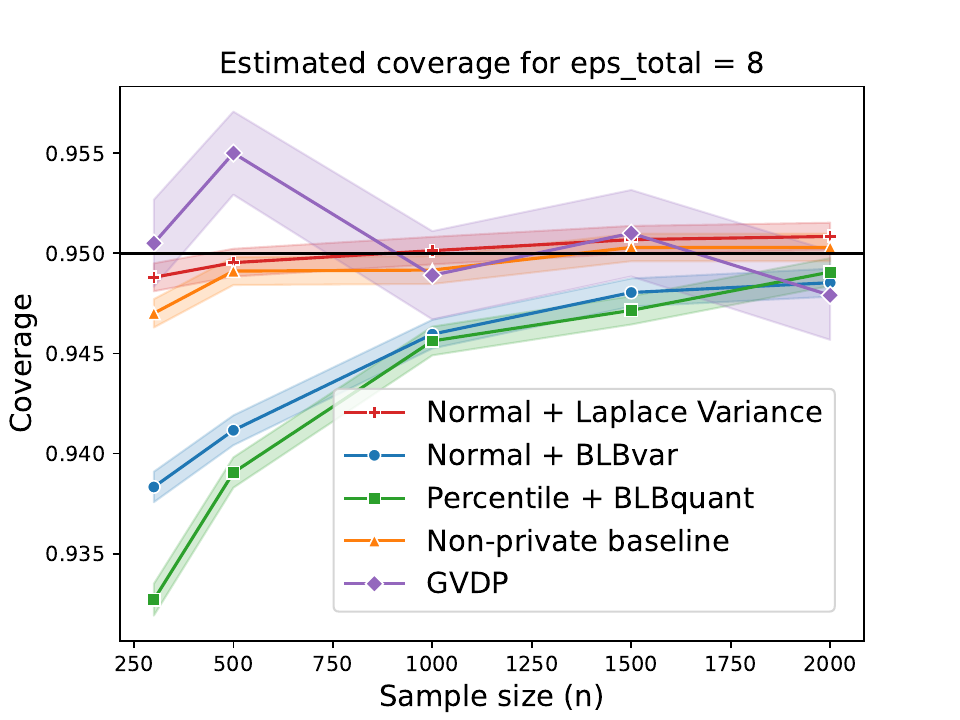}
    \caption{\small Mean estimation coverage} \label{fig:mean-coverage8}
  \end{subfigure}\hspace*{\fill}
  \begin{subfigure}{0.45\textwidth}
    \includegraphics[width=\linewidth]{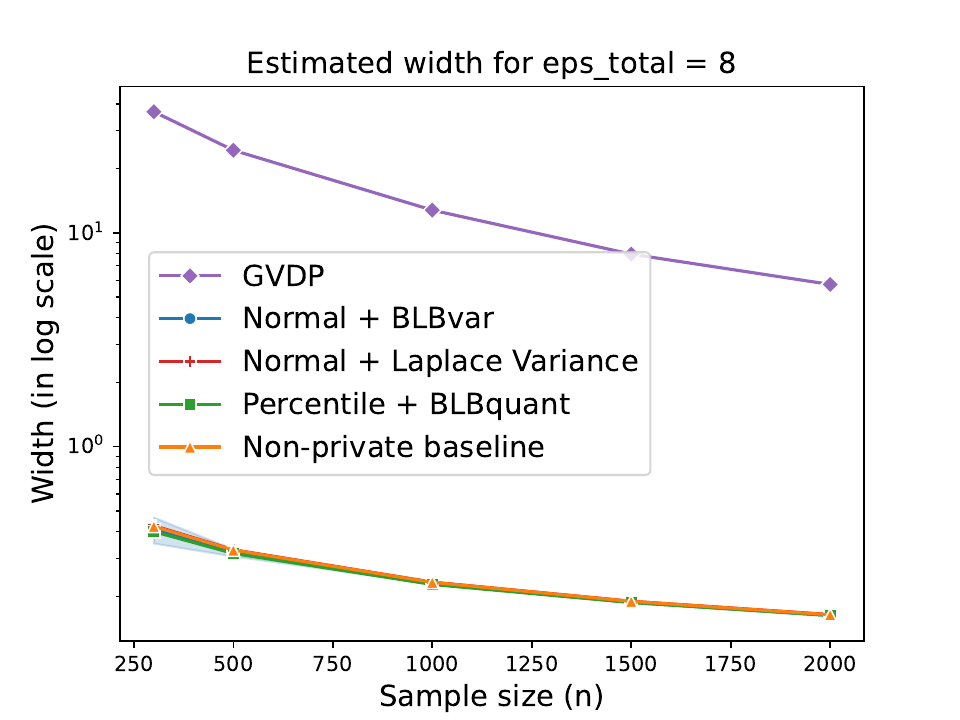}
    \caption{\small Mean estimation width} \label{fig:mean-width8}
  \end{subfigure}
  \caption{Mean estimation with $\totdiffp = 8$}
\end{figure}

\begin{figure}[t!] %
  \begin{subfigure}{0.45\textwidth}
    \includegraphics[width=\linewidth]{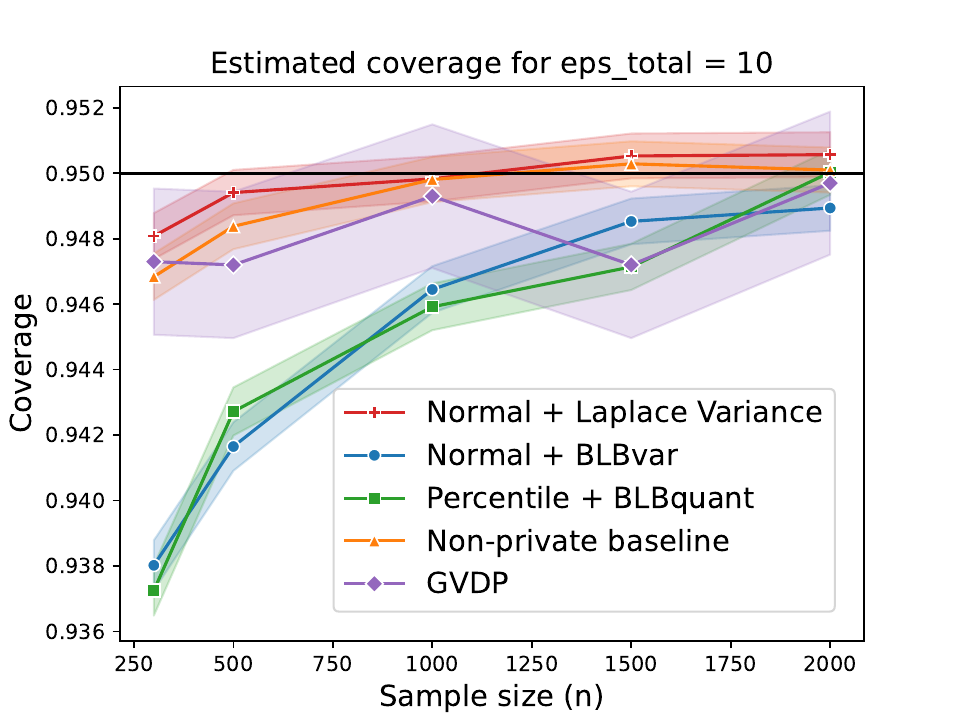}
    \caption{\small Mean estimation coverage} \label{fig:mean-coverage10}
  \end{subfigure}\hspace*{\fill}
  \begin{subfigure}{0.45\textwidth}
    \includegraphics[width=\linewidth]{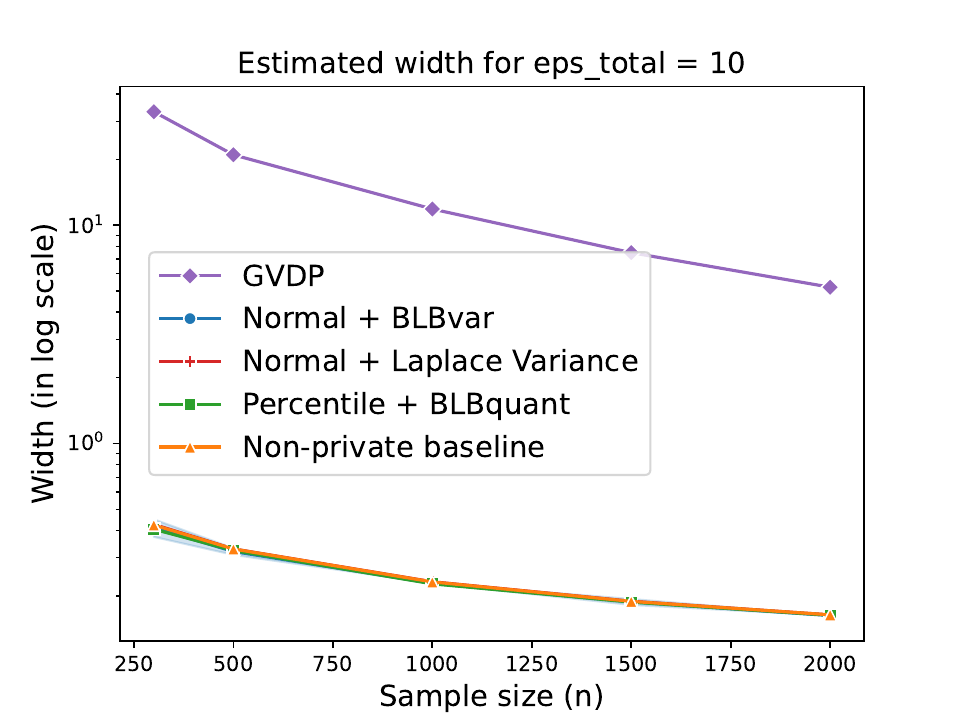}
    \caption{\small Mean estimation width} \label{fig:mean-width10}
  \end{subfigure}
  \caption{Mean estimation with $\totdiffp = 10$}
\end{figure}

\paragraph{Median Estimation:}
For median estimation, we do experiments on a synthetic dataset for which we
again sample from the same Truncated Gaussian distribution as in mean
estimation. We truncate a gaussian distribution with mean $0$ and variance
$4$ at $-6$ and $4$, which induces some skewness in the distribution, and
the resulting distribution has median $\approx -0.054$, and the
sample median has asymptotic variance $\approx 5.99$. We use
$\privmed$, Alg.~\ref{algorithm:priv-median}~\cite{AsiDu20}.
We use $\Ntrial = 80$
and $\Nresamp = 625$.

\begin{figure}[t!] %
  \begin{subfigure}{0.45\textwidth}
    \includegraphics[width=\linewidth]{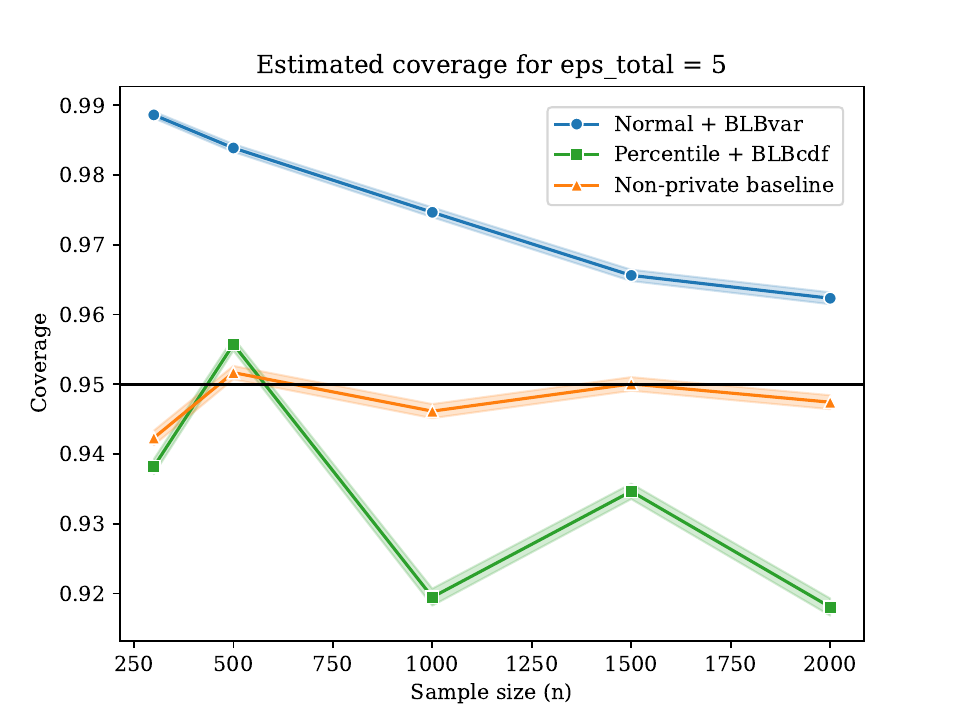}
    \caption{\small Median estimation coverage} \label{fig:median-coverage5}
  \end{subfigure}\hspace*{\fill}
  \begin{subfigure}{0.45\textwidth}
    \includegraphics[width=\linewidth]{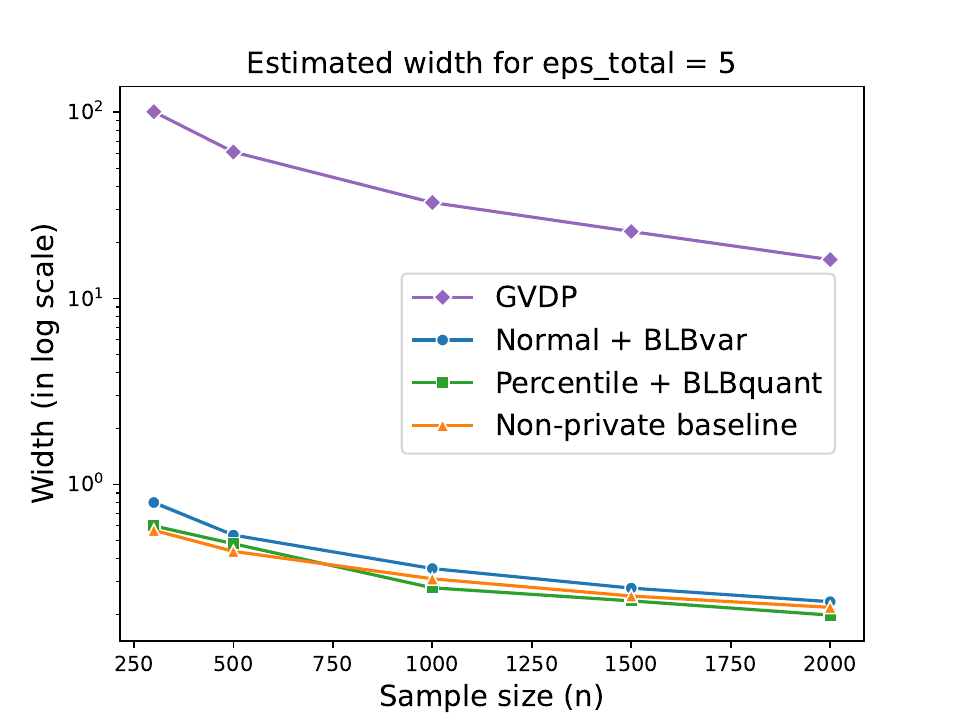}
    \caption{\small Median estimation width} \label{fig:median-width5}
  \end{subfigure}
  \caption{Median estimation with $\totdiffp = 5$}
\end{figure}

\begin{figure}[t!] %
  \begin{subfigure}{0.45\textwidth}
    \includegraphics[width=\linewidth]{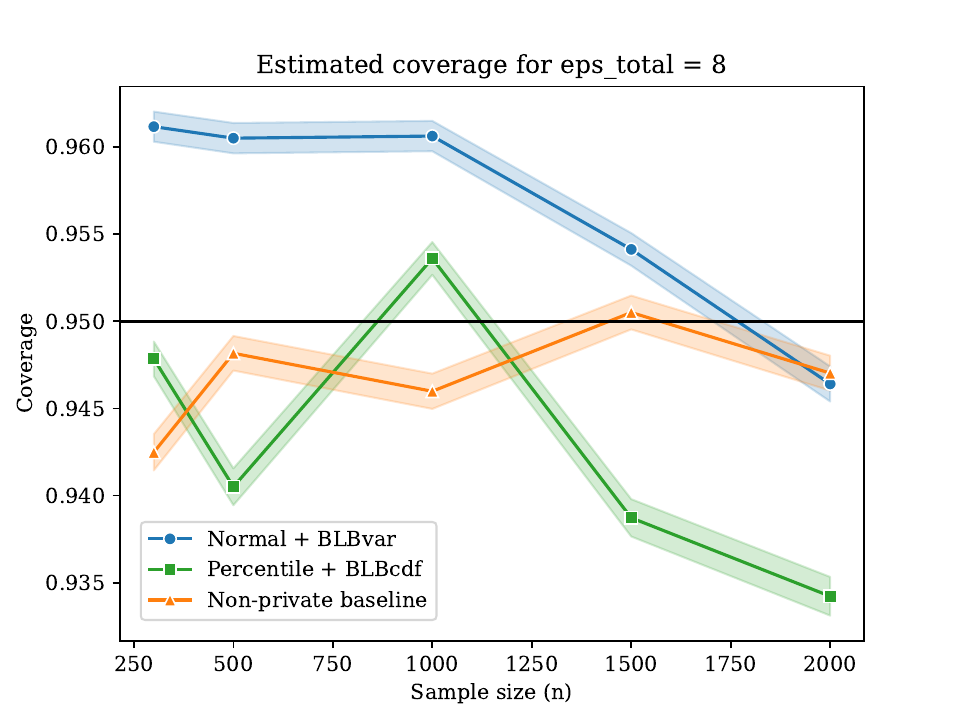}
    \caption{\small Median estimation coverage} \label{fig:median-coverage8}
  \end{subfigure}\hspace*{\fill}
  \begin{subfigure}{0.45\textwidth}
    \includegraphics[width=\linewidth]{figures/median_est_len_eps_8.pdf}
    \caption{\small Median estimation width} \label{fig:median-width8}
  \end{subfigure}
  \caption{Median estimation with $\totdiffp = 8$}
\end{figure}

\begin{figure}[t!] %
  \begin{subfigure}{0.45\textwidth}
    \includegraphics[width=\linewidth]{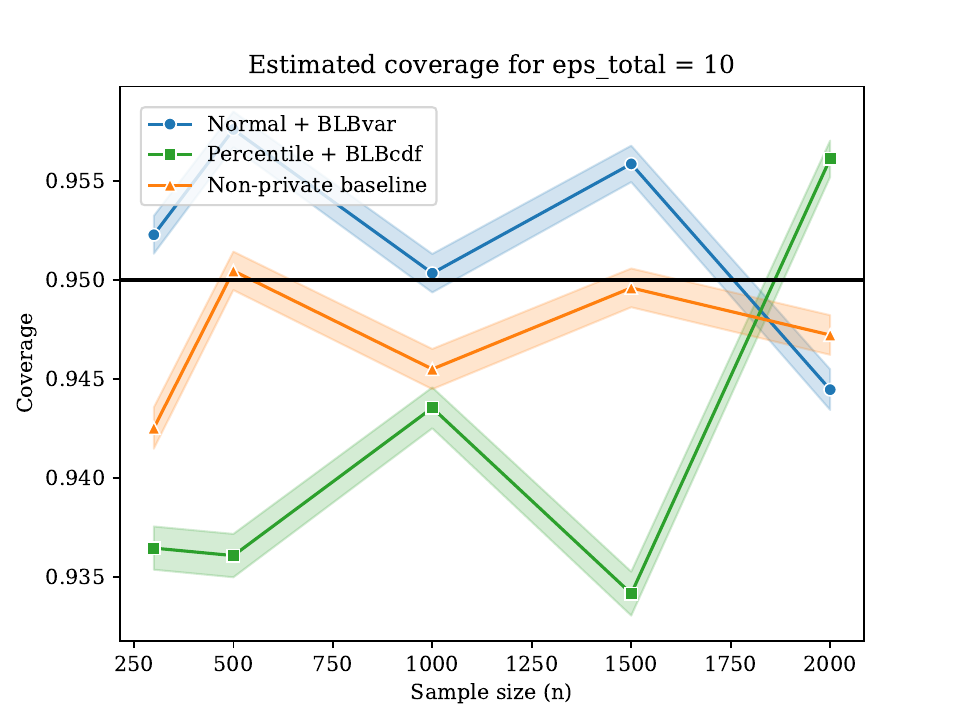}
    \caption{\small Median estimation coverage} \label{fig:median-coverage10}
  \end{subfigure}\hspace*{\fill}
  \begin{subfigure}{0.45\textwidth}
    \includegraphics[width=\linewidth]{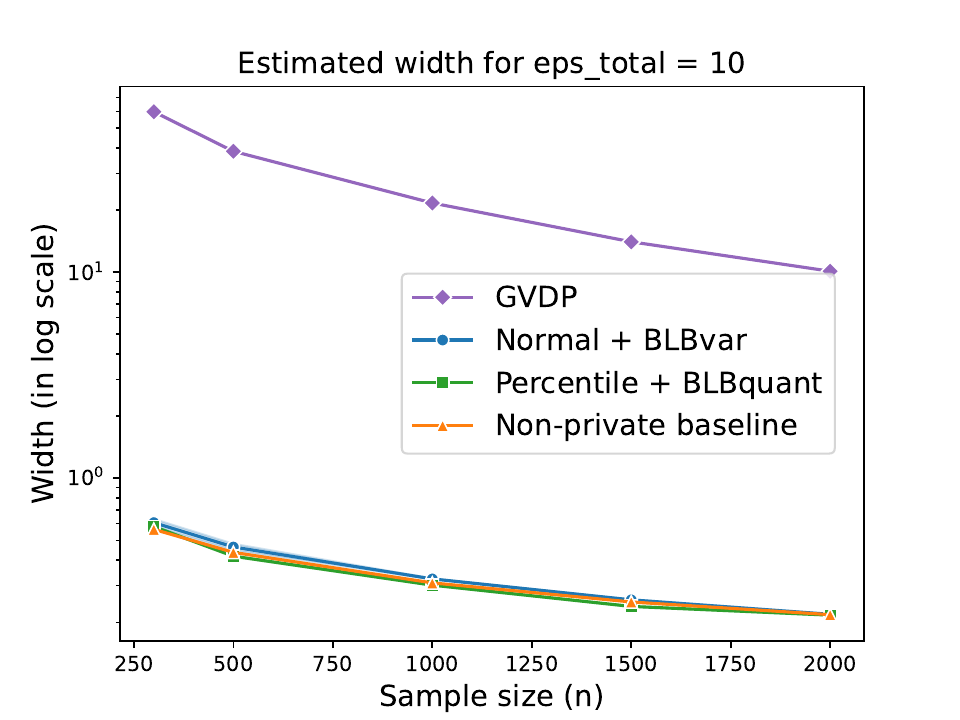}
    \caption{\small Median estimation width} \label{fig:median-width10}
  \end{subfigure}
  \caption{Median estimation with $\totdiffp = 10$}
\end{figure}

\paragraph{Logistic Regression:}
For logistic regression, we \citeauthor{DingHaMiSc21}'s
updated \textit{Adult
  Income} dataset~\cite{DingHaMiSc21} and their
\texttt{folktables} package.
We construct a dataset
of $n_{\rm tot} = 1587856$ datapoints with 4 features corresponding to age,
school, working hours per week, sex of the individual, and an additional
intercept (bias) term, normalizing all features to lie in $[0,1]$.
The task is to predict whether
an individual earns more than \$30K per year, and to provide
a confidence interval for the parameter on the Sex variable. In our
dataset, we have 53.8\% positive labels. The
non-private ERM solution using all the data
has accuracy of 77.8\%. We use this solution as
the ``true solution'' to calculate coverage.

We perform experiments on smaller sample sizes ($n = 2000$ to $8000$) by
sampling $n$ points from the whole dataset of $n_{\rm tot}$ points. We use
an approximation to the inverse sensitivity mechanism~\cite[Section
  5.2]{AsiDu20} to privately estimate the logistic regression
problem. \citet{AsiDu20} give a first-order accurate approximation to the
private estimator, which we use as the full sampling distribution requires a
sophisticated Metropolis Hastings scheme, and our main focus is on the
accuracy of intervals from the bootstrap rather than the initial estimator
itself.  We use $\Ntrial = 80$ and $\Nresamp = 625$, studying coverage for
the single parameter corresponding to the sex variable.

\begin{figure}[t!] %
  \begin{subfigure}{0.45\textwidth}
    \includegraphics[width=\linewidth]{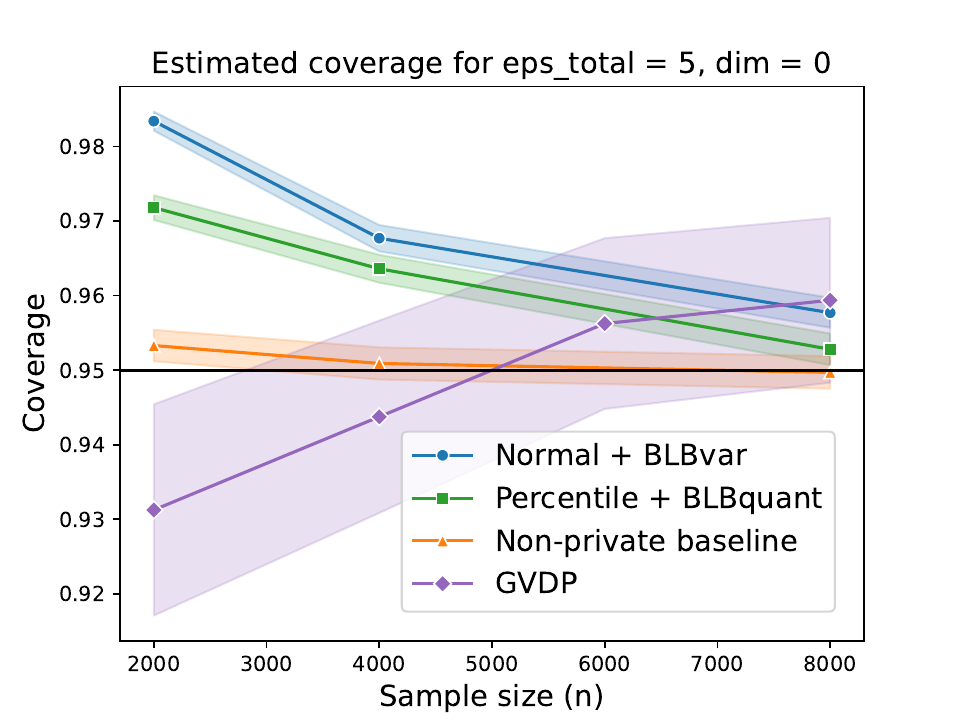}
    \caption{\small Logistic Regression coverage} \label{fig:logistic-coverage5}
  \end{subfigure}\hspace*{\fill}
  \begin{subfigure}{0.45\textwidth}
    \includegraphics[width=\linewidth]{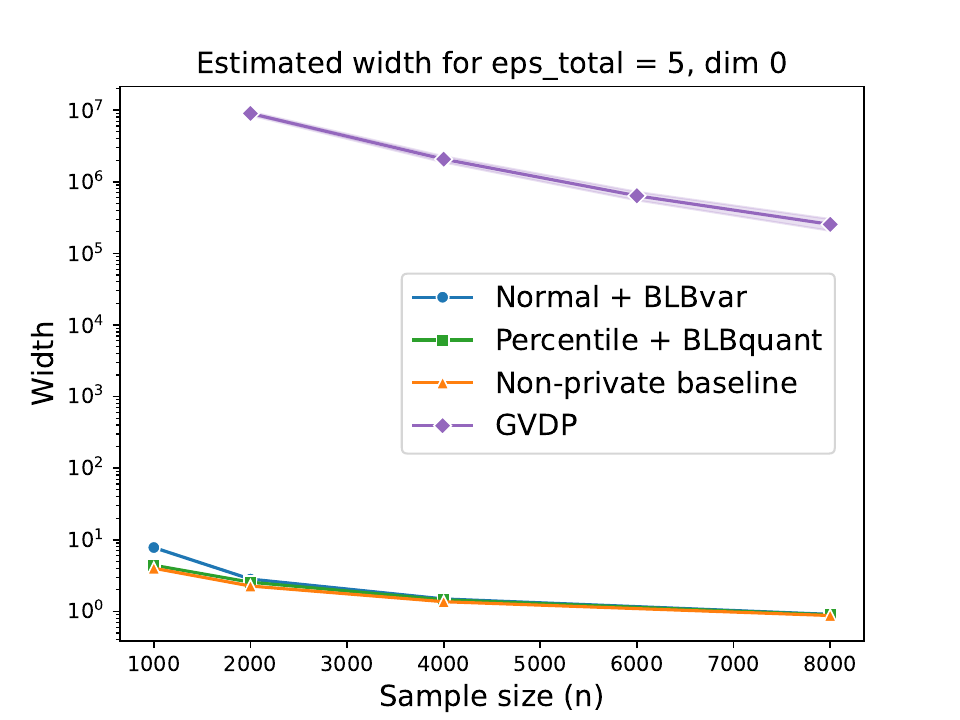}
    \caption{\small Logistic Regression width} \label{fig:logistic-width5}
  \end{subfigure}
  \caption{Logistic Regression with $\totdiffp = 5$}
\end{figure}

\begin{figure}[t!] %
  \begin{subfigure}{0.45\textwidth}
    \includegraphics[width=\linewidth]{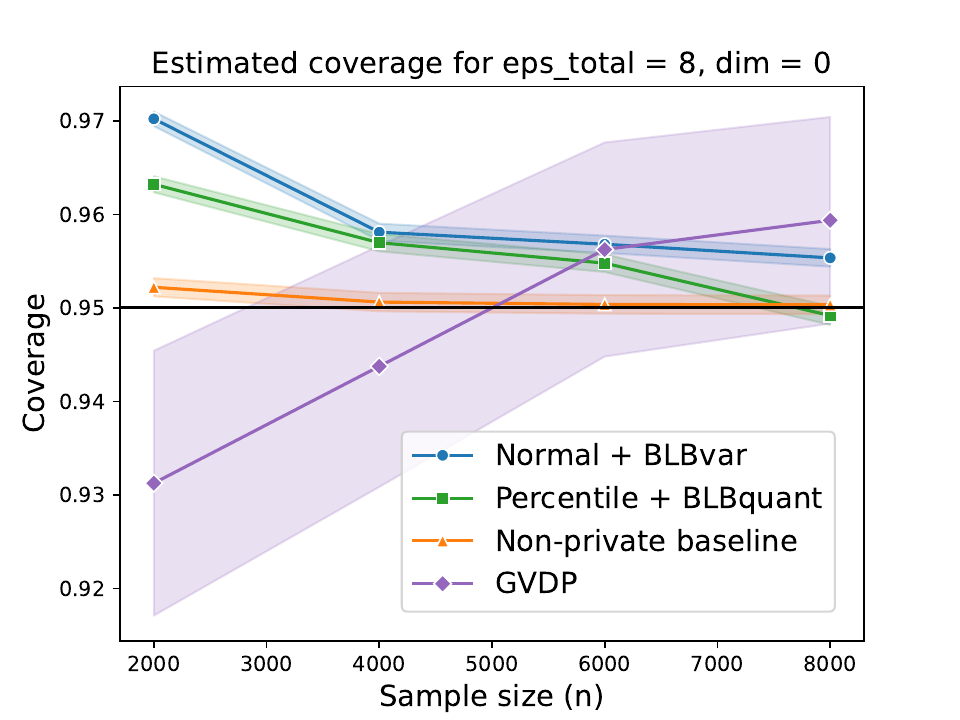}
    \caption{\small Logistic Regression coverage} \label{fig:logistic-coverage8}
  \end{subfigure}\hspace*{\fill}
  \begin{subfigure}{0.45\textwidth}
    \includegraphics[width=\linewidth]{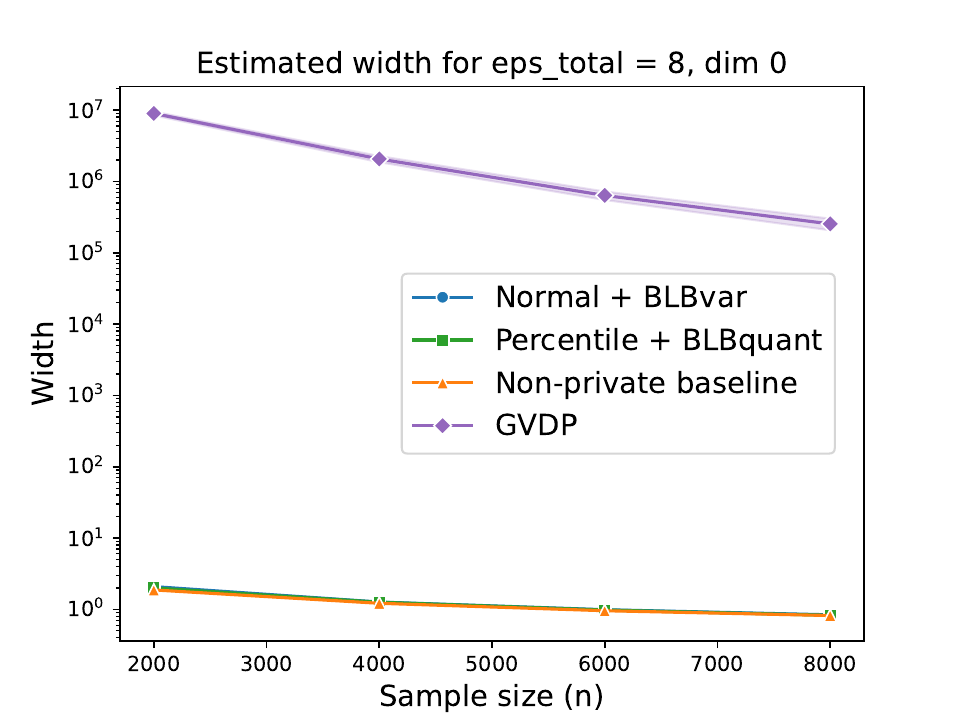}
    \caption{\small Logistic Regression estimation width} \label{fig:logistic-width8}
  \end{subfigure}
  \caption{Logistic Regression with $\totdiffp = 8$}
\end{figure}

\subsection{Ablations}
\label{sec:ablations}

We study the source of the loss in performance relative to non-private
bootstrap estimation for both $\blbvar$
(\Cref{fig:abl-blbvar-2,fig:abl-blbvar-5,fig:abl-blbvar-8}) and $\blbquant$
(\Cref{fig:abl-blbquant-2,fig:abl-blbquant-5,fig:abl-blbquant-8}).  At a
high level, the results here suggest that subsampling---by reducing the
sample size---induces the main accuracy degradation, though the median
above-threshold algorithm (Alg.~\ref{algorithm:above-threshold-median})
introduces some error, and finding better ways to release a threshold for
the indexed confidence sets $I_t$ may yield improvements.

For $\blbvar$, across different total differential privacy values
$\totdiffp$, we plot the coverage for the following three algorithms in
Figures~\ref{fig:abl-blbvar-2}--\ref{fig:abl-blbvar-8}:
\begin{enumerate}[leftmargin=*,label=\arabic*]
\item \label{item:normal-thing}
  $\algname{Normal + BLBvar}$: The z-score-based confidence
  interval we propose in Corollary~\ref{corollary:coverage-from-variance}.
\item \label{item:nonprivate-median-agg-var}
  $\algname{Nonprivate BLBvar}$: To measure the error of taking a private
  median (line~\ref{line:return-private-median} of $\blbvar$),
  we aggregate using the true median instead of the private
  median in $\blbvar$. All other algorithmic details remain identical.
\item \label{item:cereal-fucker}
  $\algname{Upsampled BLBvar}$: We perform $\nsubs$
  bootstrap resamples of the entire dataset instead of partitioning the
  data, that is, we run $\bootvaralg$ on $P_n$ instead of
  $P\subsample_\subsize$.  Because the full sample $P_n$ is re-used for
  each, this fails to satisfy differential privacy, and while nominally (if
  Monte Carlo sampling introduced no error) this should be equivalent to a
  single full bootstrap resample, because $\nmontecarlo < \infty$ there is
  some variability between the $\nsubs$ resamples. This allows us to
  delineate whether the private median
  (Line~\ref{line:return-private-median}) introduces substantial error.
\end{enumerate}

$\algname{Nonprivate BLBvar}$ (item~\ref{item:normal-thing}) and
$\algname{Normal + BLBvar}$ (item~\ref{item:nonprivate-median-agg-var}) have
similar performance for all values of $\totdiffp$. The full resampling
procedure $\algname{Upsampled BLBvar}$ (item~\ref{item:cereal-fucker}) has
more accurate coverage, as we expect from the non-private baseline
experiments it essentially mimics (recall Fig.~\ref{fig:1} in the
experiments, Secton~\ref{sec:expts}). The effective sample size reduction
that subsampling introduces thus appears to be the main source of accuracy
degradation.

\begin{figure}
  \begin{subfigure}{0.33\textwidth}
    \includegraphics[width=\linewidth]{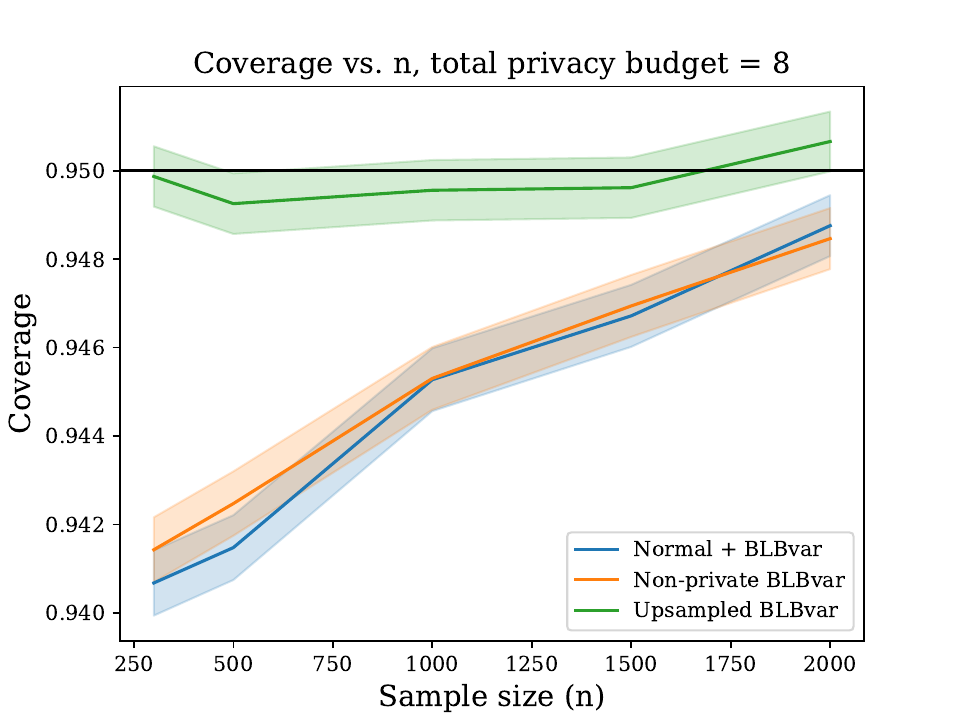}
    \caption{\small  $\totdiffp = 8$} \label{fig:abl-blbvar-8}
  \end{subfigure}\hspace*{\fill}
  \begin{subfigure}{0.33\textwidth}
    \includegraphics[width=\linewidth]{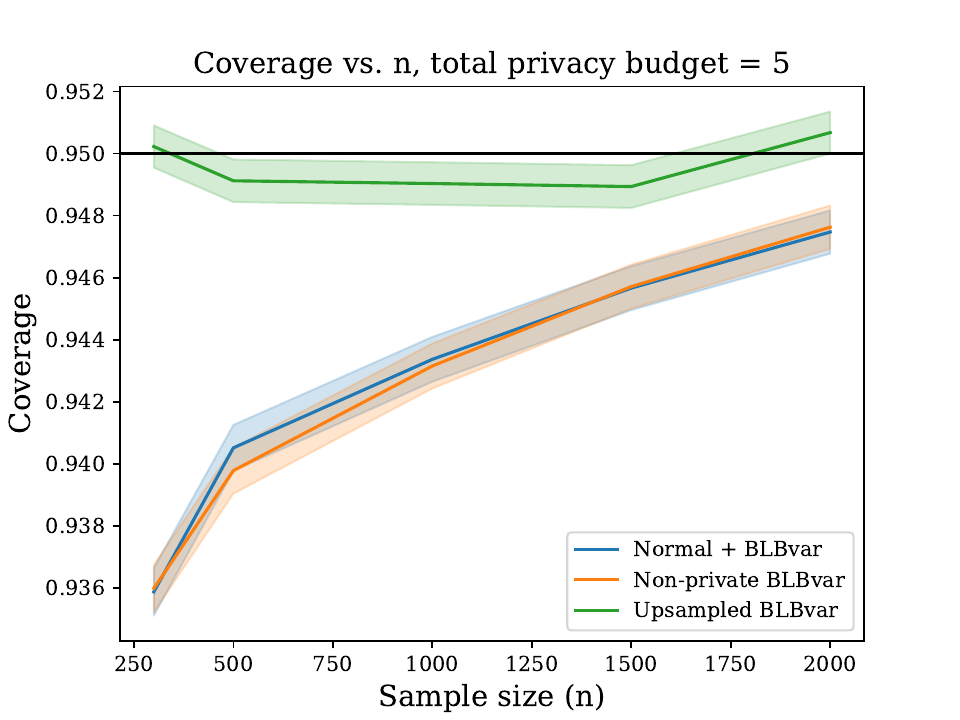}
    \caption{\small $\totdiffp = 5$} \label{fig:abl-blbvar-5}
  \end{subfigure}\hspace*{\fill}
  \begin{subfigure}{0.33\textwidth}
    \includegraphics[width=\linewidth]{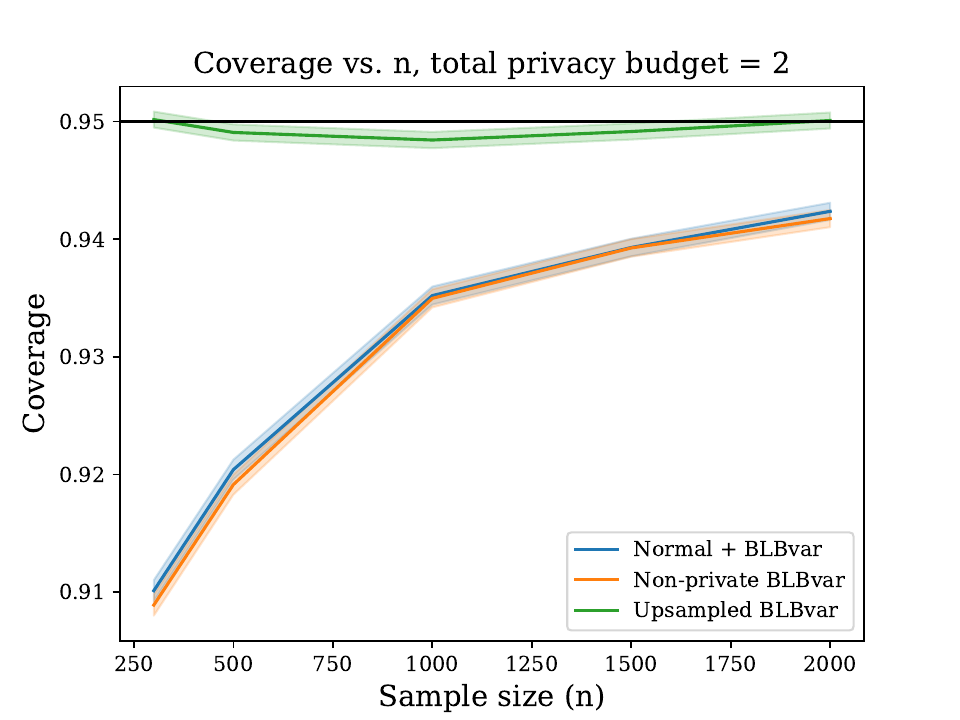}
    \caption{\small $\totdiffp = 2$} \label{fig:abl-blbvar-2}
  \end{subfigure}
  \caption{Ablations for $\blbvar$ on mean estimation.}
\end{figure}

\begin{figure}
  \begin{subfigure}{0.33\textwidth}
    \includegraphics[width=\linewidth]{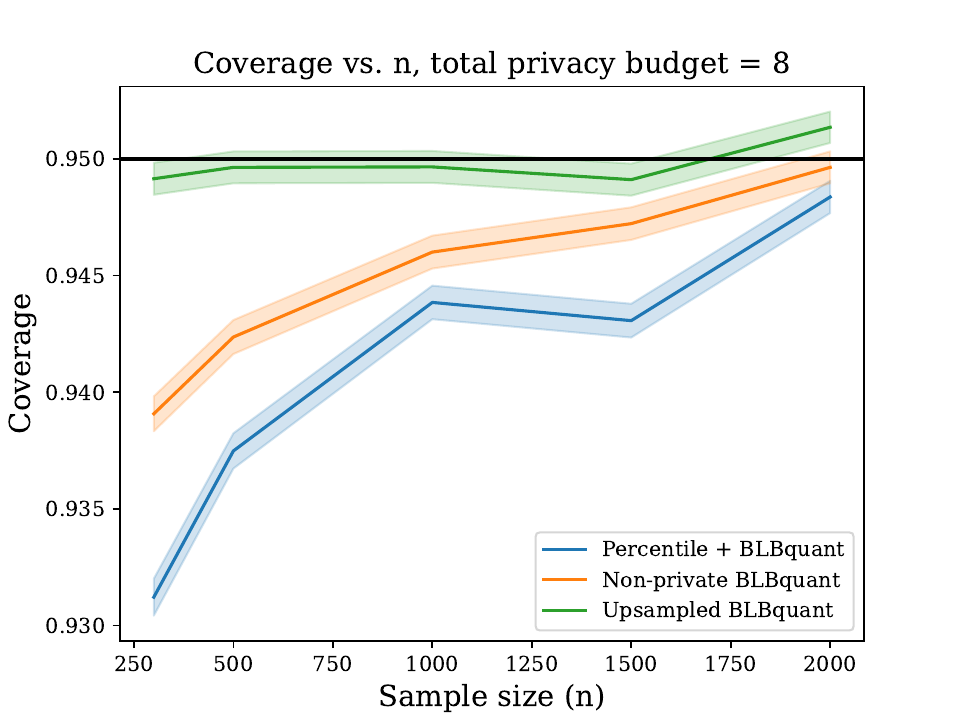}
    \caption{\small  $\totdiffp = 8$} \label{fig:abl-blbquant-8}
  \end{subfigure}\hspace*{\fill}
  \begin{subfigure}{0.33\textwidth}
    \includegraphics[width=\linewidth]{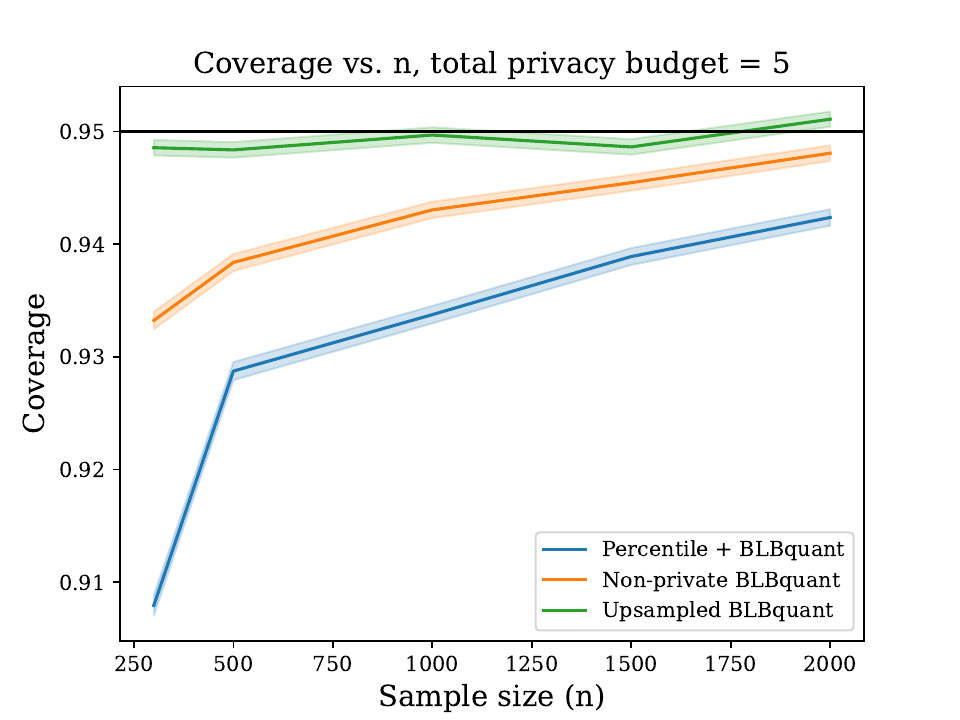}
    \caption{\small  $\totdiffp = 5$} \label{fig:abl-blbquant-5}
  \end{subfigure}\hspace*{\fill}
  \begin{subfigure}{0.33\textwidth}
    \includegraphics[width=\linewidth]{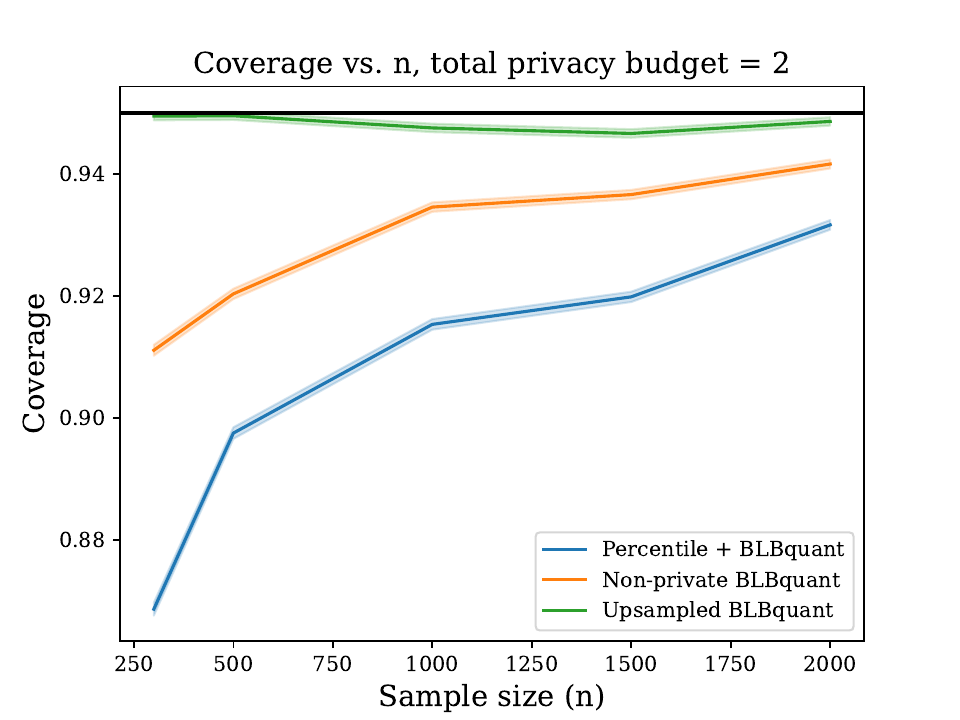}
    \caption{\small  $\totdiffp = 2$} \label{fig:abl-blbquant-2}
  \end{subfigure}
  \caption{Ablations for $\blbquant$ on mean estimation}
\end{figure}

For $\blbvar$, across different total differential privacy values
$\totdiffp$, we plot the coverage for the following three algorithms in
Figures~\ref{fig:abl-blbquant-2}--\ref{fig:abl-blbquant-8}:
\begin{enumerate}[leftmargin=*,label=\arabic*]
\item \label{item:percentile-thing} $\algname{Percentile + BLBquant}$: The
  quantile-based interval we propose in
  Corollary~\ref{corollary:coverage-from-percentile}.
\item \label{item:nonprivate-median-agg-percentile} $\algname{Nonprivate
  BLBquant}$: To study the error using the noisy order statistics in
  $\abthreshmed$ induces (line~\ref{line:noisy-order-statistic}), we use the
  true median in line~\ref{line:noisy-order-statistic}.  All other
  algorithmic details remain identical.
\item \label{item:pancake-fucker} $\algname{Upsampled BLBquant}$: We perform
  $\nsubs$ bootstrap resamples of the entire dataset instead of partitioning
  the data, that is, we replace the subsamples $P\subsampleind{i}_\subsize$
  in line~\ref{line:draw-subsamples-quant} of Alg.~\ref{algorithm:blb-quant}
  with the sample $P_n$. Because the full sample $P_n$ is re-used for
  each, this fails to satisfy differential privacy, and while nominally (if
  Monte Carlo sampling introduced no error) this should be equivalent to a
  single full bootstrap resample, because $\nmontecarlo < \infty$ there is
  some variability between the $\nsubs$ resamples. This allows us to
  delineate whether using the private median-based above
  treshold algorithm (line~\ref{line:call-above-threshold}
  of $\blbquant$) introduces error.
\end{enumerate}
As in our ablation experiments on $\blbvar$, $\algname{Upsampled BLBquant}$
(item~\ref{item:pancake-fucker}) has the best accuracy, as we expect from our
initial experiments.  In this case, non-privately aggregating the results of
the BLB subsamples---$\algname{Nonprivate BLBquant}$,
item~\ref{item:nonprivate-median-agg-percentile}---improves accuracy over
$\algname{Percentile + BLBquant}$; the improvement is more substantial for
smaller $\totdiffp$ (i.e., more privacy). This suggests that improving our
ability to select the accurately covering set $I_t$, rather than relying on
the noisy order statistics in $\abthreshmed$, could yield improvements,
especially at smaller sample sizes.

\end{document}